\pdfoutput=1
\documentclass[opre]{informs4}

\RequirePackage{tgtermes}
\RequirePackage{newtxtext}
\RequirePackage{newtxmath}
\RequirePackage{bm}
\makeatletter
\DeclareFontEncoding{LS1}{}{}
\DeclareFontSubstitution{LS1}{stix}{m}{n}
\DeclareSymbolFont{STIXlargesymbols}{LS1}{stix}{m}{it}
\renewcommand{\widetilde}{\mathaccent"0\hexnumber@\symSTIXlargesymbols 9B }
\makeatother

\newlength{\DSnat}
\newlength{\DSstr}
\newlength{\DSshr}
\newlength{\DSbelow}
\renewcommand{\Dskips}[4]{%
  \setlength{\DSnat}{#1em}\setlength{\DSnat}{0.8\DSnat}%
  \setlength{\DSstr}{#2em}\setlength{\DSstr}{0.8\DSstr}%
  \setlength{\DSshr}{#3em}\setlength{\DSshr}{0.8\DSshr}%
  \setlength{\DSbelow}{#4em}\setlength{\DSbelow}{0.8\DSbelow}%
  \setlength{\abovedisplayskip}{1\DSnat plus 1\DSstr minus 1\DSshr}%
  \setlength{\abovedisplayshortskip}{0pt plus 1\DSstr}%
  \setlength{\belowdisplayshortskip}{1\DSbelow plus 1\DSstr minus 1\DSstr}%
  \setlength{\belowdisplayskip}{\abovedisplayskip}}
\OneAndAHalfSpacedXI

\usepackage{microtype}
\usepackage{mathtools}
\usepackage{enumitem}
\usepackage{algorithm}
\usepackage[noend]{algorithmic}
\usepackage{booktabs}
\usepackage{siunitx}
\usepackage{xcolor}
\usepackage{natbib}
 \bibpunct[, ]{(}{)}{,}{a}{}{,}%
 \def\bibfont{\small}%
\usepackage{xurl}
\usepackage[hidelinks]{hyperref}
\usepackage[capitalize,noabbrev]{cleveref}

\EquationsNumberedThrough
\TheoremsNumberedThrough
\def\ExampleTextFont{\thit}
\def\theARTICLETOPLEFT{}
\def\theARTICLETOPRIGHT{}
\makeatletter
\def\theARTICLEABSTRACT{%
  \vspace*{18pt}%
  \begin{minipage}[t]{\textwidth}\parindent1em\ABSfont
    \noindent\theABSTRACT\endgraf
    \vskip5pt
    \theKEYWORDS
    \noindent\hrulefill
  \end{minipage}%
  \vspace*{0pt}%
}
\makeatother

\let\AppendixHeadSize=\normalsize
\let\AppendixFontSize=\normalsize
\let\AppendixSectionHeadSize=\normalsize

\makeatletter
\def\theRRHFirstLine{}%
\def\theLRHFirstLine{}%
\def\theRRHSecondLine{}%
\def\theLRHSecondLine{}%
\makeatother

\AddToHook{env/algorithm/begin}{\SingleSpacedXI}
\AddToHook{env/algorithmic/begin}{\footnotesize}

\crefname{assumption}{Assumption}{Assumptions}
\crefformat{section}{\S#2#1#3}
\crefformat{subsection}{\S#2#1#3}
\crefname{appendix}{Appendix}{Appendices}

\allowdisplaybreaks
\setlist{itemsep=2pt,topsep=4pt}

\let\origendproof\endproof
\def\endproof{\Halmos\origendproof}

\newcommand{\R}{\mathbb R}
\newcommand{\E}{\mathbb E}
\newcommand{\Pp}{\mathbb P}
\newcommand{\cH}{\mathcal H}
\newcommand{\cL}{\mathcal L}
\newcommand{\cJ}{\mathcal J}
\newcommand{\cY}{\mathcal{Y}}

\newcommand{\Var}{\mathbb V}
\newcommand{\Cov}{\operatorname{Cov}}
\newcommand{\FamilyRow}[2]{\mathmakebox[0.36\linewidth][l]{#1}\mathmakebox[0.62\linewidth][l]{#2}}
\newlength{\CaseWidth}
\AtBeginDocument{\settowidth{\CaseWidth}{$z\,I_{z/r}(a_i,b_i)-\mu_i\,I_{z/r}(a_i+1,b_i),$}}
\newcommand{\CaseFirst}[1]{\mathmakebox[\CaseWidth][l]{#1}}

\newsavebox{\FDbox}
\newlength{\FDhalf}
\newlength{\FDshift}
\newlength{\FDup}
\newlength{\FDdown}
\NewDocumentCommand{\FitDelims}{mmm}{\mathinner{\mathpalette\FitDelimsBody{{#1}{#2}{#3}}}}
\NewDocumentCommand{\FitParens}{m}{\FitDelims{(}{)}{#1}}
\newcommand{\FDmidplain}[1]{\mathord{#1}}
\let\FitMid\FDmidplain
\newcommand{\FitDelimsBody}[2]{\FitDelimsDo#1#2}
\newcommand{\FitDelimsDo}[4]{%
  \sbox{\FDbox}{$\mathsurround=0pt \let\FitMid\FDmidplain #1#4$}%
  \setlength{\FDhalf}{\ht\FDbox}%
  \addtolength{\FDhalf}{\dp\FDbox}%
  \divide\FDhalf by 2
  \setlength{\FDup}{\FDhalf}%
  \addtolength{\FDup}{\fontdimen22\textfont2}%
  \setlength{\FDdown}{\FDhalf}%
  \addtolength{\FDdown}{-\fontdimen22\textfont2}%
  \setlength{\FDshift}{\ht\FDbox}%
  \addtolength{\FDshift}{-\FDup}%
  \renewcommand{\FitMid}[1]{\FDdelim{\left##1}{\right.}}%
  \mathopen{\FDdelim{\left#2}{\right.}}#4\mathclose{\FDdelim{\left.}{\right#3}}}
\newcommand{\FDdelim}[2]{%
  \raisebox{\FDshift}{$\delimiterfactor=1000 \delimitershortfall=0pt \nulldelimiterspace=0pt #1\vrule height\FDup depth\FDdown width 0pt#2$}}

\MANUSCRIPTNO{}

\begin{document}

\RUNAUTHOR{Goldberger and Rudi} \RUNTITLE{A Simple and General Approach to top-$k$ Selection}

\TITLE{A Simple and General Approach to top-$k$ Selection}

\ARTICLEAUTHORS{%
\AUTHOR{Motti Goldberger} \AFF{Yale School of Management, \EMAIL{motti.goldberger@yale.edu}} \AUTHOR{Nils Rudi} \AFF{Yale School of Management, \EMAIL{nils.rudi@yale.edu}} }

\ABSTRACT{ This paper considers the Bayesian top-$k$ identification problem, wherein a decision maker is tasked with selecting $k$ items from a pool of items of unknown quality. To assess the items, the decision maker adaptively allocates a fixed sampling budget, where each sample of an item yields a noisy observation of its quality, with the objective of minimizing the expected opportunity cost of the selected set. We propose an approach that allocates each batch of samples by minimizing a surrogate objective derived from an upper bound on the expected opportunity cost. The approach is general: samples can be allocated in batches, the selection size $k$ is arbitrary, and it applies to many sampling distributions. It is also simple and intuitive. We prove that the method's posterior expected opportunity cost decays at an exponential rate of at least half the best rate achievable by any sampling rule. In simulations across a wide range of settings, the method is competitive with, and often improves upon, the best existing algorithm in each setting. }


\maketitle

\section{Introduction}

This paper studies the problem of a decision maker who faces a large pool of items with unknown quality, can assess them through noisy and costly samples, and wants to identify and select the best item(s). This is commonly referred to as a \emph{ranking and selection} task in the operations literature \citep{ChenLin2000,chenLee2011,HongFanLuo2021}, and as a \emph{best-arm(s) identification} task in the multi-armed bandit literature \citep{kalyanakrishnan2012,kaufmann2016jmlr,russo2020}. This problem arises in many domains, including clinical trials \citep{villar2015,chickGansYapar2022}, A/B/n testing \citep{scott2015,juHuHendersonHong2019}, machine-learning model selection and hyperparameter optimization \citep{snoek2012,jamiesonTalwalkar2016,li2018hyperband}, and simulation optimization \citep{fuGloverApril2005}. In each setting, the key question is how to adaptively allocate a fixed sampling budget to maximize the quality of the final selection.

A central goal of this paper is to treat the problem with generality so that the method is useful across a range of settings. Two aspects of this generality are worth discussing up front. The first concerns how samples are collected. The most adaptive option is to allocate one sample at a time and wait to observe the outcome before choosing the next. However, this can be impractical to implement. For example, in clinical trials, one may not be able to wait to observe an outcome before assigning the next treatment \citep{robertson2023rar}, and in simulation optimization, many simulations may be run in parallel \citep{HongFanLuo2021}. We therefore consider the partially adaptive case, in which \emph{batches} of samples are chosen in rounds, and outcomes are observed between rounds. The second aspect is the target number of items to select. Choosing the single best one ($k=1$) is the most common and most studied case, but some of the applications above instead call for selecting several items. Our method can handle any target selection size $k$, giving it a wider range of applicability.

There are two common metrics for how ``good'' a sample allocation procedure is. One is its probability of correct selection (PCS), the probability that the selected items are the true best items \citep{ChenLin2000,glynnJuneja2004,li2025greedy}. The other is its expected opportunity cost (EOC), sometimes referred to as simple regret, which is the expected shortfall in value between the selected set and the true best set \citep{chickInoue2001,GaoChen2015}. PCS treats every incorrect selection equally, while EOC makes bigger mistakes more costly. For many tasks, EOC, which we adopt as our objective, is the more natural fit since a big mistake can be far worse than a marginal one. That said, the two are closely related. In simulations, a method that performs well on one metric tends to perform well on the other, and our theoretical guarantees apply to both objectives.

Given the objective, the decision maker's problem can be formulated as a finite-horizon dynamic program. In each round, given the current state (the posterior distribution over item values) and the available budget, one chooses the optimal batch of samples by accounting for all possible outcomes and future states. However, the complexity of this program grows rapidly with the budget size, so solving it exactly is out of reach \citep{FrazierPowell2008,HongFanLuo2021,ZhangPengZhangZhou2023}. This has motivated a large body of literature, which we group into two broad categories. The first set of methods allocates samples by solving a sequence of static optimization problems, one per round, in place of the dynamic program. The second group constructs allocation rules directly from the current state without solving an optimization problem. We take the first approach and construct a per-round surrogate objective whose minimizer gives the sample allocation.

There are several ways to form a sequence of one-round static objectives. In a Bayesian formulation, one common approach is to maximize the expected value of information (EVI), the improvement in the quality of the final selection that collecting the next batch would produce \citep{chickInoue2001,ChickBrankeSchmidt2010}. When samples follow a Gaussian distribution and are chosen one at a time, \citet{FrazierPowell2008,frazier2009} show that this can be computed exactly for the top-$1$ EOC objective, yielding the knowledge gradient (KG) policy. In general, however, even the one-step look-ahead objective is intractable, so the literature resorts to approximations. Examples targeting PCS include AOAP for top-$1$ identification \citep{peng2018}, and AOAm \citep{ZhangPengZhangZhou2023} and EMIp \citep{Liu2024} for top-$k$, all designed for Gaussian samples. For AOAm, \citet{ZhangPengZhangZhou2023} also prove that the sampling proportions converge to the proportions that maximize the asymptotic error-decay rate, assuming their necessary sample-proportion conditions have a unique solution.

A second approach in this category, referred to as optimal computing budget allocation (OCBA) \citep{chenLee2011}, assumes that in each round the current estimates of the means are the true means and finds the sample allocation that would be optimal under this assumption \citep{ChenLin2000}. However, even with this assumption, the problem is generally intractable. It is therefore commonly approached by constructing a bound on the objective, a lower bound for PCS or an upper bound for EOC, and deriving allocations that optimize the bound as the budget goes to infinity. Most of this work concerns Gaussian samples and top-$1$ selection under PCS, with extensions to top-$k$ \citep{ChenHeFuLee2008,gaoChen2016}, and to top-$1$ and top-$k$ selection under EOC \citep{HeChick2007,GaoChen2015}.

An alternative approach targets learning the allocation that maximizes the exponential rate at which the probability of incorrect selection decays as the budget grows, the large-deviation rate \citep{glynnJuneja2004}. Algorithms of this kind run sequentially, again treating the current mean estimates as the truth, and choose allocations using the optimality conditions for this rate. \citet{chenRyzhov2023} establish such an algorithm that applies to general distributions in the $k=1$ setting and prove that the sampling proportions converge to the optimal proportions as the budget approaches infinity.

Next, we review methods that construct sampling rules directly from the current state. Methods similar in spirit to our approach are confidence and gap-based frequentist approaches, which allocate samples based on how close items are to the current estimated top-$k$ cutoff and how uncertain their estimated values are, such as gap-e \citep{gabillon2011,bubeck2013} and KL-LUCB \citep{kaufmann2013}. Our allocation algorithm is based on the same two quantities, but differs in how the top-$k$ cutoff is defined, how the distance from it is measured, and how uncertainty is used.

Elimination algorithms \citep{audibert2010,karnin2013} divide sampling into rounds, allocate samples uniformly among the remaining items within a round, and eliminate the empirically weakest items according to an elimination schedule. The successive accepts and rejects (SAR) approach \citep{bubeck2013} also ``accepts'' empirically strong items, avoiding spending samples on items already determined to likely be in the top-$k$.

Algorithms based on the expected improvement (EI) rule, which originates from the Bayesian optimization literature \citep{jonesSchonlauWelch1998}, have been designed for top-$1$ selection problems \citep{ryzhov2016,QinKlabjan2017,chenRyzhov2019}. For each item, the algorithm computes the expected amount by which the item's value exceeds that of the current best item, and samples the item with the largest improvement score. In the Gaussian case, \citet{chenRyzhov2019} show that a variant of EI achieves the optimal sampling proportions asymptotically.

A family of Bayesian algorithms referred to as \emph{top two} algorithms allocates a sample in each round by identifying a ``leader'' item and a ``challenger'' item, then flipping a weighted coin to decide which of the two to sample \citep{russo2020}. Various methods have been proposed for choosing the leader and challenger \citep{Jourdan2022}, for example, based on Thompson sampling \citep{russo2020} or expected improvement scores \citep{QinKlabjan2017}. For $k=1$ and a broad class of sampling distributions, \citet{russo2020} proves that with an unbiased coin, the posterior probability of an incorrect selection decays at an exponential rate of at least half the best achievable rate. \citet{russo2020} further shows that with suitable tuning of the coin flip probability, it attains the optimal decay rate. \citet{QinKlabjan2017} establish the corresponding results for top two expected improvement with Gaussian samples. \citet{you2023} and \citet{QinYou2026} extend top two methods to the top-$k$ setting and give a rigorous way to weight the coin such that in the top-$1$ Gaussian setting, the posterior probability of an incorrect selection decays at the optimal exponential rate without a tuning parameter, though their general top-$k$ sampling rules do not come with posterior error decay rate guarantees. Establishing optimality for general $k$ is more difficult because the conditions governing how samples must be allocated among items have a more complicated structure \citep{ZhangPengZhangZhou2023,QinYou2026}. \citet{russo2020} conjectures that his top-$k$ extension of top two Thompson sampling achieves a posterior error decay rate of at least half the optimum.

Given the rich literature on this problem, our motivation for this work stems from a related project. In \citet{GoldbergerRudi2026}, we study the top-$k$ identification problem in a \emph{different} setting with three key components that existing methods do not address in combination. First, the decision maker can choose among different binary sample forms (in particular, a binary rating of an item or a comparison of two items), each with a different cost and signal quality. Methods that allocate directly from the current state have no clear way to handle this, whereas methods that solve a static problem may, because the costs and signal quality of each sample form can be incorporated into the objective. Second, that paper considers ``small'' budget settings, so analysis in the asymptotic regime is not of primary importance. Third, samples may be collected in batches. Among the methods above, OCBA-style and elimination-based algorithms support batching \citep{jun2016}. However, OCBA algorithms generally rely on asymptotic analysis, and elimination-based algorithms are often outperformed in simulation by other methods such as gap-e and KL-LUCB \citep{bubeck2013,kaufmann2013,ZhangPengZhangZhou2023}. This motivated us to develop a new method for that specific problem. In this paper, building on insights from that work, we develop a simpler, far more general method for Bayesian top-$k$ identification.

\textbf{Our approach.} We take the approach of solving a sequence of tractable static optimization problems, which we develop in two steps. First, we give an upper bound on EOC, structurally different from those in the OCBA literature on EOC, and discuss why this structure has the potential for an improved algorithm. Second, to address the intractability of minimizing this bound as a one-round static problem, we approximate the future posterior distribution after a batch via expected moment matching rather than relying on asymptotic analysis. This idea originates in the EVI and OCBA literature with Gaussian samples by \citet{ChenChen1997}, \citet{HeChick2007}, and \citet{ZhangPengZhangZhou2023}, and we generalize it to other distributions. These two steps yield a simple, interpretable per-round static surrogate objective. A key property of the surrogate is that it is separable across items, enabling efficient algorithms for solving the batch problem without relying on heuristics.

\textbf{Results and theoretical guarantees.} Similar to the top two and modified expected-improvement methods for $k=1$, our method is conceptually simple and intuitive, yet remarkably comes with strong theoretical performance guarantees for general $k$. For a broad class of sampling distributions, we prove that posterior EOC (as well as the posterior probability of an incorrect selection) decays at an exponential rate of at least half the best achievable rate. This establishes the type of general-$k$ guarantee conjectured by \citet{russo2020}, but for our sampling rule.

In a range of simulations, we demonstrate that our method is competitive with and often improves on the current state-of-the-art approaches, including those specialized to $k = 1$ as well as those for general $k > 1$, those specialized to Gaussian as well as other distributions, and both at small and large budgets. These comparisons are conducted in the non-batched setting to include as many existing methods as possible. In the batched setting, we show that our method significantly outperforms existing batch algorithms and its performance degrades less as the batch size grows.

We believe this method can be valuable to both researchers and practitioners. It is flexible across a wide range of practical settings, is competitive with the best algorithms in each of these settings, and extends naturally to richer settings, such as those with varying sample forms and sampling costs.

The rest of the paper is organized as follows: \Cref{sec:model} formulates the model and objective, \Cref{sec:bound} derives the surrogate objective and the batch allocation algorithm, \Cref{sec:asymptotic} establishes the theoretical guarantees, \Cref{sec:experiments} presents the numerical experiments, and \cref{sec:discussion} concludes.

\section{Model and objective}\label{sec:model}
We consider $I$ items, indexed $i=1,\ldots,I$, each with an unknown mean $\theta_i\in\cJ$, where $\cJ\subseteq\R$ is an open interval. The goal is to select the $k<I$ items with the largest means. Each item $i$ has a prior distribution $\Pi_{i,0}$ on $\theta_i$ with density $\pi_{i,0}$. For item $i$, let $\cY_i$ be the outcome space, and let $\{P_i(\cdot\mid\theta):\theta\in\cJ\}$ be a known mean-parametrized family of distributions on $\cY_i$ with density or mass function $p_i(\cdot\mid\theta)$. Conditional on $\theta_i$, the $n$th observation $Y_{i,n}$ from item $i$ is drawn i.i.d.\ from $P_i(\cdot\mid\theta_i)$, with $\E[Y_{i,n}\mid\theta_i]=\theta_i$. We make the following regularity assumption.

\begin{assumption}[Regular one-parameter Bayesian model]\label{ass:model}
For each $i=1,\ldots,I$, the prior $\Pi_{i,0}$ and sampling family $\{P_i(\cdot\mid\theta):\theta\in\cJ\}$ are such that for every $n\ge 0$ and every $(y_1,\ldots,y_n)\in\cY_i^n$ the posterior of $\theta_i$ conditioned on $\{Y_{i,1}=y_1,\ldots,Y_{i,n}=y_n\}$ has continuous positive density on $\cJ$ and finite second moment, and the corresponding posterior predictive distribution of $Y_{i,n+1}$ has finite second moment. In addition, the priors and observations are independent across items.
\end{assumption}
\cref{rem:examples} gives examples of standard conjugate pairs (ESCP) satisfying \cref{ass:model} that we will refer to throughout the paper.
\begin{remark}[ESCP]\label{rem:examples}
ESCP covered by \cref{ass:model}:
\begin{itemize}[leftmargin=*]
\item[--] Gaussian: $Y_{i,n}\mid\theta_i\sim N(\theta_i,\sigma_i^2)$ with $\sigma_i^2$ known, \; prior $\theta_i\sim N(m_i(0),s_i^2(0))$, \; $\cJ=\R$.
\item[--] Bernoulli: $Y_{i,n}\mid\theta_i\sim\mathrm{Bern}(\theta_i)$, \; prior $\theta_i\sim\mathrm{Beta}(a_i(0),b_i(0))$, \; $\cJ=(0,1)$.
\item[--] Poisson: $Y_{i,n}\mid\theta_i\sim\mathrm{Pois}(\theta_i)$,\; prior $\theta_i\sim\mathrm{Gamma}(a_i(0),b_i(0))$ with rate $b_i(0)$, \; $\cJ=(0,\infty)$.
  \item[--] Gamma with known shape: $Y_{i,n}\mid\theta_i\sim\mathrm{Gamma}(\kappa_i,\theta_i/\kappa_i)$ with scale $\theta_i/\kappa_i$ and $\kappa_i$ known ($\kappa_i = 1$ gives the exponential distribution), \; prior $\theta_i\sim\mathrm{InvGamma}(a_i(0),b_i(0))$ with $a_i(0)>2,\,b_i(0)>0$, \; $\cJ=(0,\infty)$.

\item[--] Lognormal with known log-variance: $Y_{i,n}\mid\theta_i\sim\mathrm{Lognormal}(\log\theta_i\!-\!\sigma_i^2/2,\,\sigma_i^2)$ with $\sigma_i^2$ known, \; prior $\log\theta_i\sim N(m_i(0),s_i^2(0))$, \; $\cJ=(0,\infty)$.
\end{itemize}
\end{remark}
Sampling proceeds in $T\!+\!1$ rounds, an optional warm-up round indexed $t=0$ followed by $T$ adaptive rounds. In the warm-up round, each item is sampled $n_0\in\mathbb{Z}_{\ge0}$ times. In rounds $t=1,\ldots,T$, we choose a batch of $B_t\in\mathbb{Z}_+$ samples, where the per-round batch sizes $\{B_t\}_{t=1}^T$ are pre-specified. The choice of samples in round $t$ is denoted by the vector $\boldsymbol\Delta^{t}=(\Delta_1^{t},\ldots,\Delta_I^{t})$,  where $\Delta_i^{t}\in\mathbb Z_{\ge0}$ is the number of samples assigned to item $i$. Let $N_i(t)$ denote the total samples allocated to item $i$ up to but not including round $t$, so that  $N_i(0)=0$, $N_i(1)=n_0$, and
\begin{equation*}
    N_i(t+1)=N_i(t)+\Delta_i^{t}, \qquad t=1,\ldots,T.
\end{equation*}
Let $\cH_t$ denote the history at time $t$, which is the collection of priors $\{ \Pi_{i,0}\}_{i=1}^I$ and observations $\{Y_{i,n}:1\le i\le I,\,1\le n\le N_i(t)\}$. We write $\E_t[\cdot]$ and $\Pp_t(\cdot)$ for conditional expectation and probability given $\cH_t$. Let $\Pi_{i,t}$ denote the posterior distribution of $\theta_i$ given $\cH_t$ and define $\mu_i(t):=\E_t[\theta_i], v_i(t):=\Var_t(\theta_i), F_{i,t}(x):=\Pp_t(\theta_i\le x).$  In round $t$, we refer to $\Pi_{i,t}$ as the current prior.

In any round $t$, if we were to select items based on the current prior, we would pick the $k$ items with the largest posterior means. Thus, it is convenient to define the following: for any vector $\mathbf{x}\in\R^I$, let $S_k(\mathbf{x})$ denote the set of indices of the $k$ largest elements of $\mathbf{x}$. Write $\boldsymbol\mu(t)=(\mu_1(t),\ldots,\mu_I(t))$ and $\boldsymbol\theta=(\theta_1,\ldots,\theta_I)$. For any $\boldsymbol\theta\in\R^I$, the opportunity cost of the current selection $S_k(\boldsymbol\mu(t))$ is
\begin{equation*}
    \mathrm{OC}(S_k(\boldsymbol\mu(t)),\boldsymbol\theta)
    :=
    \sum_{i\in S_k(\boldsymbol\theta)}\theta_i
    -
    \sum_{i\in S_k(\boldsymbol\mu(t))}\theta_i.
\end{equation*}
Thus the Bayesian expected opportunity cost given $\cH_t$ is
\begin{equation}\label{eq:EOC}
    \mathrm{EOC}(\cH_t)
    =
    \E_{\boldsymbol\theta\mid \cH_t}
    \bigl[\mathrm{OC}(S_k(\boldsymbol\mu(t)),\boldsymbol\theta)\bigr].
\end{equation}
The decision maker's problem is to sequentially choose the sample batches $\{\boldsymbol\Delta^{t}\}_{t=1}^T$ that minimize the terminal expected opportunity cost $\mathrm{EOC}(\cH_{T+1})$. This can be formulated as a dynamic program, but it is intractable to solve, as discussed in the introduction. We next formulate a \emph{static} surrogate objective that guides the choice of batch in each round.

\section{Surrogate objective via EOC bound}\label{sec:bound}
$\mathrm{EOC}(\cH_t)$ does not generally admit an expression amenable to analysis, so \eqref{eq:EOC} is not immediately useful for choosing which batch of samples to collect. Thus, we construct the per-round surrogate objective in two steps. First, \cref{prop:eoc-bound} gives an upper bound on $\mathrm{EOC}(\cH_t)$ that is amenable to analysis. Second, in \cref{subsec:forward}, we make this bound a function of the round-$t$ allocation to guide the choice of future samples. We then discuss how the resulting per-round objective differs from those in prior work in \cref{subsec:related-bounds} and give efficient algorithms for finding the batch that minimizes it in \cref{subsec:batch}. \Cref{subsec:sequential} presents the sequential sampling procedure, and \cref{sec:examples} gives explicit expressions of the sampling rule for the ESCP. We begin with the EOC upper bound, which is based on comparing each item's $\theta_i$ to a threshold.

\begin{lemma}[Threshold upper bound on EOC]\label{prop:eoc-bound}
\begin{enumerate}[label=\textnormal{(\roman*)}]
\item For every threshold $q\in\R$,
\begin{equation*}
    \mathrm{EOC}(\cH_t)
    \le
    \E_t\FitDelims{[}{]}{
        \sum_{i\in S_k(\boldsymbol\mu(t))}(q-\theta_i)^+
        +
        \sum_{i\notin S_k(\boldsymbol\mu(t))}(\theta_i-q)^+
    }.
\end{equation*}
\item The right-hand side of \textnormal{(i)} is convex in $q$ and has a unique minimizer $q^*(t)$ characterized by
\begin{equation}\label{eq:q-star}
    \E_t\bigl[|\{i:\theta_i\ge q^*(t)\}|\bigr]
    =
    \sum_{i=1}^I \{1-F_{i,t}(q^*(t))\}
    =k.
\end{equation}
\end{enumerate}
\end{lemma}
\proof{Proof}All proofs are provided in the Appendix.\origendproof

\noindent Part (i) gives a family of upper bounds on EOC, each separable across items. For an item in the current top-$k$ set, its corresponding term in the bound is the expected shortfall of its true mean below threshold $q$. For an item outside this set, its corresponding term is the expected amount its true mean exceeds the threshold. Part (ii) states that the tightest bound in this family is given by $q^*(t)$, the threshold such that, in expectation, exactly $k$ item means are larger than it. This makes $q^*(t)$ a natural threshold choice when constructing the surrogate objective. We also consider another top-$k$ threshold motivated by the theoretical analysis in \cref{sec:asymptotic}. Write $d_i(x,y):=\mathrm{KL}(P_i(\cdot\mid x)\,\|\,P_i(\cdot\mid y))$ for the Kullback--Leibler (KL) divergence between the sampling distributions of item $i$ at means $x$ and $y$, and let $\mu_{[1]}(t)\ge\cdots\ge\mu_{[I]}(t)$ denote the ordered posterior means. We define the threshold
\begin{equation}\label{eq:q-dagger}
q^\dagger(t):=
\begin{cases}
\displaystyle
\argmin_{\mu_{[k+1]}(t)<q<\mu_{[k]}(t)}
\sum_{i=1}^I\frac{1}{d_i(\mu_i(t),q)},
& \mu_{[k]}(t)>\mu_{[k+1]}(t),\\[10pt]
\mu_{[k]}(t),
& \mu_{[k]}(t)=\mu_{[k+1]}(t).
\end{cases}
\end{equation}
\Cref{sec:asymptotic} explains why this choice of top-$k$ threshold is ``good''. In what follows, let $q(t)$ denote the top-$k$ threshold used in round $t$, which can be either $q^*(t)$ or $q^\dagger(t)$.

\subsection{Forward-looking objective}\label{subsec:forward}

Given $q(t)$, the threshold bound at state $\cH_t$ is
\begin{equation}\label{eq:upperq}
    \E_t\FitDelims{[}{]}{
        \sum_{i\in S_k(\boldsymbol\mu(t))}(q(t)-\theta_i)^+
        +
        \sum_{i\notin S_k(\boldsymbol\mu(t))}(\theta_i-q(t))^+
    }.
\end{equation}
To make this a function of a proposed round-$t$ allocation $\boldsymbol\Delta=(\Delta_1,\ldots,\Delta_I)$, we take the expectation with respect to an approximation of the random posterior of $\theta_i$ \emph{after} allocating $\boldsymbol\Delta$ and observing its outcomes. This approximation preserves the shape of $\Pi_{i,t}$ and has mean and variance equal to the expected posterior mean and variance. We say two distributions have the same \emph{shape} if they are equivalent after centering by their mean and scaling by their standard deviation. Put together, these conditions say that we replace the random future posterior with a deterministic, ``shrunken” copy of the current prior, as it has the same mean as the prior, keeps the same shape, and changes only through a contraction around its mean by the expected amount of uncertainty that the proposed samples will remove. \Cref{fig:surrogate} gives an example.

\begin{figure}[t]
\centering
\includegraphics[width=0.95\textwidth]{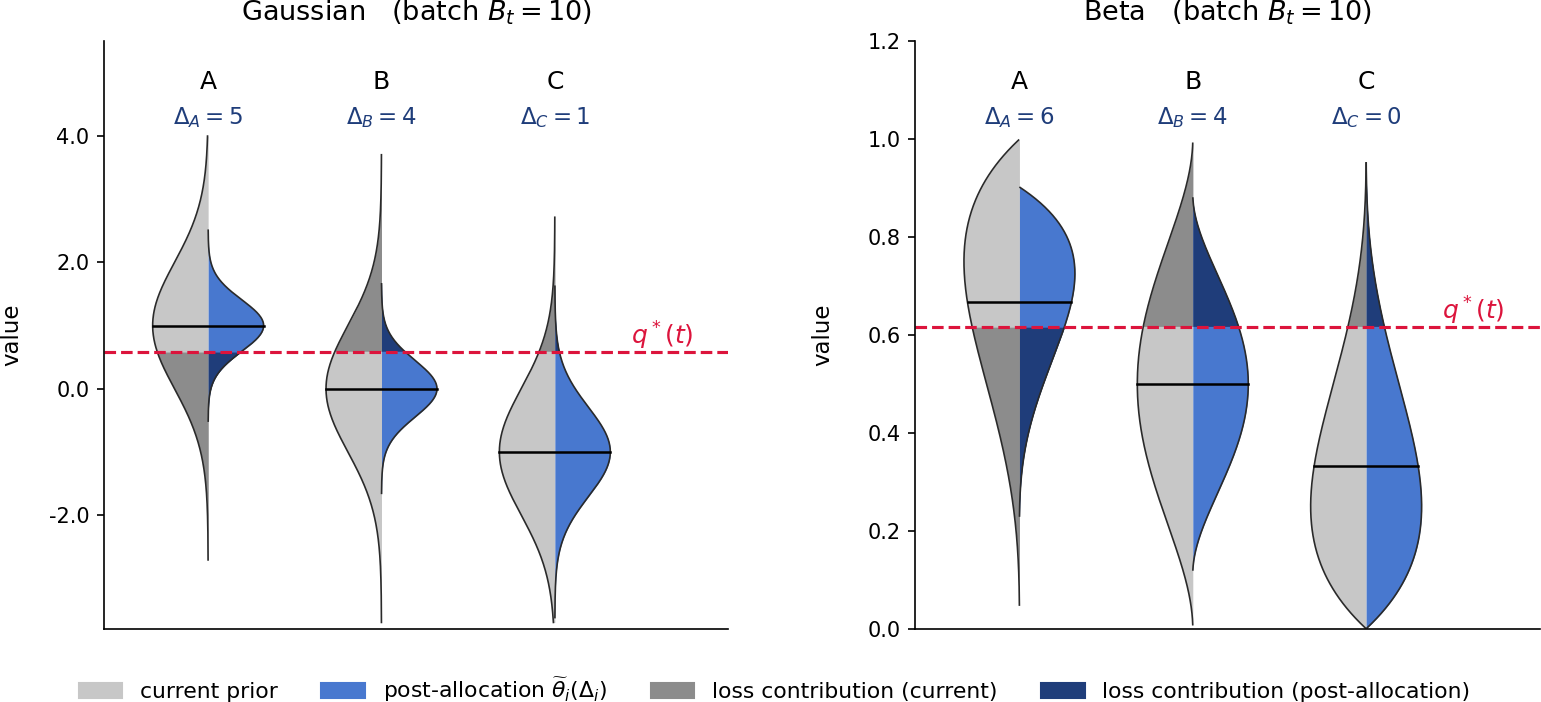}
\caption{Structure of the surrogate loss $\cL_t(\cdot)$ and allocation for $I=3$, $k=1$ using $q(t)=q^*(t)$. For each item $i\in\{A,B,C\}$, the left distribution is its current prior for $\theta_i$, and the right is the posterior moment-matched distribution $\widetilde\theta_i(\Delta_i)$. The allocation shown minimizes \eqref{eq:Pt} with $B_t=10$. The shaded area is the item's contribution to the loss.}
\label{fig:surrogate}
\vspace{-5pt}
\end{figure}
Formally, we replace each $\theta_i$ in \eqref{eq:upperq} by the random variable $\widetilde\theta_i(\Delta_i)$ that satisfies
\begin{enumerate}[label=\textnormal{(\roman*)}]
\item $\frac{\widetilde\theta_i(\Delta_i)-\E_t[\widetilde\theta_i(\Delta_i)]}{\sqrt{\Var_t(\widetilde\theta_i(\Delta_i))}}\overset{d}{=}\frac{\theta_i-\mu_i(t)}{\sqrt{v_i(t)}}$,
\item $\E_t[\widetilde\theta_i(\Delta_i)]=\E[\mu_i(t\!+\!1)\mid\boldsymbol\Delta,\cH_t]=\mu_i(t)$,
\item $\Var_t(\widetilde\theta_i(\Delta_i))=\E[\Var(\theta_i\mid\cH_{t+1})\mid\boldsymbol\Delta,\cH_t]=:\bar v_i(t,\Delta_i)$.
\end{enumerate}
For a given $\cH_t$, this distribution is unique and is given by
\begin{equation}\label{eq:affine}
    \widetilde\theta_i(\Delta_i)
    =\mu_i(t)+c_i(t,\Delta_i)(\theta_i-\mu_i(t)),
    \qquad
    c_i(t,\Delta_i):=\sqrt{\frac{\bar v_i(t,\Delta_i)}{v_i(t)}}.
\end{equation}
We then define the surrogate loss induced by the allocation $\boldsymbol\Delta$ as
\begin{align}
    \cL_t(\boldsymbol\Delta)
    &:=\E_t\!\FitDelims{[}{]}{
        \sum_{i\in S_k(\boldsymbol\mu(t))}(q(t)-\widetilde\theta_i(\Delta_i))^+
        +
        \sum_{i\notin S_k(\boldsymbol\mu(t))}(\widetilde\theta_i(\Delta_i)-q(t))^+
    }\nonumber\\
    &=\sum_{i=1}^I\E_t[(q(t)-\widetilde\theta_i(\Delta_i))^+]
    +\sum_{i\notin S_k(\boldsymbol\mu(t))}(\mu_i(t)-q(t)),\label{eq:Lt-general}
\end{align}
where the second equality follows from $(x\!-\!q)^+=(x\!-\!q)\!+\!(q\!-\!x)^+$. The second sum is constant with respect to $\boldsymbol\Delta$. For the first sum we have
\begin{equation}\label{eq:ell-general}
    \ell_{i,t}(\Delta_i)
    :=\E_t[(q(t)-\widetilde\theta_i(\Delta_i))^+]
    =c_i(t,\Delta_i)\,\E_t[(z_i(t,\Delta_i)-\theta_i)^+],
\end{equation}
where $z_i(t,\Delta_i):=\mu_i(t)\!+\!(q(t)\!-\!\mu_i(t))/c_i(t,\Delta_i)$. Choosing the batch that minimizes \eqref{eq:Lt-general} is thus equivalent to
\begin{equation}\label{eq:Pt}
    \min_{\boldsymbol\Delta\in\mathbb Z_{\ge0}^I}
    \sum_{i=1}^I\ell_{i,t}(\Delta_i)
    \qquad\text{s.t.}\qquad
    \sum_{i=1}^I\Delta_i=B_t,
\end{equation}
which is the per-round surrogate program. Conceptually, allocating according to \eqref{eq:Pt} sends samples to the items whose posterior has substantial mass on both sides of the top-$k$ threshold $q(t)$, weighted by the distance of that mass from $q(t)$. These are the items whose membership in the top-$k$ set is uncertain and whose misclassification, if it occurs, could be costly.

\subsection{Relation to OCBA-style objectives}\label{subsec:related-bounds}

Before analyzing how to solve \eqref{eq:Pt}, we discuss how its structure differs from static objectives in the OCBA literature for EOC. These objectives often use an upper bound on EOC based on a sum of losses associated with pairwise misorderings. For example, for top-$1$ EOC with Gaussian samples, \citet{HeChick2007} use
\vspace{-3pt}
\begin{equation*}
    \mathrm{EOC}(\cH_t)
    \le\sum_{i\neq i_{(1)}}\E_t[(\theta_i-\theta_{i_{(1)}})^+].
\end{equation*}
\vspace{-4pt}
\citet{GaoChen2015} give the following bound on top-$k$ \emph{frequentist} EOC:
\begin{equation*}
    \mathrm{EOC}\le
    \sum_{j=1}^{k}\sum_{i\notin\{i_{(1)},\ldots,i_{(j)}\}}\sum_{n=1}^{j}
    (\mu_{i_{(n)}}(t)-\mu_i(t))
    \Phi\!\FitParens{-\frac{\mu_{i_{(n)}}(t)-\mu_i(t)}{\sqrt{v_{i_{(n)}}(t)+v_i(t)}}},
\end{equation*}
where $\Phi$ is the standard normal CDF, and $i_{(1)},\ldots,i_{(I)}$ order the items by their current posterior means. Their static objective is then based on driving this bound down. The potential issue with using such a bound is that the candidate top-$k$ items $\{i_{(1)},\ldots,i_{(k)}\}$ are in far more terms in the sum than the non-top-$k$ items. For instance, when $k=1$, $i_{(1)}$ is in all $I\!-\!1$ terms, while all other items are only in one term. Thus, an objective based on such a bound may over-count the benefit of sampling the top items. Our bound avoids these pairwise terms by comparing each item to the same threshold $q(t)$, resulting in a single term for each item regardless of whether or not it is in $S_k(\boldsymbol\mu(t))$. \cref{fig:alloc} shows the distribution of sample allocation across items for our method, called TEAk, and the method of \citet{GaoChen2015}, called EOC-m. EOC-m systematically allocates more samples to items in the top-$k$, even items far above the top-$k$ boundary, which is likely one of the main factors driving our method's improvement over EOC-m, which we demonstrate in \cref{sec:experiments}.

\begin{figure}[t]
\centering
\includegraphics[width=0.95\textwidth]{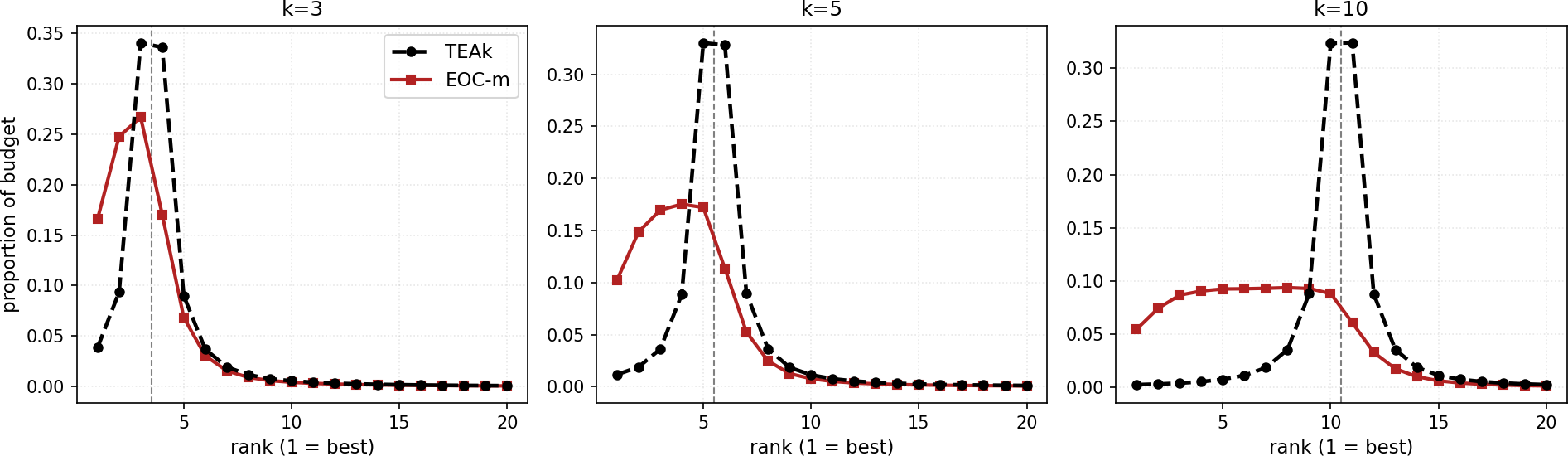}
\caption{Budget allocation across items on the Gaussian instance described in \cref{sec:experiments}, with $I=20$ at $k=3,5,10$. The plots show the proportion of the total budget spent on each item, with items ordered by their true mean. TEAk uses $q^*(t)$. The dashed line marks the top-$k$ boundary.}
\label{fig:alloc}
\vspace{-20pt}
\end{figure}

\subsection{Batch allocation algorithm}\label{subsec:batch}

Although \eqref{eq:Pt} is an integer program, its separability across items allows us to solve it efficiently by dynamic programming and, in some cases, using a greedy procedure. Both are standard techniques. The main technical result of this subsection is \cref{prop:greedy-conditions}(ii)--(iii), which gives conditions on the posterior distributions under which the greedy procedure is optimal.

The dynamic programming procedure in \cref{alg:batch-dp} has two phases. First, we compute $\ell_{i,t}(d)$ for every item $i$ and candidate allocation $d=0,\ldots,B_t$. Second, for each $r=1,\ldots,I$ and $b=0,\ldots,B_t$, we solve the subproblem of allocating exactly $b$ samples among the first $r$ items by forward induction, and then recover the optimal allocation by backtracking. Let $V_r(b)$ denote the minimum loss for this subproblem so that $V_I(B_t)$ is the optimal value of \eqref{eq:Pt}.
\vspace{-2pt}
\begin{algorithm}[!htb]
\caption{Dynamic Programming Batch Builder (DPBB)}
\label{alg:batch-dp}
\begin{algorithmic}[1]
\STATE \textbf{Precompute cost values:} for each $i=1,\ldots,I$ and $d=0,\ldots,B_t$, compute $\ell_{i,t}(d)$
\STATE \textbf{Initialize:} set $V_0(0)=0$ and $V_0(b)=+\infty$ for $b=1,\ldots,B_t$
\STATE \textbf{Solve subproblems:}
\FOR{$r=1,\ldots,I$}
    \FOR{$b=0,\ldots,B_t$}
        \STATE Compute $V_r(b)=\min_{\Delta\in\{0,\ldots,b\}}\{V_{r-1}(b\!-\!\Delta)\!+\!\ell_{r,t}(\Delta)\}$ and record an argmin $\Delta_r^*(b)$
    \ENDFOR
\ENDFOR
\STATE \textbf{Recover the allocation:} set $b\leftarrow B_t$
\FOR{$r=I,\ldots,1$}
    \STATE Set $\Delta_r^t=\Delta_r^*(b)$ and update $b\leftarrow b\!-\!\Delta_r^t$
\ENDFOR
\STATE \textbf{Return} $\boldsymbol\Delta^t=(\Delta_1^t,\ldots,\Delta_I^t)$
\end{algorithmic}
\end{algorithm}
\vspace{-15pt}
\begin{proposition}[Optimality of DPBB]\label{prop:dp-exact}
DPBB returns a global optimizer of \eqref{eq:Pt}. Let $C_\ell$ denote the cost of computing each value $\ell_{i,t}(d)$. The total computational cost is $O(IB_tC_\ell\!+\!IB_t^2)$.
\end{proposition}
For the greedy procedure, define the marginal decrease from allocating one more sample to item $i$ by $D_{i,t}(\Delta_i):=\ell_{i,t}(\Delta_i)\!-\!\ell_{i,t}(\Delta_i\!+\!1).$ The procedure in \cref{alg:batch-greedy} sequentially assigns the $B_t$ samples to the item with the current largest marginal decrease.
\vspace{-2pt}
\begin{algorithm}[!htb]
\caption{Greedy Batch Builder (GBB)}
\label{alg:batch-greedy}
\begin{algorithmic}[1]
\STATE \textbf{Initialize:} $\boldsymbol\Delta^t=\mathbf0$
\FOR{$n=1,\ldots,B_t$}
    \STATE For each $i=1,\ldots,I$, compute $D_{i,t}(\Delta_i^t)=\ell_{i,t}(\Delta_i^t)\!-\!\ell_{i,t}(\Delta_i^t\!+\!1)$
    \STATE Choose $i^*\in\argmax_{1\le i\le I}D_{i,t}(\Delta_i^t)$
    \STATE $\Delta_{i^*}^t\leftarrow\Delta_{i^*}^t\!+\!1$
\ENDFOR
\STATE \textbf{Return} $\boldsymbol\Delta^t=(\Delta_1^t,\ldots,\Delta_I^t)$
\end{algorithmic}
\end{algorithm}
\vspace{-15pt}
\begin{proposition}[Conditions for greedy optimality]\label{prop:greedy-conditions}
\begin{enumerate}[label=\textnormal{(\roman*)}]
\item If, for every item $i$,
\vspace{-5pt}
\begin{equation}\label{eq:dim-gains}
    D_{i,t}(0)\ge D_{i,t}(1)\ge\cdots\ge D_{i,t}(B_t-1),
\end{equation}
\vspace{-22pt}

\noindent then GBB returns a global optimizer of \eqref{eq:Pt}. Its computational cost is $O(IB_tC_\ell)$.

\item Suppose that at round $t$, for every item $i$ there exists $\lambda_i(t)>0$ such that
\begin{equation}\label{eq:contraction-form}
    c_i(t,d)^2=\frac{1}{1+\lambda_i(t)d}
    \qquad\text{for every $d\in\mathbb Z_{\ge0}$.}
\end{equation}
Then \eqref{eq:dim-gains} holds, and GBB returns a global optimizer of \eqref{eq:Pt}.

\item Suppose that for every item $i$, the sampling family is a regular one-parameter natural exponential family, meaning it has an open natural parameter space $\mathcal E_i$ and density or mass function $p_i(y\mid\eta_i)=h_i(y)\exp\{\eta_i y\!-\!A_i(\eta_i)\}$, where $\theta_i=A_i'(\eta_i)$. Suppose also that the prior on $\eta_i$ induced by $\Pi_{i,0}$ has density proportional to $\exp\{a_i\eta_i\!-\!b_iA_i(\eta_i)\}$ on $\mathcal E_i$ for some $a_i\in\mathbb R$ and $b_i>0$. Then \eqref{eq:contraction-form} holds at every round, and GBB returns a global optimizer of \eqref{eq:Pt}.
\end{enumerate}
\end{proposition}

\noindent Part (i) says that GBB finds an optimal allocation whenever the gain from each additional sample is nonincreasing. However, checking this condition requires computing all marginal decreases in every round. Part (ii) avoids that computation by giving a sufficient condition for \eqref{eq:dim-gains} based on how the expected posterior variance decreases with additional samples, and part (iii) establishes that the condition holds for regular natural exponential families with conjugate priors.

\begin{remark}[Sampling costs]\label{rem:costs}
In some settings, items may differ in their sampling costs. Let $w_i\in\mathbb Z_{>0}$ denote the cost of one sample of item $i$. The constraint in \eqref{eq:Pt} becomes $\sum_{i=1}^I w_i\Delta_i\le B_t$ and the same dynamic programming approach finds the optimal batch by accounting for these costs in the remaining budget.
\end{remark}

\subsection{Sequential algorithm}\label{subsec:sequential}

\Cref{alg:sequential} gives the full sampling procedure. In each round, we update the threshold from the current prior and allocate samples by solving \eqref{eq:Pt} using one of the batch builders above.
\vspace{-5pt}
\begin{algorithm}[!htb]
\caption{Threshold-EOC allocation for top-$k$ selection (TEAk)}
\label{alg:sequential}
\begin{algorithmic}[1]
\STATE \textbf{Input:} priors $(\Pi_{i,0})_{i=1}^I$, sampling distributions $(P_i(\cdot\mid\theta_i))_{i=1}^I$, warm-up $n_0$, batch sizes $(B_t)_{t=1}^T$
\STATE \textbf{Initialize:} $N_i(0)=0$ for all $i$
\STATE \textbf{Warm-up:} for each $i$, allocate $n_0$ samples, observe outcomes, and update $N_i(1)=n_0$ and $\Pi_{i,1}$
\FOR{$t=1,\ldots,T$}
    \STATE Compute $q(t)$ using \eqref{eq:q-star} or \eqref{eq:q-dagger}
    \STATE Run DPBB to obtain $\boldsymbol\Delta^t$, or GBB when \eqref{eq:dim-gains} holds for every item
    \STATE For each $i$, allocate $\Delta_i^t$ samples and observe outcomes
    \STATE Update $N_i(t\!+\!1)=N_i(t)\!+\!\Delta_i^t$ and $\Pi_{i,t+1}$
\ENDFOR
\STATE \textbf{Return} $S_k(\boldsymbol\mu(T\!+\!1))$
\end{algorithmic}
\end{algorithm}
\vspace{-15pt}
\subsection{ESCP explicit expressions for \cref{alg:sequential}}\label{sec:examples}
This section instantiates \cref{alg:sequential} for the five ESCP in \cref{rem:examples}. Recall from \eqref{eq:ell-general} that
\begin{equation*}
    \ell_{i,t}(\Delta_i) = c_i(t,\Delta_i)\,\E_t\bigl[(z_i(t,\Delta_i) - \theta_i)^+\bigr],
    \qquad
    z_i(t,\Delta_i) := \mu_i(t) + \frac{q(t) - \mu_i(t)}{c_i(t,\Delta_i)}.
\end{equation*}
Computing the round $t$ allocation thus requires explicit formulas for $c_i(t,\Delta_i)$ as well as $\E_t[(z\!-\!\theta_i)^+]$ for each family, which are given in \cref{prop:families}.

\begin{proposition}\label{prop:families}
Fix round $t$ and write $\mu_i:=\mu_i(t)$. For \textnormal{(ii)--(iv)}, write $a_i:=a_i(t)$ and $b_i:=b_i(t)$.
\begin{enumerate}[label=\textnormal{(\roman*)}]
\item \textbf{Gaussian, known variance $\sigma_i^2$.} Posterior $N(m_i,s_i^2)$, $\mu_i=m_i$:
\begin{equation*}
\FamilyRow{c_i(t,\Delta_i)^2=\frac{1}{1+(s_i^2/\sigma_i^2)\Delta_i}}
{\E_t[(z-\theta_i)^+]=(z-m_i)\,\Phi\!\big(\tfrac{z-m_i}{s_i}\big)+s_i\,\phi\!\big(\tfrac{z-m_i}{s_i}\big)}
\end{equation*}
where $\phi$ is the standard normal density.
\item \textbf{Binomial/Bernoulli.} Observations are binomial with success probability $p_i$ and number of trials $r\in\mathbb Z_+$. Posterior $p_i\sim\mathrm{Beta}(a_i,b_i)$, $\mu_i=ra_i/n_i$ where $n_i=a_i\!+\!b_i$:
\begin{equation*}
\FamilyRow{c_i(t,\Delta_i)^2=\frac{1}{1+(r/n_i)\Delta_i}}
{\E_t[(z-\theta_i)^+]=
\begin{cases}
\CaseFirst{0,} & z\le 0,\\
z\,I_{z/r}(a_i,b_i)-\mu_i\,I_{z/r}(a_i+1,b_i), & 0<z<r,\\
z-\mu_i, & z\ge r
\end{cases}}
\end{equation*}
where $I_z(a,b)$ is the regularized incomplete beta function.

\item \textbf{Poisson.} Posterior $\mathrm{Gamma}(a_i,b_i)$ with shape $a_i$ and rate $b_i$, $\mu_i=a_i/b_i$:
\begin{equation*}
\FamilyRow{c_i(t,\Delta_i)^2=\frac{1}{1+\Delta_i/b_i}}
{\E_t[(z-\theta_i)^+]=
\begin{cases}
\CaseFirst{0,} & z\le 0,\\
z\,\gamma(a_i,b_i z)-\mu_i\,\gamma(a_i+1,b_i z), & z>0
\end{cases}}
\end{equation*}
where $\gamma(a,x)$ is the lower regularized incomplete gamma function.
\item \textbf{Gamma, known shape $\kappa_i$.} Posterior $\mathrm{InvGamma}(a_i,b_i)$, $\mu_i=b_i/(a_i\!-\!1)$:
\begin{equation*}
\FamilyRow{c_i(t,\Delta_i)^2=\frac{1}{1+(\kappa_i/(a_i-1))\Delta_i}}
{\E_t[(z-\theta_i)^+]=
\begin{cases}
\CaseFirst{0,} & z\le 0,\\
z\,\Gamma(a_i,b_i/z)-\mu_i\,\Gamma(a_i-1,b_i/z), & z>0
\end{cases}}
\end{equation*}
where $\Gamma(a,x)$ is the upper regularized incomplete gamma function.
\item \textbf{Lognormal.} Posterior lognormal with log-mean $m_i$, log-variance $s_i^2$, $\mu_i=e^{m_i+s_i^2/2}$, and $s_i^2(\Delta_i):=(1/s_i^2\!+\!\Delta_i/\sigma_i^2)^{-1}$:
\begin{equation*}
\FamilyRow{c_i(t,\Delta_i)^2=\frac{1-e^{-s_i^2(\Delta_i)}}{1-e^{-s_i^2}}}
{\E_t[(z-\theta_i)^+]=
\begin{cases}
\CaseFirst{0,} & z\le 0,\\
z\,\Phi\!\big(\tfrac{\log z-m_i}{s_i}\big)-\mu_i\,\Phi\!\big(\tfrac{\log z-m_i-s_i^2}{s_i}\big), & z>0
\end{cases}.}
\end{equation*}
\end{enumerate}
For \textnormal{(i)--(iv)}, $c_i(t,\Delta_i)^2$ has the form \eqref{eq:contraction-form}, so GBB solves \eqref{eq:Pt} exactly by \cref{prop:greedy-conditions}(ii). For \textnormal{(v)}, $c_i(t,\Delta_i)^2$  does not have this form in general.
\end{proposition}

\section{Theoretical guarantees}\label{sec:asymptotic}

The guarantees in \cref{sec:bound} concern the optimization of the surrogate objective \eqref{eq:Pt}, not the performance of TEAk in reducing EOC. In this section, we give asymptotic performance guarantees for TEAk. We begin by showing that TEAk is consistent ($\mathrm{EOC}(\cH_t)$ converges to zero) and then make additional distributional assumptions to obtain a much stronger result. In particular, with threshold $q^\dagger(t)$, we show that the asymptotic exponential rate at which $\mathrm{EOC}(\cH_t)$ converges to zero, referred to as the posterior convergence rate, is at least half of the best rate achievable by any sequential sampling rule.

Throughout this section, we fix a true mean vector $\boldsymbol\theta^0\in\cJ^I$. We write $\Pp_{\boldsymbol\theta^0}$ and $\E_{\boldsymbol\theta^0}$ for probability and expectation over observations generated from $P_i(\cdot\mid\theta_i^0)$. All convergence statements hold $\Pp_{\boldsymbol\theta^0}$-almost surely. Although the true means are fixed, $\mathrm{EOC}(\cH_t)$ remains the Bayesian posterior expectation in \eqref{eq:EOC}. Our rate guarantees also hold for the posterior error probability $\Pp_t(S_k(\boldsymbol\theta)\ne S_k(\boldsymbol\theta^0))$, the quantity studied in the $k=1$ analyses of \citet{russo2020,QinKlabjan2017}, and \citet{Shang2020}.

These posterior guarantees should not be confused with guarantees for the frequentist EOC $\E_{\boldsymbol\theta^0}[ \mathrm{OC}(S_k(\boldsymbol\mu(t)),\boldsymbol\theta^0)]$ and the frequentist probability of incorrect selection $\Pp_{\boldsymbol\theta^0}( S_k(\boldsymbol\mu(t))\ne S_k(\boldsymbol\theta^0))$. The posterior quantities measure the EOC and posterior error probability \emph{conditional} on an observed history, while the frequentist quantities average over observation histories. This distinction matters for exponential-rate guarantees: a sampling rule that achieves the optimal posterior convergence rate need not achieve the corresponding frequentist rate. In fact, no single sampling rule can generally guarantee a constant fraction of the best exponential decay rate of the frequentist probability of incorrect selection (likewise frequentist expected opportunity cost) uniformly across problem instances \citep{Goldberger2026,degenne2023existence}.

\subsection{Consistency}\label{subsec:consistency}

\begin{proposition}[Consistency]\label{prop:eoc-bpcs-consistency}
Run TEAk for infinitely many rounds with threshold $q^*(t)$ or $q^\dagger(t)$. For $q^\dagger(t)$, assume that the minimum in \eqref{eq:q-dagger} is attained whenever $\mu_{[k]}(t)>\mu_{[k+1]}(t)$. There is a set $\Theta^*\subseteq\cJ^I$ with prior probability one such that for every $\boldsymbol\theta^0\in\Theta^*$, the following hold as $t\to\infty$:
\begin{enumerate}[label=\textnormal{(\roman*)}]
\item $N_i(t)\to\infty$, $\mu_i(t)\to\theta_i^0$, and $v_i(t)\to0$ for every $i$.
\item $\mathrm{EOC}(\cH_t)\to0$.
\item If $\theta^0_{[k]}>\theta^0_{[k+1]}$, then $S_k(\boldsymbol\mu(t))=S_k(\boldsymbol\theta^0)$ for all sufficiently large $t$.
\end{enumerate}
\end{proposition}
Part (i) says that every item is sampled infinitely often, and consequently its posterior mean converges to the true mean, and its posterior variance to $0$. Parts (ii) and (iii) then follow: the posterior EOC converges to $0$ and, when the true top-$k$ gap is positive, there is a finite time after which $S_k(\boldsymbol\mu(t))=S_k(\boldsymbol\theta^0)$.

\subsection{The optimal exponential rate}\label{subsec:optimal-eoc-rate}
For the remainder of this section, assume $1\le B_t\le B$ for every $t$ where $B$ is a fixed positive integer, and assume $\theta_1^0\ge\cdots\ge\theta_k^0>\theta_{k+1}^0\ge\cdots\ge\theta_I^0$. We next establish an upper bound on the posterior convergence rate that any sequential sampling rule can achieve, and characterize the sample allocation that attains this bound. This gives the benchmark against which we will evaluate TEAk. We rely on the general posterior large-deviation analysis of \citet{russo2020}, which we then apply to our top-$k$ EOC objective. Accordingly, we follow the distributional assumptions of \citet{russo2020}.

\begin{assumption}[Bounded exponential-family model]\label{ass:rate-model}
In addition to \cref{ass:model}, $\cJ=(L,U)$ with $-\infty<L<U<\infty$. For every item $i$, the likelihood has the form
\begin{equation*}
    p_i(y\mid x)=h_i(y)\exp\{\eta_i(x)T_i(y)-A_i(\eta_i(x))\},
    \qquad x\in\cJ,
\end{equation*}
where $T_i$ is strictly increasing and $A_i$ is twice differentiable. Write $\mathcal E_i:=\eta_i(\cJ)=(\underline\eta_i,\overline\eta_i)$, with $-\infty<\underline\eta_i<\overline\eta_i<\infty$, and let $\widetilde\pi_{i,0}$ denote the prior density of $\eta_i(\theta_i)$. We assume
\begin{equation*}
    \sup_{\eta\in\mathcal E_i}|A_i'(\eta)|<\infty,
    \qquad
    0<\inf_{\eta\in\mathcal E_i}\widetilde\pi_{i,0}(\eta)
    \le\sup_{\eta\in\mathcal E_i}\widetilde\pi_{i,0}(\eta)<\infty.
\end{equation*}
\end{assumption}
The boundedness assumptions are used in the posterior convergence analysis primarily to facilitate generality across this broad range of families. As \citet[p.~1635]{russo2020} notes, ``Analogous results are easily established for certain unbounded conjugate priors.'' All five ESCP in \cref{rem:examples} satisfy \cref{ass:rate-model} when their priors are restricted and normalized to a bounded interval whose closure is contained in $\cJ$.

The next result shows that the exponential decay rate of a posterior probability is determined by the empirical sampling proportions. Let $\Delta_I:=\{\boldsymbol w\in\R^I:w_i\ge0,\ \sum_iw_i=1\}$, and define the empirical sampling proportions by $\widehat w_i(t):=N_i(t)/\sum_jN_j(t)$ for $t\ge2$.

\begin{lemma}[{\citet[Proposition~5]{russo2020}}]\label{lem:posterior-open-set}
Under any sequential sampling rule, for every nonempty open set $O\subseteq\cJ^I$,
\begin{equation}\label{eq:posterior-open-set-rate}
    -\frac{\log\Pp_t(\boldsymbol\theta\in O)}{\sum_iN_i(t)}
    -\inf_{\boldsymbol x\in O}\sum_{i=1}^I\widehat w_i(t)d_i(\theta_i^0,x_i)
    \longrightarrow0.
\end{equation}
\end{lemma}
To interpret the first term in \eqref{eq:posterior-open-set-rate}, $\{\bm{\theta} \in O\}$ is the ``bad'' event whose probability we want to control. For top-$k$ selection, $O$ is the set of \emph{alternative} parameters: parameters whose top-$k$ set is different from $S_k(\boldsymbol\theta^0)$. Ignoring ties which have a posterior probability of zero, these are the parameters where some pair of items $(i,j)$ satisfies $i\le k < j$, and yet $\theta_j > \theta_i$. Accordingly, define the alternative space
\begin{equation*}
    \operatorname{Alt}(\boldsymbol\theta^0)
    :=\bigcup_{i\le k<j}
    \{\boldsymbol x\in\cJ^I:x_j>x_i\}
\end{equation*}
so that $\Pp_t(\boldsymbol\theta\in\operatorname{Alt}(\boldsymbol\theta^0)) =\Pp_t(S_k(\boldsymbol\theta)\ne S_k(\boldsymbol\theta^0))$. We will apply \cref{lem:posterior-open-set} with $O=\operatorname{Alt}(\boldsymbol\theta^0)$. In the second term, for fixed sampling proportions $\boldsymbol w$, the quantity $\sum_iw_id_i(\theta_i^0,x_i)$ is the expected log-likelihood ratio per observation between the true parameter $\boldsymbol\theta^0$ and an alternative $\boldsymbol x$. The infimum over $O$ therefore identifies the alternative parameter that this allocation provides the least evidence to distinguish from the true parameter. For $\boldsymbol w\in\Delta_I$, denote this infimum at $O=\operatorname{Alt}(\boldsymbol\theta^0)$ by $\Gamma(\boldsymbol w;\boldsymbol\theta^0)$. \cref{lem:pairwise-rate} characterizes its structure.
\begin{lemma}[Pairwise representation of $\Gamma$]\label{lem:pairwise-rate}
For every $\boldsymbol w\in\Delta_I$,
\begin{equation}\label{eq:posterior-rate}
    \Gamma(\boldsymbol w;\boldsymbol\theta^0)
    :=\inf_{\boldsymbol x\in\operatorname{Alt}(\boldsymbol\theta^0)}
    \sum_{\ell=1}^I w_\ell d_\ell(\theta_\ell^0,x_\ell)
    =\min_{i\le k<j}
    \min_{x\in[\theta_j^0,\theta_i^0]}
    \{w_id_i(\theta_i^0,x)+w_jd_j(\theta_j^0,x)\}.
\end{equation}
\end{lemma}
\cref{lem:pairwise-rate} says that $\Gamma(\boldsymbol w;\boldsymbol\theta^0)$ is determined by the pair of items satisfying $i\le k<j$ whose relative ranking is hardest to learn under the sampling design $\boldsymbol w$. Maximizing over $\boldsymbol w$ gives
\begin{equation}\label{eq:optimal-posterior-rate}
    \Gamma^*(\boldsymbol\theta^0)
    :=\max_{\boldsymbol w\in\Delta_I}
    \Gamma(\boldsymbol w;\boldsymbol\theta^0).
\end{equation}
The next proposition establishes that $\Gamma^*(\boldsymbol\theta^0)$ is the optimal exponential decay rate of posterior EOC and that it is attained by any sampling rule whose empirical proportions converge to a maximizer of \eqref{eq:optimal-posterior-rate}.

\begin{proposition}[Optimal exponential rate of posterior EOC]\label{prop:optimal-eoc-rate}
\begin{enumerate}[label=\textnormal{(\roman*)}]
\item The maximum in \eqref{eq:optimal-posterior-rate} is attained, and $0<\Gamma^*(\boldsymbol\theta^0)<\infty$.
\item Every sampling rule satisfies
\begin{equation*}
    \limsup_{t\to\infty}
    -\frac{\log\mathrm{EOC}(\cH_t)}{\sum_iN_i(t)}
    \le\Gamma^*(\boldsymbol\theta^0).
\end{equation*}
\item If the sampling proportions converge to a maximizer of \eqref{eq:optimal-posterior-rate}, then
\begin{equation*}
    \lim_{t\to\infty}
    -\frac{\log\mathrm{EOC}(\cH_t)}{\sum_iN_i(t)}
    =\Gamma^*(\boldsymbol\theta^0).
\end{equation*}
\end{enumerate}
\end{proposition}
We next show that, with threshold $q^\dagger(t)$, TEAk's posterior convergence rate is at least $\Gamma^*(\boldsymbol\theta^0)/2$.

\subsection{The rate of TEAk}\label{subsec:threshold-allocation}
To determine the rate TEAk achieves, it is sufficient to characterize the sample proportions produced asymptotically. First, we do this assuming the threshold $q(t)$ converges to some point $q\in(\theta_{k+1}^0,\theta_k^0)$. For such a $q$, define the proportions
\begin{equation}\label{eq:threshold-allocation}
    w_i(q):=
    \frac{1/d_i(\theta_i^0,q)}
    {\sum_{j=1}^I1/d_j(\theta_j^0,q)},
    \qquad i=1,\ldots,I.
\end{equation}

\begin{theorem}[Allocation and rate for a convergent threshold]\label{thm:threshold-allocation}
Assume $q(t)\to q_0\in(\theta_{k+1}^0,\theta_k^0)$. Then
\begin{enumerate}[label=\textnormal{(\roman*)}]
\item $\widehat w_i(t)\to w_i(q_0)$ for every $i$,
\item $-\frac{\log\mathrm{EOC}(\cH_t)}{\sum_iN_i(t)} \to\Gamma(\boldsymbol w(q_0);\boldsymbol\theta^0)$,
\item $\Gamma(\boldsymbol w(q_0);\boldsymbol\theta^0) \ge\bigl(\sum_{i=1}^I1/d_i(\theta_i^0,q_0)\bigr)^{-1}$.
\end{enumerate}
\end{theorem}
Part (i) says that the asymptotic sampling proportions of TEAk are described by \eqref{eq:threshold-allocation}. To prove this, we show that
  \begin{equation*}
      -\frac{\log D_{i,t}(d)}{N_i(t)}
      \longrightarrow d_i(\theta_i^0,q_0)
      \qquad\text{for every $i$ and $d=0,\ldots,B-1$.}
  \end{equation*}
Thus, $d_i(\theta_i^0,q_0)$ gives the exponential decay rate of the marginal decrease in the surrogate per sample of item $i$. Since TEAk allocates samples to the items with the largest marginal decreases (the argument under DPBB is a bit more technical but also holds), it asymptotically samples the items with the smallest $N_i(t)d_i(\theta_i^0,q_0)$, which leads to balancing $N_i(t)d_i(\theta_i^0,q_0)$ across all items. Therefore, the sampling proportions converge to $w_i(q_0)$, which makes $w_i(q_0)d_i(\theta_i^0,q_0)$ equal across items. Part (ii) says that TEAk's posterior convergence rate is then obtained by plugging these proportions into $\Gamma$. Part (iii) gives a lower bound on TEAk's rate, and we establish the one-half guarantee by comparing this lower bound, rather than $\Gamma(\boldsymbol w(q_0);\boldsymbol\theta^0)$ itself, with the optimal rate $\Gamma^*(\boldsymbol\theta^0)$. This lower bound motivates choosing the threshold that makes it tightest,
    \begin{equation*}
        q^\dagger\in
        \argmin_{\theta_{k+1}^0<q<\theta_k^0}
        \sum_{i=1}^I\frac1{d_i(\theta_i^0,q)}.
    \end{equation*}
Replacing the true means with the posterior means results in our choice of threshold $q^\dagger(t)$ in \eqref{eq:q-dagger}. This parallels our earlier choice of $q^*(t)$: there we minimize an upper bound on the current EOC, whereas here we maximize a lower bound on its exponential decay rate.

It remains to show that $q^\dagger(t) \to q^\dagger$. To establish this convergence, we assume the minimizer is unique.

\begin{assumption}[Unique threshold]\label{ass:unique-threshold}
For every $\boldsymbol x\in\cJ^I$ with $x_{[k]}>x_{[k+1]}$, the function $q\mapsto\sum_{i=1}^I1/d_i(x_i,q)$ has a unique minimizer on $(x_{[k+1]},x_{[k]})$.
\end{assumption}

\begin{remark}\label{rem:unique-threshold-families}
Each likelihood family in \cref{rem:examples} satisfies \cref{ass:unique-threshold}.
\end{remark}

\begin{lemma}[Convergence of $q^\dagger(t)$]\label{lem:rate-threshold-convergence}
Run \cref{alg:sequential} with threshold $q^\dagger(t)$. Then $q^\dagger(t)\to q^\dagger$.
\end{lemma}
By \cref{lem:rate-threshold-convergence}, \cref{thm:threshold-allocation} applies with $q_0=q^\dagger$. We further show that $\FitParens{\sum_{i=1}^I1/d_i(\theta_i^0,q^\dagger)}^{-1} \ge\Gamma^*(\boldsymbol\theta^0)/2$, which yields our main theoretical result.

\begin{theorem}[Exponential rate of TEAk]\label{thm:half-optimal-eoc-rate}
Run \cref{alg:sequential} with threshold $q^\dagger(t)$. Then
  \begin{equation*}
      \lim_{t\to\infty}
      -\frac{\log\mathrm{EOC}(\cH_t)}{\sum_iN_i(t)}
      \ge\frac12\Gamma^*(\boldsymbol\theta^0).
  \end{equation*}
\end{theorem}
Thus, the posterior convergence rate of TEAk is at least half the optimal posterior convergence rate. The conclusions in \cref{prop:optimal-eoc-rate} and in \cref{thm:threshold-allocation,thm:half-optimal-eoc-rate} also hold when $\mathrm{EOC}(\cH_t)$ is replaced by the posterior error probability $\Pp_t(S_k(\boldsymbol\theta)\ne S_k(\boldsymbol\theta^0))$.

\subsection{Numerical illustration}\label{subsec:rate-illustration}
We illustrate \cref{thm:threshold-allocation,thm:half-optimal-eoc-rate} on a Gaussian instance with $I=20$, $k=5$, likelihood variance of $1$, and true means $\theta_i^0=-(i\!-\!1)/2$ so that the gap between consecutive means is $0.5$. On this instance $\Gamma^*(\boldsymbol\theta^0)=0.0243$ and $q^\dagger=-2.25$, the midpoint between $\theta_k^0$ and $\theta_{k+1}^0$. The rate of the threshold allocation is $\Gamma(\boldsymbol w(q^\dagger);\boldsymbol\theta^0)=0.57\,\Gamma^*(\boldsymbol\theta^0)$ and the lower bound on it in \cref{thm:threshold-allocation}(iii) is $0.53\,\Gamma^*(\boldsymbol\theta^0)$.

We run TEAk with $q^\dagger(t)$ on this instance with $B_t=1$, a sampling budget of \num{20000}, and a $N(0,10^2)$ prior over $50$ replications. We also run two fixed allocations: the oracle $\boldsymbol w^*$ that maximizes \eqref{eq:optimal-posterior-rate} using the true means, and equal allocation. \cref{fig:rate} shows the results.

\begin{figure}[t]
\centering
\includegraphics[width=0.95\textwidth]{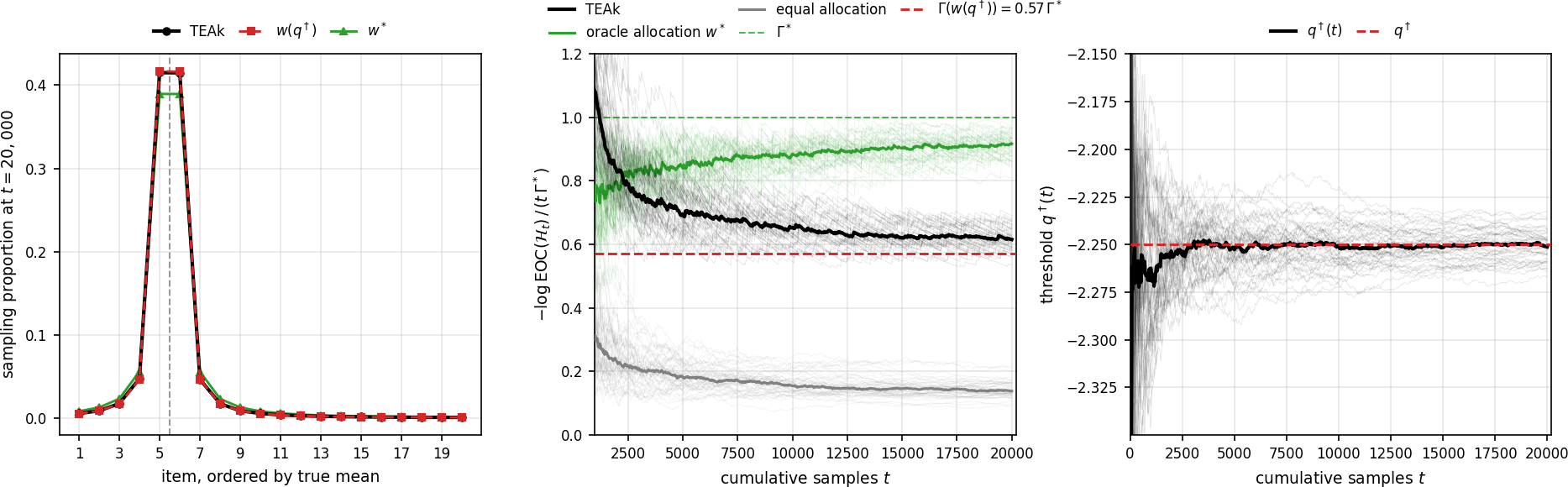}
\caption{Left: TEAk's sampling proportions at $t=\num{20000}$, averaged over replications, as well as $\boldsymbol w(q^\dagger)$ and $\boldsymbol w^*$. Middle: $-\log\mathrm{EOC}(\cH_t)/(t\,\Gamma^*(\boldsymbol\theta^0))$ for TEAk, the oracle allocation, and equal allocation on each of the $50$ replications (median in bold). Right: the threshold $q^\dagger(t)$ on each of the $50$ replications (median in bold).}
\label{fig:rate}
\vspace{-2pt}
\end{figure}

The left panel shows the average sampling proportions at the end of the run. TEAk's proportions coincide with $\boldsymbol w(q^\dagger)$, as guaranteed by \cref{thm:threshold-allocation}(i), and are a bit off from $\boldsymbol w^*$: TEAk puts $83\%$ of the budget on the $k$th and $(k\!+\!1)$th items, whereas $\boldsymbol w^*$ puts $78\%$. The middle panel shows an approximation\footnote{The posterior $\mathrm{EOC}(\cH_t)$ has no closed form, so we plot the upper bound $U_t:=\sum_{i\in S_k(\boldsymbol\mu(t))}\sum_{j\notin S_k(\boldsymbol\mu(t))}\E_t[(\theta_j\!-\!\theta_i)^+]$ in its place, and each term has an explicit expression under Gaussian posteriors. Since $U_t/(k(I\!-\!k))\le\mathrm{EOC}(\cH_t)\le U_t$, the plotted value differs from $-\log\mathrm{EOC}(\cH_t)/t$ by at most $\log(k(I\!-\!k))/t$.} of $-\log\mathrm{EOC}(\cH_t)/t$ along each replication as a fraction of $\Gamma^*(\boldsymbol\theta^0)$. For TEAk, it decreases towards $\Gamma(\boldsymbol w(q^\dagger);\boldsymbol\theta^0)$ in line with \cref{thm:threshold-allocation}(ii), and for the oracle allocation it approaches $\Gamma^*(\boldsymbol\theta^0)$. The right panel shows the convergence $q^\dagger(t)\to q^\dagger$ on each replication.

\section{Numerical experiments}\label{sec:experiments}
In this section, we evaluate TEAk's performance numerically by comparing it to the most competitive algorithms in the literature. For this comparison, we use a batch size of $1$, as most benchmarks are designed for the fully sequential problem. \cref{subsec:batching} then turns to the batched setting and compares the performance of TEAk across batch sizes. Throughout, TEAk and TEAk-$q^\dagger(t)$ denote \cref{alg:sequential} with the threshold $q^*(t)$ and $q^\dagger(t)$ respectively. Across all settings, the two are nearly identical in both performance and allocation.

\subsection{Setup}
We consider $I\in\{20,50,100\}$, with a sampling budget of up to $\num{1000}\!\cdot\!I$ for each setting. For $I=20$ we consider $k\in\{1,3,5,10\}$, and for $I\in\{50,100\}$ we consider $k\in\{1,5,10,20\}$. On each ($I,k$) pair, we consider four likelihood families: Gaussian, exponential, Bernoulli, and lognormal. For each of the four families, we consider three generative distributions for their true means $\boldsymbol{\theta}$, given in \cref{tab:instances}. These ranges, although somewhat arbitrary, are targeted so that the difficulty of the identification problem is approximately the same across the likelihood families. The results are generally the same across these three generative distributions of the true means, so we report the results in the middle column in this section, giving the rest in \cref{app:pcs}. We run \num{10000} replications of each simulation, drawing the true means afresh in each.
\begin{table}[!htb]
\centering
\small
\caption{The four likelihood families and the three distributions used to generate the true means.}
\label{tab:instances}
\begin{tabular}{l lll}
\toprule
Likelihood $Y_{i,n}\mid\theta_i$        & \multicolumn{3}{c}{Distribution of true means $\theta_i$}\\
\midrule
$N(\theta_i,1)$                             & $\mathrm{Unif}(-2,2)$       & $\mathrm{Unif}(-1,1)$    & $\mathrm{Unif}(-0.5,0.5)$\\
$\mathrm{Exp}(1/\theta_i)$                  & $\mathrm{Unif}(1,9)$        & $\mathrm{Unif}(1,5)$     & $\mathrm{Unif}(1,3)$\\
$\mathrm{Bern}(\theta_i)$                   & $\mathrm{Unif}(0.05,0.95)$  & $\mathrm{Unif}(0.2,0.8)$ & $\mathrm{Unif}(0.35,0.65)$\\
$\mathrm{Lognormal}(\log\theta_i\!-\!1/2,\,1)$  & $\mathrm{Unif}(1,9)$        & $\mathrm{Unif}(1,5)$     & $\mathrm{Unif}(1,3)$\\
\bottomrule
\end{tabular}
\end{table}
\vspace{-10pt}
\paragraph{Fair comparison.}
We use three measures to ensure each method is on equal footing. First, within a replication, all methods have the same true means and the outcomes come from the same stream, so the $n$th observation of an item $i$ is identical across methods. Second, to ensure fairness across Bayesian and frequentist methods, the Bayesian methods all use the same uninformative (Jeffreys) prior for each distribution family: the flat prior for the Gaussian mean, $\propto 1/\theta$ for the exponential mean, $\mathrm{Beta}(1/2,1/2)$ for Bernoulli, and the flat prior on $\log\theta$ for the lognormal mean \citep{Jeffreys1946}. Third, since some benchmark methods require a warm-up, each algorithm receives a warm-up of $n_0=3$ samples per item.

\paragraph{Benchmarks.}
We compare against the following benchmarks, all of which aim to maximize PCS unless otherwise specified.
\begin{itemize}[leftmargin=*]
    \item[--] \emph{Gaussian-Specific from the Ranking and Selection Literature:} For $k=1$ AOAP  from \citet{peng2018} and  for $k>1$ AOAm from \citet{ZhangPengZhangZhou2023} and EMIp from \citet{Liu2024}. For the top-$k$ EOC setting, EOC-m from \citet{GaoChen2015}.

    \item[--] \emph{General $k = 1$:}  Top two Thompson sampling (TTTS) of \citet{russo2020} with $\beta=1/2$ and its variant that does not rely on a tuning parameter, TTTS-IDS from \citet{QinYou2026}. BOLD \citep{chenRyzhov2023}.

    \item[--] \emph{General $k > 1$ from the Best Arm(s) Identification Literature:} The top-$k$ Thompson sampling approaches TS-PPS-IDS and TS-KKT(.5)-IDS from \citet{you2023,QinYou2026}, KL-LUCB from \citet{kaufmann2013}, and gap-e from \citet{bubeck2013}.

  \item[--] \emph{Uniform:} Equal allocation across all items.
\end{itemize}

\paragraph{Metrics.}
Since most benchmarks aim to maximize PCS, we report both EOC and PCS, estimated from the \num{10000} replications. Consistent with previous literature and the discussion in the introduction, the two are closely aligned. Since none of the methods other than gap-e depend on the total budget, we have their EOC and PCS for every fixed budget up to $\num{1000}\!\cdot\!I$.

\subsection{Results}
To keep the paper focused, we report PCS at the largest budget ($\num{1000}\!\cdot\!I$), and for EOC, we report performance curves across budgets for the Gaussian setting. The PCS and EOC performance curves for all settings are in \cref{app:pcs}.

\begin{figure}[t]
\centering
\includegraphics[width=0.95\textwidth]{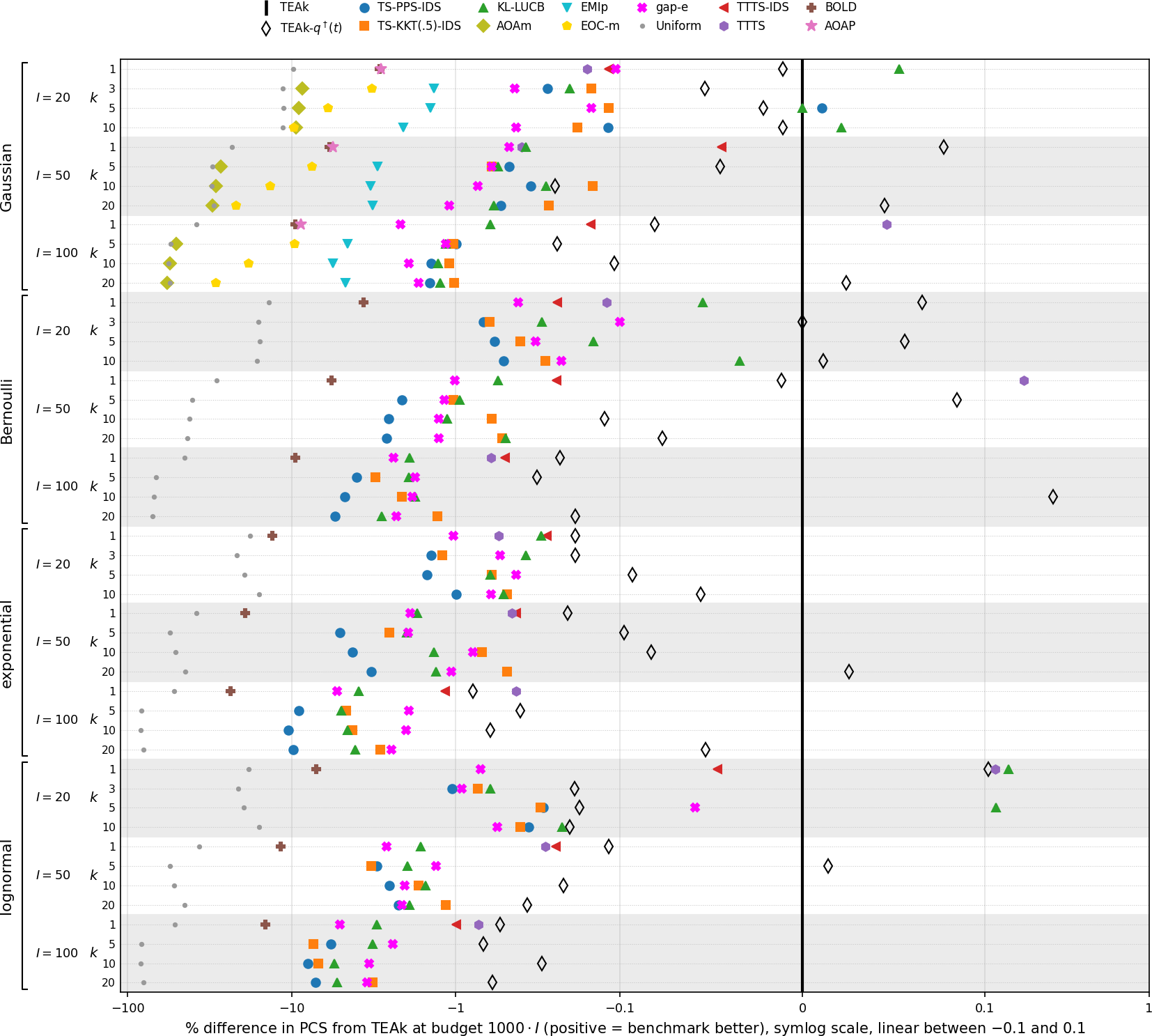}
\caption{\% difference in PCS between benchmarks and TEAk at budget $\num{1000}\!\cdot\!I$ (positive means benchmark is better).}
\label{fig:final}
\end{figure}

\begin{figure}[t]
\centering
\includegraphics[width=0.95\textwidth]{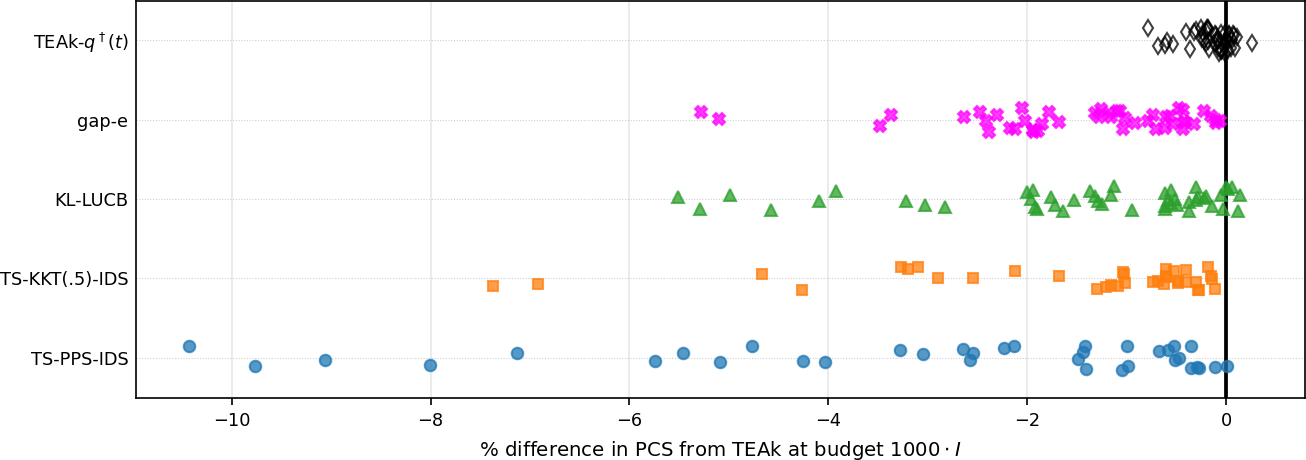}
\caption{The distribution across all settings in \cref{fig:final} for the four strongest benchmarks and TEAk-$q^\dagger(t)$.}
\label{fig:strip}
\end{figure}

\paragraph{PCS.}
\cref{fig:final} shows the PCS of each benchmark relative to TEAk for all configurations. TEAk generally has the highest PCS, and when it does not, it is within a few percent of the best. TEAk's biggest advantage over the benchmarks is on larger problems ($I = 100$) and larger $k$. To show more clearly how TEAk improves over the best benchmarks and how it compares with TEAk-$q^\dagger(t)$, \cref{fig:strip} takes the four best-performing benchmarks, together with TEAk-$q^\dagger(t)$, and plots the distribution over all configurations of their PCS relative to TEAk. Each benchmark falls short of TEAk by $5\%$ or more in some configurations, while none beats it by more than $0.2\%$ in any configuration. In addition, TEAk outperforms TEAk-$q^\dagger(t)$ in $36$ of the $48$ configurations, and in every configuration the PCS of TEAk-$q^\dagger(t)$ is between $0.8\%$ below and $0.3\%$ above that of TEAk.

\paragraph{EOC over the whole budget.}
\cref{fig:grid-gauss} plots EOC on a log scale against cumulative samples spent for the Gaussian family and each $(I,k)$ pair. At $k=1$, the most competitive benchmarks are the TTTS variants, and for general $k$ the most competitive are gap-e, KL-LUCB, and the top-$k$ Thompson sampling methods. Which benchmark is best (or near-best) varies across configurations and varies with the budget within a configuration. TEAk is the exception: in every configuration, its EOC is the lowest, or close to it, across all budgets. The plots for the other three sampling families show a similar picture.

\begin{figure}[!htb]
\centering
\includegraphics[width=0.95\textwidth]{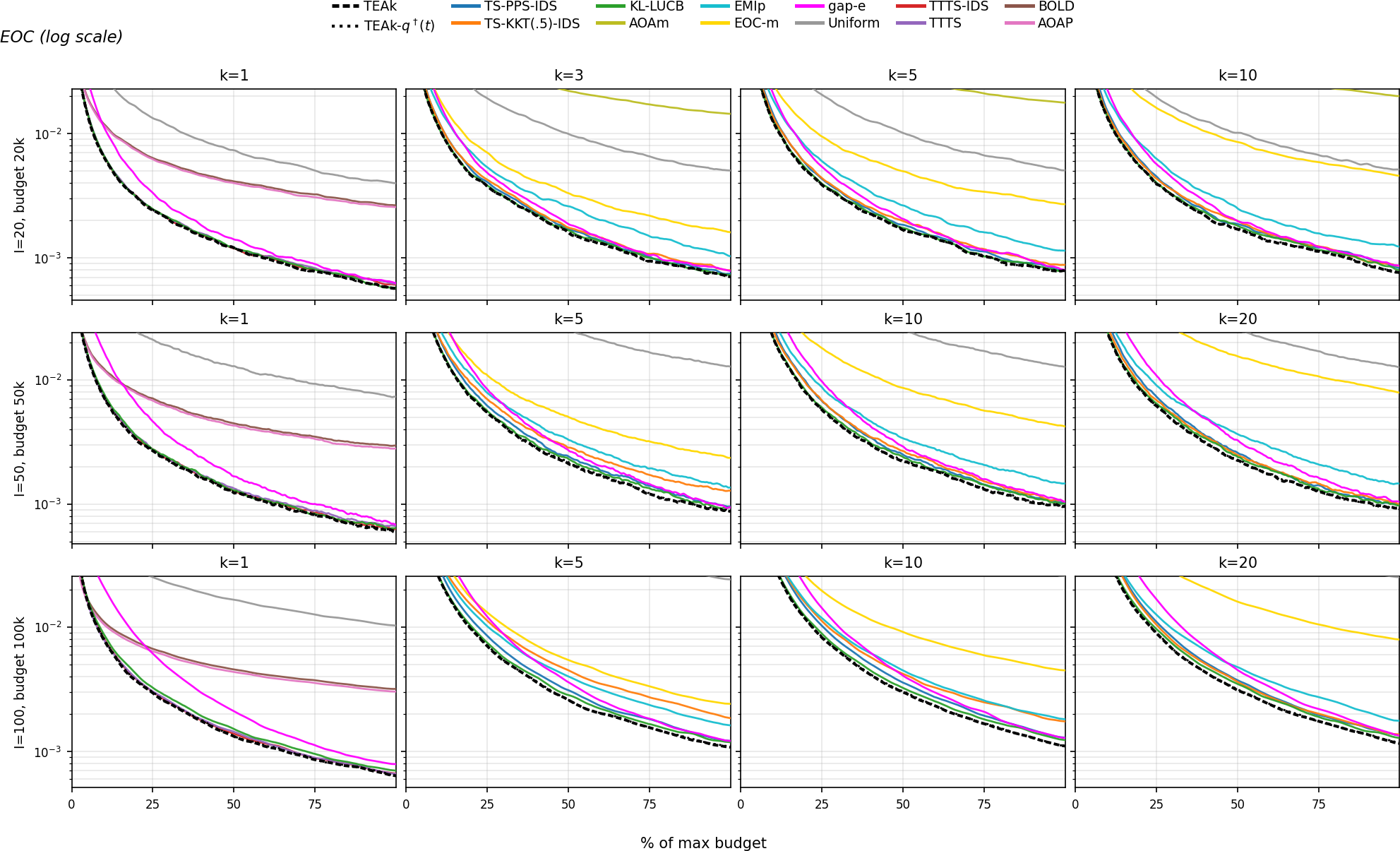}
\caption{Gaussian: EOC (log scale) against budget (max budget $ = \num{1000}\!\cdot\!I$). Rows vary the problem
size $I$ and columns vary $k$.}
\label{fig:grid-gauss}
\vspace{-20pt}
\end{figure}

\subsection{Batching}\label{subsec:batching}
We now investigate TEAk's performance in the batched setting, one of its main motivations.  We compare against BatchSAR \citep{jun2016}, a version of the successive accepts and rejects algorithm of \citet{bubeck2013} designed specifically for fixed-budget top-$k$ identification with batched sampling. We consider the Gaussian and Bernoulli families with $I\in\{20,50,100\}$, $k=5$, and the middle column of \cref{tab:instances}. Since BatchSAR's schedule depends on the total budget, we cannot run it with a large budget and read off its performance at every point along the way, so we fix a budget of $50\!\cdot\!I$ and consider batch sizes of $1\%$, $2\%$, $5\%$, and $10\%$ of the budget. \cref{fig:batchsar} plots PCS, the objective BatchSAR is designed to maximize. Not only does TEAk have a higher PCS at the smallest batch size, but it also decreases less as the batch size grows: from batches of $1\%$ to $10\%$ of the budget, its PCS falls by up to $12\%$, and BatchSAR's by up to $29\%$.

\begin{figure}[!htb]
\centering
\includegraphics[width=0.95\textwidth]{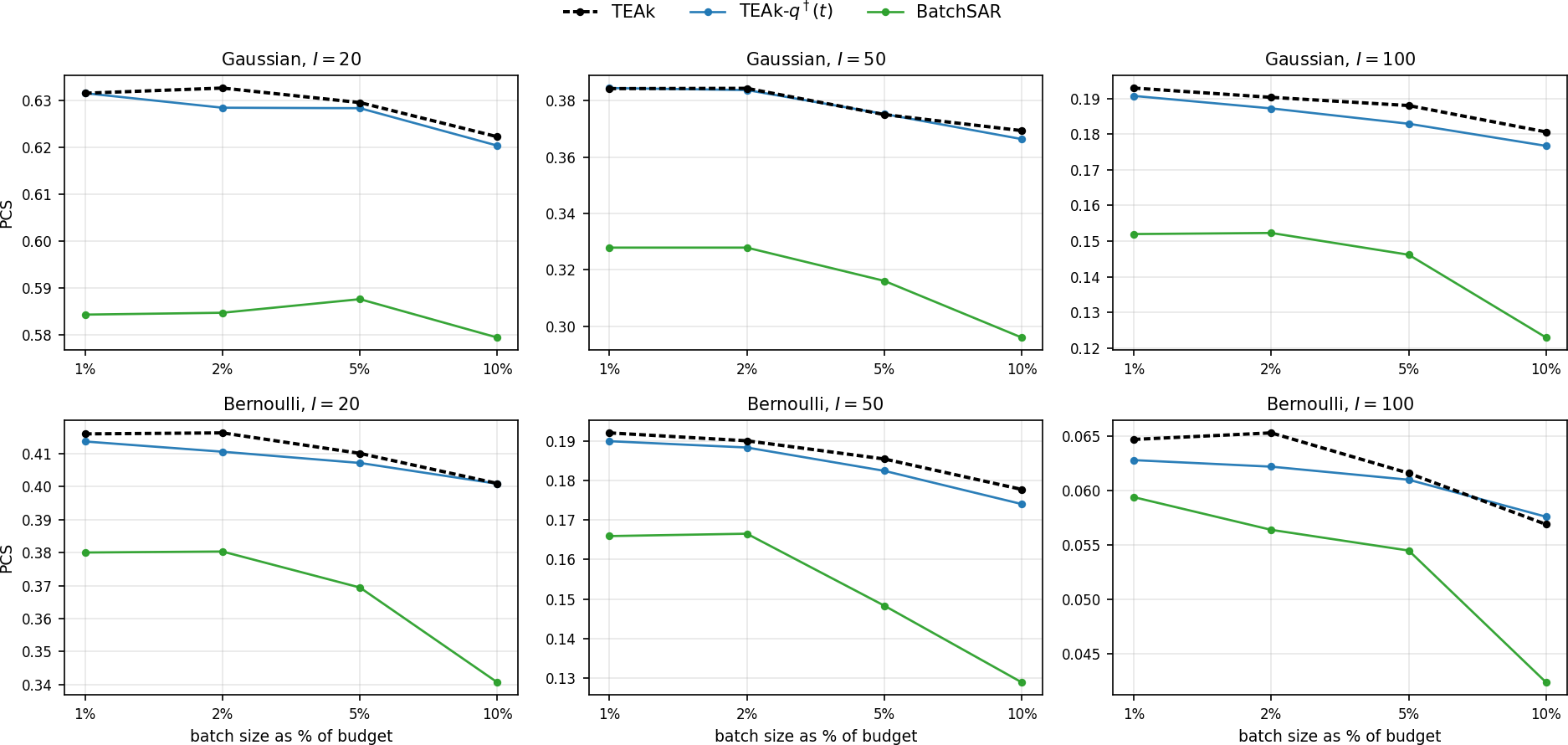}
\caption{PCS at the final budget of $50\!\cdot\!I$ against the batch size for the Gaussian (top) and Bernoulli (bottom) families with $k=5$.}
\label{fig:batchsar}
\end{figure}

\section{Discussion}\label{sec:discussion}
We give a simple and interpretable algorithm for top-$k$ selection that accommodates batching in a principled way and can be applied to a wide range of sampling distribution families. We prove that it is consistent and that with the threshold $q^\dagger(t)$, its posterior convergence rate is at least half of the best rate achievable by any sampling rule. In simulation, it is competitive with the state-of-the-art methods across likelihood families, problem sizes, values of $k$, and batch sizes.

\bibliography{refs}
\clearpage

\begin{APPENDICES}
\renewcommand{\theHsection}{appendix.\Alph{section}}

\crefalias{section}{appendix}

\section{Proofs for \cref{sec:bound}}\label{app:proofs}

\begin{proof}{Proof of \cref{prop:eoc-bound}(i)}
Fix any realized $\boldsymbol\theta\in\R^I$ and any $q\in\R$. Then
\begin{align*}
\mathrm{OC}(S_k(\boldsymbol\mu(t)),\boldsymbol\theta)
&= \sum_{i\in S_k(\boldsymbol\theta)}\theta_i
    - \sum_{i\in S_k(\boldsymbol\mu(t))}\theta_i \\
&\overset{(a)}{=} \sum_{i\in S_k(\boldsymbol\theta)\setminus S_k(\boldsymbol\mu(t))}\theta_i
    - \sum_{i\in S_k(\boldsymbol\mu(t))\setminus S_k(\boldsymbol\theta)}\theta_i \\
&\overset{(b)}{=} \sum_{i\in S_k(\boldsymbol\theta)\setminus S_k(\boldsymbol\mu(t))}(\theta_i-q)
    + \sum_{i\in S_k(\boldsymbol\mu(t))\setminus S_k(\boldsymbol\theta)}(q-\theta_i) \\
&\le \sum_{i\notin S_k(\boldsymbol\mu(t))}(\theta_i-q)^+
    + \sum_{i\in S_k(\boldsymbol\mu(t))}(q-\theta_i)^+,
\end{align*}
where (a) holds because indices in $S_k(\boldsymbol\mu(t)) \cap S_k(\boldsymbol\theta)$ cancel, and (b) holds because $|S_k(\boldsymbol\mu(t))\setminus S_k(\boldsymbol\theta)| = |S_k(\boldsymbol\theta)\setminus S_k(\boldsymbol\mu(t))|$. Taking posterior expectation proves the claim.
\end{proof}

\begin{proof}{Proof of \cref{prop:eoc-bound}(ii)}
Define
\begin{equation*}
G_t(q) := \E_t\FitDelims{[}{]}{
        \sum_{i\in S_k(\boldsymbol\mu(t))}(q-\theta_i)^+
        +
        \sum_{i\notin S_k(\boldsymbol\mu(t))}(\theta_i-q)^+
    }.
\end{equation*}
By \cref{ass:model} each posterior CDF $F_{i,t}$ is continuous on $\cJ$, so differentiating under the expectation gives
\begin{equation*}
\frac{d}{dq}\E_t[(q-\theta_i)^+] = F_{i,t}(q), \qquad \frac{d}{dq}\E_t[(\theta_i-q)^+] = F_{i,t}(q) - 1.
\end{equation*}
Hence
\begin{equation*}
G_t'(q) = \sum_{i\in S_k(\boldsymbol\mu(t))} F_{i,t}(q) + \sum_{i\notin S_k(\boldsymbol\mu(t))}(F_{i,t}(q) - 1) = \sum_{i=1}^I F_{i,t}(q) - (I-k).
\end{equation*}
Additionally, by \cref{ass:model} each posterior density is positive on $\cJ$, so each $F_{i,t}$ is strictly increasing on $\cJ$. Thus $G_t'$ is continuous and strictly increasing on $\cJ$, which gives convexity of $G_t$. Taking limits, $F_{i,t}(q)\to 0$ as $q\to-\infty$ and $F_{i,t}(q)\to 1$ as $q\to+\infty$, so $G_t'(q)\to -(I\!-\!k) < 0$ as $q\to-\infty$ and $G_t'(q)\to k > 0$ as $q\to+\infty$. By the intermediate value theorem, $G_t'$ has a unique zero $q^*(t)\in\cJ$, characterized by $\sum_i F_{i,t}(q^*(t)) = I\!-\!k$, equivalently \eqref{eq:q-star}, and $q^*(t)$ is the unique minimizer of $G_t$.
\end{proof}

\begin{proof}{Proof of \cref{prop:dp-exact}}
Starting with $r=0$, set $V_0(0)=0$ and $V_0(b)=+\infty$ for $b=1,\ldots,B_t$. For the inductive step, suppose $V_{r-1}(\cdot)$ is already known and consider the $(r,b)$ subproblem. Because the objective in \eqref{eq:Pt} is separable across items, fixing item $r$'s allocation at some $\delta\in\{0,\ldots,b\}$ splits the subproblem into two pieces: the cost $\ell_{r,t}(\delta)$ incurred by item $r$, plus the cost of optimally allocating exactly the residual $b\!-\!\delta$ samples among items $1,\ldots,r\!-\!1$, which is $V_{r-1}(b\!-\!\delta)$. Therefore, for each $(r,b) \in \{(x,y) : 1 \leq x \leq I, 0 \leq y \leq B_t\}$ the optimal value function is given by
\begin{equation*}
    V_r(b) = \min_{\delta\in\{0,\ldots,b\}}\{V_{r-1}(b-\delta)+\ell_{r,t}(\delta)\}, \qquad \Delta_r^*(b) \in \argmin_{\delta\in\{0,\ldots,b\}}\{V_{r-1}(b-\delta)+\ell_{r,t}(\delta)\}.
\end{equation*}
Once we solve this for each $(r,b)$, an optimal allocation vector $\boldsymbol\Delta^{t}$ is recovered by backtracking. $\Delta_I^{t}=\Delta_I^*(B_t)$, leaving a residual budget $B_t\!-\!\Delta_I^{t}$ for items $1,\ldots,I\!-\!1$. Then $\Delta_{I-1}^{t} = \Delta_{I-1}^*(B_t\!-\!\Delta_I^{t}) $ and so on down to item $1$:
\begin{equation*}
    \Delta_r^{t} = \Delta_r^*\!\Bigl(B_t-\sum_{j=r+1}^I \Delta_j^{t}\Bigr), \qquad r=I,I-1,\ldots,1.
\end{equation*}
By construction, backtracking recovers an allocation whose components sum to $B_t$ and whose objective value is $V_I(B_t)$. It is therefore a global optimizer of \eqref{eq:Pt}. The total runtime is the cost of precomputing the $I(B_t\!+\!1)$ values $\ell_{i,t}(d)$ which is $O(IB_t C_\ell)$ and the cost of filling in the $O(IB_t)$ DP states (each requiring a minimization over at most $B_t\!+\!1$ values of $\delta$) which is $O(IB_t^2)$.
\end{proof}

\begin{proof}{Proof of \cref{prop:greedy-conditions}(i)}
For any allocation $\boldsymbol{\Delta}$, the improvement relative to assigning no samples in the current round is
\begin{equation*}
    \sum_{j=1}^I \ell_{j,t}(0) - \sum_{j=1}^I \ell_{j,t}(\Delta_j) = \sum_{j=1}^I \sum_{h=0}^{\Delta_j - 1}\bigl[\ell_{j,t}(h) - \ell_{j,t}(h+1)\bigr] = \sum_{j=1}^I \sum_{h=0}^{\Delta_j - 1} D_{j,t}(h),
\end{equation*}
where the inner sum is zero when $\Delta_j = 0$. Minimizing $\sum_j \ell_{j,t}(\Delta_j)$ is therefore equivalent to maximizing
\begin{equation*}
    G(\boldsymbol{\Delta}) := \sum_{j=1}^I \sum_{h=0}^{\Delta_j - 1} D_{j,t}(h)
\end{equation*}
subject to $\sum_j \Delta_j = B_t$. Let $\boldsymbol{\Delta}^G$ denote the greedy allocation and $\boldsymbol{\Delta}$ any other feasible allocation. We will show $G(\boldsymbol{\Delta}^G) \ge G(\boldsymbol{\Delta})$. If $\boldsymbol{\Delta} = \boldsymbol{\Delta}^G$, there is nothing to prove.  Otherwise, since both allocations sum to $B_t$, there exist indices $i, j$ with $\Delta_i > \Delta_i^G$ and $\Delta_j < \Delta_j^G$.  We show that one can transfer a sample from $i$ to $j$ without decreasing $G$.

Because $\Delta_i > \Delta_i^G$, the term $D_{i,t}(\Delta_i\!-\!1)$ appears in $G(\boldsymbol{\Delta})$.  Since $\Delta_i\!-\!1 \ge \Delta_i^G$, \eqref{eq:dim-gains} gives
\begin{equation*}
    D_{i,t}(\Delta_i - 1) \le D_{i,t}(\Delta_i^G).
\end{equation*}
Because $\Delta_j < \Delta_j^G$, the term $D_{j,t}(\Delta_j)$ does not appear in $G(\boldsymbol{\Delta})$. Since $\Delta_j \le \Delta_j^G\!-\!1$, \eqref{eq:dim-gains} gives
\begin{equation*}
    D_{j,t}(\Delta_j) \ge D_{j,t}(\Delta_j^G - 1).
\end{equation*}
Consider the greedy step at which the $(\Delta_j^G)$th sample was assigned to item $j$.  Before that step, $j$ had received $\Delta_j^G\!-\!1$ samples and $i$ had received at most $\Delta_i^G$ samples.  The marginal gain of allocating to $j$ at that step was $D_{j,t}(\Delta_j^G\!-\!1)$ and the marginal gain of $i$ was at least $D_{i,t}(\Delta_i^G)$ by \eqref{eq:dim-gains}.  Since greedy chose $j$,
\begin{equation*}
    D_{j,t}(\Delta_j^G - 1) \ge D_{i,t}(\Delta_i^G).
\end{equation*}
Combining the three inequalities,
\begin{equation*}
    D_{j,t}(\Delta_j) \;\ge\; D_{j,t}(\Delta_j^G - 1) \;\ge\; D_{i,t}(\Delta_i^G) \;\ge\; D_{i,t}(\Delta_i - 1).
\end{equation*}
Define $\boldsymbol{\Delta}'$ by transferring one sample from $i$ to $j$: $\Delta'_i = \Delta_i\!-\!1$, $\Delta'_j = \Delta_j\!+\!1$, and $\Delta'_\ell = \Delta_\ell$ for $\ell\notin\{i,j\}$.  Then $\sum_\ell \Delta'_\ell = B_t$, and
\begin{equation*}
    G(\boldsymbol{\Delta}') - G(\boldsymbol{\Delta}) = D_{j,t}(\Delta_j) - D_{i,t}(\Delta_i - 1) \ge 0.
\end{equation*}
Repeating these exchanges transforms $\boldsymbol{\Delta}$ into $\boldsymbol{\Delta}^G$ without ever decreasing the gain.  Therefore $G(\boldsymbol{\Delta}^G) \ge G(\boldsymbol{\Delta})$, and $\boldsymbol{\Delta}^G$ minimizes $\sum_j \ell_{j,t}(\Delta_j)$ over the feasible set.
\end{proof}

\begin{proof}{Proof of \cref{prop:greedy-conditions}(ii)}
Fix item $i$ and round $t$, and write $c(d) := c_i(t, d)$, $\lambda := \lambda_i(t)$, $\ell(d) := \ell_{i,t}(d)$, $\theta := \theta_i$, $\mu := \mu_i(t)$, $q := q(t)$, and $\pi(\cdot)$ for the density of the posterior $\Pi_{i,t}$. By the definitions of $\widetilde\theta_i(d)$ and $\ell_{i,t}(d)$,
\begin{equation*}
    \ell(d) = \E\bigl[(q - \mu - c(d)(\theta - \mu))^+\bigr].
\end{equation*}
We will relax $\ell(\cdot)$ to the reals $d \ge 0$ and show $\ell$ is convex in $d$, thus $D_{i,t}(d) = \ell(d)\!-\!\ell(d\!+\!1)$ is nonincreasing in $d \in \mathbb{Z}_{\ge 0}$, which is \eqref{eq:dim-gains}.

We first compute the derivatives of $c$. By assumption, $c(d) = (1\!+\!\lambda d)^{-1/2}$, so
\begin{equation*}
    c'(d) = -\frac{\lambda}{2}(1 + \lambda d)^{-3/2} = -\frac{\lambda}{2}\,c(d)^3.
\end{equation*}
Differentiating again,
\begin{equation*}
    c''(d) = -\frac{\lambda}{2} \cdot 3 c(d)^2 \cdot c'(d) = -\frac{\lambda}{2} \cdot 3 c(d)^2 \cdot \Bigl(-\frac{\lambda}{2}\,c(d)^3\Bigr) = \frac{3\lambda^2}{4}\,c(d)^5 \ge 0.
\end{equation*}
Next we differentiate $\ell$. Let $U(d) := \mu\!+\!(q\!-\!\mu)/c(d)$ so
\begin{equation*}
    \ell(d) = \int_\R \bigl(q - \mu - c(d)(\theta - \mu)\bigr)^+ \pi(\theta)\,d\theta = \int_{-\infty}^{U(d)} \bigl(q - \mu - c(d)(\theta - \mu)\bigr) \pi(\theta)\,d\theta.
\end{equation*}

Define $I(\tau) := \int_{-\infty}^{\tau} (\theta\!-\!\mu)\pi(\theta)\,d\theta$, taking $\pi=0$ outside $\cJ$. Since $\E[|\theta\!-\!\mu|]<\infty$, differentiating under the expectation gives
\begin{equation*}
    \ell'(d)
    =-c'(d)\int_{-\infty}^{U(d)}
      (\theta-\mu)\pi(\theta)\,d\theta
    =-c'(d)I(U(d)).
\end{equation*}
For $U(d)\in\cJ$, differentiating again gives
\begin{align*}
  \ell''(d)
  &= -c''(d)\,I(U(d)) - c'(d)\,I'(U(d))\,U'(d) \\
  &= -c''(d)\,I(U(d)) - c'(d) \cdot \tfrac{q-\mu}{c(d)}\,\pi(U(d)) \cdot
  \Bigl(-\tfrac{q-\mu}{c(d)^2}\,c'(d)\Bigr) \\
  &= -c''(d)\,I(U(d)) + c'(d)^2\,\frac{(q-\mu)^2}{c(d)^3}\,\pi(U(d)).
\end{align*}
Both terms are nonnegative. To see this, $c''(d)\ge0$, and $I(\tau)\le0$ for $\tau\le\mu$. For $\tau\ge\mu$, using $\E[\theta\!-\!\mu]=0$ gives
\begin{equation*}
    I(\tau)
    =-\int_\tau^\infty(\theta-\mu)\pi(\theta)\,d\theta
    \le0.
\end{equation*}
Thus $\ell''(d)\ge0$ whenever $U(d)\in\cJ$. On any interval where $U(d)$ is outside $\cJ$, $\ell$ is constant. Since $\ell'(d)=-c'(d)I(U(d))$ is continuous also when $U(d)$ crosses an endpoint of $\cJ$, $\ell'$ is nondecreasing on $[0,\infty)$. Therefore $\ell$ is convex, so $D_{i,t}(d)$ is nonincreasing in $d$.
\end{proof}

\begin{proof}{Proof of \cref{prop:greedy-conditions}(iii)}Fix an item $i$ and a round $t$. By
conjugacy, the posterior density of $\eta_i$ given $\mathcal{H}_t$ is proportional to $\exp\{a\eta\!-\!bA_i(\eta)\},$ where
\begin{equation*}
a
=
a_i+\sum_{n=1}^{N_i(t)}Y_{i,n},
\qquad
b
=
b_i+N_i(t).
\end{equation*}
The standard moment identity for this natural exponential family conjugate posterior gives $\mu_i(t)=\frac{a}{b}$ \citep{diaconisYlvisaker1979}. Additionally,
\begin{equation*}
0
=
\int_{\mathcal E_i}
\frac{d}{d\eta}
\FitDelims{[}{]}{
\FitParens{A_i'(\eta)-\frac{a}{b}}
\exp\{a\eta-bA_i(\eta)\}
}d\eta,
\end{equation*}
which, after expanding the derivative and normalizing, gives
\begin{equation*}
\mathbb E_t[A_i''(\eta_i)]
=
b\,\mathbb E_t\FitDelims{[}{]}{
\FitParens{A_i'(\eta_i)-\frac{a}{b}}^2
}
=
b\,v_i(t).
\end{equation*}
Fix $d\geq 1$, and let $S_d = \sum_{n=1}^{d}Y_{i,N_i(t)+n}$ denote the sum of the $d$ additional observations. Conditional on $\eta_i$, $\mathbb E_t[S_d\mid\eta_i] = dA_i'(\eta_i) = d\theta_i$ and $\mathbb V_t(S_d\mid\eta_i) = dA_i''(\eta_i).$ Therefore, by the law of total variance,
\begin{equation*}
\mathbb V_t(S_d)
= \mathbb E_t\FitDelims{[}{]}{
\mathbb V_t(S_d\mid\eta_i)
}
+
\mathbb V_t\FitParens{
\mathbb E_t[S_d\mid\eta_i]
}
=db\,v_i(t)+d^2v_i(t)=
d(b+d)v_i(t).
\end{equation*}
After observing these additional samples, the posterior density of $\eta_i$ is proportional to $\exp\FitDelims{\{}{\}}{ (a\!+\!S_d)\eta\!-\!(b\!+\!d)A_i(\eta) }.$ Applying the same conjugate-posterior mean identity gives
\begin{equation*}
\mathbb E\FitDelims{[}{]}{
\theta_i\mid \cH_t,S_d
}
=
\frac{a+S_d}{b+d}.
\end{equation*}
Hence,
\begin{equation*}
\mathbb V_t\FitParens{
\mathbb E\FitDelims{[}{]}{
\theta_i\mid \cH_t,S_d
}
}
=
\frac{\mathbb V_t(S_d)}{(b+d)^2}
=
\frac{d}{b+d}v_i(t).
\end{equation*}
Finally, the law of total variance gives
\begin{equation*}
v_i(t)
=
\bar v_i(t,d)
+
\mathbb V_t\FitParens{
\mathbb E\FitDelims{[}{]}{
\theta_i\mid \cH_t,S_d
}
}.
\end{equation*}
Therefore
\begin{equation*}
c_i(t,d)^2
=
\frac{\bar v_i(t,d)}{v_i(t)}
=
\frac{b}{b+d}
=
\frac{1}{1+\frac{1}{b}d}.
\end{equation*}
Since $b=b_i\!+\!N_i(t)>0$, \eqref{eq:contraction-form} holds with $\lambda_i(t) = \frac{1}{b_i+N_i(t)}.$
\end{proof}

\section{Closed forms for the conjugate families of \cref{prop:families}}\label{app:families}

In each subsection, we first define the round-$t$ posterior parameters and then suppress their argument $t$ in the remaining formulas.

\subsection{Gaussian observations with known variance}\label{subsec:gaussian}

Suppose $Y_{i,n}\mid\theta_i\sim N(\theta_i,\sigma_i^2)$ with $\sigma_i^2$ known, and prior $\theta_i\sim N(m_i(0),s_i^2(0))$. At the start of round $t$,
\begin{equation*}
    \theta_i\mid\cH_t\sim N(m_i(t),\,s_i^2(t)),
    \qquad
    s_i^2(t) = \FitParens{\frac{1}{s_i^2(0)} + \frac{N_i(t)}{\sigma_i^2}}^{-1},
    \quad
    m_i(t) = s_i^2(t)\!\FitParens{\frac{m_i(0)}{s_i^2(0)} + \frac{1}{\sigma_i^2}\sum_{n=1}^{N_i(t)} Y_{i,n}}.
\end{equation*}
Then $\mu_i(t) = m_i$, $v_i(t) = s_i^2$, $F_{i,t}(q) = \Phi\!\FitParens{\frac{q - m_i}{s_i}}$ and $q^*(t)$ solves
\begin{equation*}
    \sum_{i=1}^I \FitDelims{[}{]}{1 - \Phi\!\FitParens{\frac{q - m_i}{s_i}}} = k.
\end{equation*}
We next compute the expected post-round-$t$ posterior variance. After $\Delta_i$ further observations, the posterior variance of $\theta_i$ is deterministic: it depends on $\Delta_i$ but not on the realized values. This deterministic update is a special property of the Gaussian-Gaussian pair that does not hold for the other families. Hence
\begin{equation*}
    \bar v_i(t,\Delta_i) = \FitParens{\frac{1}{s_i^2} + \frac{\Delta_i}{\sigma_i^2}}^{-1},
    \qquad
    c_i(t,\Delta_i)^2 = \frac{1}{1 + s_i^2 \Delta_i/\sigma_i^2}.
\end{equation*}
This has the form \eqref{eq:contraction-form}, so GBB is exact.

Finally, for any $z \in \R$,
\begin{equation*}
    \E_t[(z - \theta_i)^+] = (z - m_i)\,\Phi\!\FitParens{\frac{z - m_i}{s_i}} + s_i\,\phi\!\FitParens{\frac{z - m_i}{s_i}}.
\end{equation*}
Combined with $c_i(t,\Delta_i)$ and \eqref{eq:ell-general}, this yields the closed form for $\ell_{i,t}(\Delta_i)$.

\subsection{Binomial with beta posterior}\label{subsec:beta}

Suppose $Y_{i,n}\mid p_i\sim\mathrm{Bin}(r,p_i)$ with $r\in\mathbb Z_+$ known, mean parameter $\theta_i = rp_i$, and prior $p_i\sim\mathrm{Beta}(a_i(0),b_i(0))$. Let $S_i(t) := \sum_{n=1}^{N_i(t)} Y_{i,n}$. At the start of round $t$,
\begin{equation*}
    p_i\mid\cH_t\sim\mathrm{Beta}(a_i(t),\,b_i(t)),
    \qquad
    a_i(t) := a_i(0) + S_i(t),
    \quad
    b_i(t) := b_i(0) + rN_i(t) - S_i(t).
\end{equation*}
Set $n_i := a_i\!+\!b_i$. Then
\begin{equation*}
    \mu_i(t) = r\,\frac{a_i}{n_i},
    \qquad
    v_i(t) = r^2\,\frac{a_i b_i}{n_i^2(n_i + 1)},
    \qquad
    F_{i,t}(q) = I_{q/r}(a_i, b_i)
    \quad(0 < q < r),
\end{equation*}
and $q^*(t)$ solves
\begin{equation*}
    \sum_{i=1}^I \FitDelims{[}{]}{1 - I_{q/r}(a_i, b_i)} = k.
\end{equation*}
We next compute the expected post-round-$t$ posterior variance. Allocating $\Delta_i$ further binomial observations is equivalent to $r\Delta_i$ further Bernoulli trials. Conditional on the future number of successes $X$,
\begin{equation*}
    p_i\mid\cH_t,X\sim\mathrm{Beta}(a_i + X, b_i + r\Delta_i - X),
\end{equation*}
so taking expectation of $\Var(\theta_i\mid\cH_t, X)$ over $X$ gives
\begin{equation*}
    \bar v_i(t,\Delta_i) = r^2\,\frac{a_i b_i}{n_i(n_i + 1)(n_i + r\Delta_i)} = v_i(t)\,\frac{n_i}{n_i + r\Delta_i},
    \qquad
    c_i(t,\Delta_i)^2 = \frac{n_i}{n_i + r\Delta_i} = \frac{1}{1 + (r/n_i)\Delta_i}.
\end{equation*}
This has the form \eqref{eq:contraction-form}, so GBB is exact.

Finally, for $0 < z < r$,
\begin{equation*}
    \E_t[(z - \theta_i)^+] = z\,I_{z/r}(a_i, b_i) - r\,\frac{a_i}{n_i}\,I_{z/r}(a_i + 1, b_i),
\end{equation*}
extending to $z \le 0$ and $z \ge r$ with values $0$ and $z\!-\!\mu_i(t)$ respectively. Combined with $c_i(t,\Delta_i)$ and \eqref{eq:ell-general}, this yields the closed form for $\ell_{i,t}(\Delta_i)$.

\subsection{Poisson with gamma posterior}\label{subsec:poisson}

Suppose $Y_{i,n}\mid\theta_i\sim\mathrm{Pois}(\theta_i)$ with prior $\theta_i\sim\mathrm{Gamma}(a_i(0),b_i(0))$, where $b_i(0)$ is the rate. After $N_i(t)$ observations,
\begin{equation*}
    \theta_i\mid\cH_t\sim\mathrm{Gamma}(a_i(t),\,b_i(t)),
    \qquad
    a_i(t) := a_i(0) + \sum_{n=1}^{N_i(t)} Y_{i,n},
    \quad
    b_i(t) := b_i(0) + N_i(t).
\end{equation*}
The posterior density is $(b_i^{a_i}/\Gamma(a_i))\,\theta^{a_i-1} e^{-b_i\theta}$ on $\theta > 0$. Then
\begin{equation*}
    \mu_i(t) = \frac{a_i}{b_i},
    \qquad
    v_i(t) = \frac{a_i}{b_i^2},
    \qquad
    F_{i,t}(q) = \gamma(a_i, b_i q)
    \quad(q > 0),
\end{equation*}
and $q^*(t)$ solves
\begin{equation*}
    \sum_{i=1}^I \Gamma(a_i, b_i q) = k.
\end{equation*}
We next compute the expected post-round-$t$ posterior variance. The future total $X := \sum_{\ell=1}^{\Delta_i} Y_{i,\ell}^{\mathrm{new}}$ satisfies $X\mid\theta_i\sim\mathrm{Pois}(\Delta_i\,\theta_i)$, so
\begin{equation*}
    \theta_i\mid\cH_t,\,X\sim\mathrm{Gamma}(a_i + X, b_i + \Delta_i),
    \qquad
    \Var(\theta_i\mid\cH_t,\,X) = \frac{a_i + X}{(b_i + \Delta_i)^2}.
\end{equation*}
Using $\E_t[X] = \Delta_i\,a_i/b_i$,
\begin{equation*}
    \bar v_i(t,\Delta_i) = \frac{a_i}{b_i(b_i + \Delta_i)} = v_i(t)\,\frac{b_i}{b_i + \Delta_i},
    \qquad
    c_i(t,\Delta_i)^2 = \frac{b_i}{b_i + \Delta_i} = \frac{1}{1 + \Delta_i/b_i}.
\end{equation*}
This has the form \eqref{eq:contraction-form}, so GBB is exact.

Finally, for $z > 0$,
\begin{equation*}
    \E_t[(z - \theta_i)^+] = z\,\gamma(a_i, b_i z) - \frac{a_i}{b_i}\,\gamma(a_i + 1, b_i z).
\end{equation*}
Combined with $c_i(t,\Delta_i)$ and \eqref{eq:ell-general}, this gives the closed form for $\ell_{i,t}(\Delta_i)$.

\subsection{Gamma observations with known shape and inverse-gamma posterior}\label{subsec:gamma-known-shape}

Suppose $Y_{i,n}\mid\theta_i\sim\mathrm{Gamma}(\kappa_i,\,\theta_i/\kappa_i)$ (shape $\kappa_i$, scale $\theta_i/\kappa_i$) with $\kappa_i > 0$ known (the case $\kappa_i = 1$ is the exponential model $Y_{i,n}\mid\theta_i\sim\mathrm{Exp}(1/\theta_i)$). The conjugate prior is $\theta_i\sim\mathrm{InvGamma}(a_i(0),b_i(0))$ with $a_i(0) > 2$ and $b_i(0)>0$, ensuring $v_i(t)$ is finite at every round. After $N_i(t)$ observations,
\begin{equation*}
    \theta_i\mid\cH_t\sim\mathrm{InvGamma}(a_i(t),\,b_i(t)),
    \qquad
    a_i(t) := a_i(0) + \kappa_i N_i(t),
    \quad
    b_i(t) := b_i(0) + \kappa_i\sum_{n=1}^{N_i(t)} Y_{i,n}.
\end{equation*}
The posterior density is $(b_i^{a_i}/\Gamma(a_i))\,\theta^{-a_i-1} e^{-b_i/\theta}$ on $\theta > 0$. Then
\begin{equation*}
    \mu_i(t) = \frac{b_i}{a_i - 1},
    \qquad
    v_i(t) = \frac{b_i^2}{(a_i - 1)^2(a_i - 2)},
    \qquad
    F_{i,t}(q) = \Gamma(a_i,\, b_i/q)
    \quad(q > 0),
\end{equation*}
and $q^*(t)$ solves
\begin{equation*}
    \sum_{i=1}^I \gamma(a_i,\, b_i/q) = k.
\end{equation*}
We next compute the expected post-round-$t$ posterior variance. Let $W := \kappa_i\sum_{\ell=1}^{\Delta_i} Y_{i,\ell}^{\mathrm{new}}$. Each gamma observation contributes likelihood proportional to $\theta^{-\kappa_i} e^{-\kappa_i y/\theta}$, so
\begin{equation*}
    \theta_i\mid\cH_t,\,W\sim\mathrm{InvGamma}(a_i + \kappa_i \Delta_i,\,b_i + W),
    \qquad
    \Var(\theta_i\mid\cH_t,\,W) = \frac{(b_i + W)^2}{(a_i + \kappa_i \Delta_i - 1)^2(a_i + \kappa_i \Delta_i - 2)}.
\end{equation*}
Conditional on $\theta_i$, $W\sim\mathrm{Gamma}(\kappa_i \Delta_i, \theta_i)$ (shape $\kappa_i \Delta_i$, scale $\theta_i$), so $\E_t[W] = \kappa_i \Delta_i\,b_i/(a_i\!-\!1)$ and $\E_t[W^2] = \kappa_i \Delta_i(\kappa_i \Delta_i\!+\!1)\,b_i^2/[(a_i\!-\!1)(a_i\!-\!2)]$, giving
\begin{equation*}
    \E_t[(b_i + W)^2] = \frac{b_i^2(a_i + \kappa_i \Delta_i - 1)(a_i + \kappa_i \Delta_i - 2)}{(a_i - 1)(a_i - 2)}
\end{equation*}
and therefore
\begin{equation*}
\begin{aligned}
    \bar v_i(t,\Delta_i)
    &= \frac{b_i^2}{(a_i - 1)(a_i - 2)(a_i + \kappa_i \Delta_i - 1)}
    = v_i(t)\,\frac{a_i - 1}{a_i + \kappa_i \Delta_i - 1}, \\
    c_i(t,\Delta_i)^2
    &= \frac{a_i - 1}{a_i + \kappa_i \Delta_i - 1}
    = \frac{1}{1 + (\kappa_i/(a_i - 1))\Delta_i}.
\end{aligned}
\end{equation*}
This has the form \eqref{eq:contraction-form}, so GBB is exact.

Finally, for $z > 0$,
\begin{equation*}
    \E_t[(z - \theta_i)^+] = z\,\Gamma(a_i, b_i/z) - \frac{b_i}{a_i - 1}\,\Gamma(a_i - 1, b_i/z).
\end{equation*}
Combined with $c_i(t,\Delta_i)$ and \eqref{eq:ell-general}, this gives the closed form for $\ell_{i,t}(\Delta_i)$.

\subsection{Lognormal observations with known log-variance}\label{subsec:lognormal}

Suppose $Y_{i,n}\mid\theta_i\sim\mathrm{Lognormal}(\log\theta_i\!-\!\sigma_i^2/2,\,\sigma_i^2)$ with $\sigma_i^2$ known, parameterized so $\E[Y_{i,n}\mid\theta_i] = \theta_i$. Let $\eta_i := \log\theta_i$ and $X_{i,n} := \log Y_{i,n}\!+\!\sigma_i^2/2$, so $X_{i,n} = \eta_i\!+\!\varepsilon_{i,n}$ with $\varepsilon_{i,n}\sim N(0,\sigma_i^2)$. With prior $\eta_i\sim N(m_i(0),s_i^2(0))$, the transformed problem in $\eta_i$ is Gaussian-Gaussian conjugate:
\begin{equation*}
    \eta_i\mid\cH_t\sim N(m_i(t),\,s_i^2(t)),
    \qquad
    s_i^2(t) = \FitParens{\frac{1}{s_i^2(0)} + \frac{N_i(t)}{\sigma_i^2}}^{-1},
    \quad
    m_i(t) = s_i^2(t)\!\FitParens{\frac{m_i(0)}{s_i^2(0)} + \frac{1}{\sigma_i^2}\sum_{n=1}^{N_i(t)} X_{i,n}}.
\end{equation*}
Hence $\theta_i = e^{\eta_i}$ has lognormal posterior with
\begin{equation*}
    \mu_i(t) = \exp\!\FitParens{m_i + \frac{s_i^2}{2}},
    \qquad
    v_i(t) = (e^{s_i^2} - 1)\,e^{2m_i + s_i^2},
    \qquad
    F_{i,t}(q) = \Phi\!\FitParens{\frac{\log q - m_i}{s_i}}
    \quad(q > 0),
\end{equation*}
and $q^*(t)$ solves
\begin{equation*}
    \sum_{i=1}^I \FitDelims{[}{]}{1 - \Phi\!\FitParens{\frac{\log q - m_i}{s_i}}} = k.
\end{equation*}
We next compute the expected post-round-$t$ posterior variance of $\theta_i = e^{\eta_i}$, rather than of $\eta_i$. After $\Delta_i$ further samples, $\eta_i$ has Gaussian posterior with deterministic variance
\begin{equation*}
    s_i^2(\Delta_i) := \FitParens{\frac{1}{s_i^2} + \frac{\Delta_i}{\sigma_i^2}}^{-1}
\end{equation*}
and random mean $M_i(\Delta_i)$. The Gaussian update and the law of total variance give
\begin{equation*}
    M_i(\Delta_i)\mid\cH_t\sim N(m_i,\,s_i^2 - s_i^2(\Delta_i)).
\end{equation*}
Conditional on $M_i(\Delta_i)$, $\theta_i$ is lognormal with parameters $(M_i(\Delta_i), s_i^2(\Delta_i))$, so
\begin{equation*}
    \Var(\theta_i\mid M_i(\Delta_i)) = (e^{s_i^2(\Delta_i)} - 1)\,e^{2M_i(\Delta_i) + s_i^2(\Delta_i)}.
\end{equation*}
The Gaussian MGF gives $\E_t[e^{2M_i(\Delta_i)}] = \exp(2m_i\!+\!2(s_i^2\!-\!s_i^2(\Delta_i)))$, so
\begin{equation*}
    \bar v_i(t,\Delta_i) = (e^{s_i^2(\Delta_i)} - 1)\,e^{2m_i + 2s_i^2 - s_i^2(\Delta_i)},
    \qquad
    c_i(t,\Delta_i)^2 = \frac{1 - e^{-s_i^2(\Delta_i)}}{1 - e^{-s_i^2}}.
\end{equation*}
This does not have the form \eqref{eq:contraction-form} in general, so \cref{prop:greedy-conditions}(ii) does not apply.

Finally, for $z > 0$,
\begin{equation*}
    \E_t[(z - \theta_i)^+] = z\,\Phi\!\FitParens{\frac{\log z - m_i}{s_i}} - \mu_i(t)\,\Phi\!\FitParens{\frac{\log z - m_i - s_i^2}{s_i}}.
\end{equation*}
Combined with $c_i(t,\Delta_i)$ and \eqref{eq:ell-general}, this yields the closed form for $\ell_{i,t}(\Delta_i)$.

\section{Proofs for \cref{sec:asymptotic}}\label{app:asymptotic-proofs}

\subsection{Proof of consistency}
To prove \cref{prop:eoc-bpcs-consistency} we first prove three useful lemmas.

\begin{lemma}[Conditional posterior consistency]\label{lem:doob-consistency}
Under \cref{ass:model}, there exists a set $\Theta^* \subseteq \cJ^I$ with prior probability one such that, for every $\boldsymbol\theta^0 \in \Theta^*$, almost surely under $\Pp_{\boldsymbol\theta^0}$: if $N_i(t) \to \infty$ as $t \to \infty$, then $\mu_i(t) \to \theta_i^0$ and $v_i(t) \to 0$.
\end{lemma}

\begin{proof}{Proof}
For each $i$, \cref{ass:model} provides the conditions for Doob's theorem \citep{Doob1949,Miller2018}: the mean-parametrization $\E_\theta[Y_{i,n}] = \theta$ makes the family $\{P_i(\cdot \mid \theta) : \theta \in \cJ\}$ identifiable, and $\E_{\Pi_{i,0}}[\theta_i^2] < \infty$ allows the choice $g(\theta_i) = \theta_i^2$. Doob's theorem applied to $g(\theta_i) = \theta_i$ and $g(\theta_i) = \theta_i^2$ gives a set $J_i^* \subseteq \cJ$ with $\Pi_{i,0}(J_i^*) = 1$ such that, for every $\theta_i^0 \in J_i^*$, $\Pp_{\theta_i^0}$-almost surely: if $N_i(t) \to \infty$, then $\mu_i(t) \to \theta_i^0$ and $\E_t[\theta_i^2] \to (\theta_i^0)^2$. Thus, $v_i(t) = \E_t[\theta_i^2]\!-\!\mu_i(t)^2 \to 0$. Define $\Theta^* := J_1^* \times \cdots \times J_I^* \subseteq \cJ^I.$ Since the priors are independent across items, $\Theta^*$ has prior probability $\prod_{i=1}^I \Pi_{i,0}(J_i^*) = 1$.
\end{proof}

\begin{lemma}[Uniform improvement from one observation]\label{lem:uniform-improvement}
Let $\Pi$ be a distribution on $\cJ$ with continuous positive density, mean $\mu := \E_\Pi[\theta]$, and variance $v := \Var_\Pi(\theta) > 0$. Let $\theta \sim \Pi$ and $Y \mid \theta \sim P(\cdot \mid \theta)$ with $\E[Y \mid \theta] = \theta$ and finite predictive second moment $\E[Y^2] < \infty$. Define $\bar v := \E[\Var(\theta \mid Y)]$, $c := \sqrt{\bar v / v}$, and $X_c := \mu\!+\!c(\theta\!-\!\mu)$. Then $c < 1$, and for every compact $K \subset \cJ$ there exists $\eta > 0$ such that
\begin{equation*}
    \E[(q - \theta)^+] - \E[(q - X_c)^+] \ge \eta
    \qquad \text{for every $q \in K$.}
\end{equation*}
\end{lemma}

\begin{proof}{Proof}
We first show $c < 1$. By the law of total variance,
\begin{equation*}
    v = \E[\Var(\theta \mid Y)] + \Var(\E[\theta \mid Y]) = \bar v + \Var(\E[\theta \mid Y]).
\end{equation*}
$\Var(\E[\theta \mid Y]) > 0$ since
\begin{equation*}
    \Cov(\E[\theta \mid Y], Y)= \E[\E[\theta \mid Y] \cdot Y] - \E[\E[\theta \mid Y]]\E[Y]  = \E[\theta Y] - \mu^2 = v,
\end{equation*}
so $\bar v < v$ and $c = \sqrt{\bar v / v} < 1$.

Next we show the inequality for a fixed $q$. Fix $q \in \cJ$ and let $\phi_q(x) := (q\!-\!x)^+$. Since $X_c = c\theta\!+\!(1\!-\!c)\mu$ is a convex combination of $\theta$ and $\mu$, convexity of $\phi_q$ gives,
\begin{equation*}
    \phi_q(X_c) \le c\,\phi_q(\theta) + (1-c)\,\phi_q(\mu).
\end{equation*}
Taking expectations,
\begin{equation*}
    \E[\phi_q(X_c)] \le c\,\E[\phi_q(\theta)] + (1-c)\,\phi_q(\mu) < \E[\phi_q(\theta)].
\end{equation*}
The second inequality follows from Jensen's inequality, and is strict since $\Pi$ has continuous positive density on $\cJ$ and $\cJ$ is open so $\Pp(\theta < q) > 0$ and $\Pp(\theta > q) > 0$. Define
\begin{equation*}
    \delta(q) := \E[(q - \theta)^+] - \E[(q - X_c)^+] > 0.
\end{equation*}
For any $q' \in \R$, $|(q'\!-\!x)^+\!-\!(q\!-\!x)^+| \le |q'\!-\!q|$ so $\delta(q)$ is 2-Lipschitz and thus is continuous in $q$. Thus for compact $K$, $\eta := \inf_{q \in K} \delta(q) > 0$.
\end{proof}

\begin{lemma}[Every item is sampled infinitely often]\label{lem:no-starvation} Under the hypotheses of \cref{prop:eoc-bpcs-consistency}, for every $\boldsymbol\theta^0\in\Theta^*$,
\begin{equation*}
    N_i(t)\to\infty
    \qquad\text{for every $i=1,\ldots,I$}
\end{equation*}
$\Pp_{\boldsymbol\theta^0}$-almost surely.
\end{lemma}

\begin{proof}{Proof}
Fix $\boldsymbol\theta^0\in\Theta^*$ and a sample path in the probability-one event on which the conclusions of \cref{lem:doob-consistency} hold. For the asymptotic claim, we consider the procedure of \cref{alg:sequential} extended to infinitely many rounds, with $B_t \in \mathbb{Z}_+$ for every $t$, thus $\sum_{t=1}^{T} B_t \to \infty$ as $T \to \infty$. Assume by contradiction that some items are sampled only finitely often, and define
\begin{equation*}
    F:=\{i:\sup_{t \geq 1} N_i(t) < \infty\}.
\end{equation*}
Since the sum of samples diverges, $F^c\ne\varnothing$. All items $i\in F$ eventually stop getting sampled, so for all sufficiently large $t$ we write $\Pi_{i,t}=\Pi_i^\infty$, $\mu_i(t)=\mu_i^\infty$, and $v_i(t)=v_i^\infty>0$. For each $j\notin F$, $N_j(t)\to\infty$, so by \cref{lem:doob-consistency}, $\mu_j(t)\to\theta_j^0$ and $v_j(t)\to0$. We begin by showing that each of $q^*(t)$ and $q^\dagger(t)$ is contained in a compact interval for all sufficiently large $t$. Choose $a, b \in \cJ$ with $a < \min_{j \in F^c} \theta_j^0$ and $b > \max_{j \in F^c} \theta_j^0$. For each $j \in F^c$, $\mu_j(t) \to \theta_j^0$ and $v_j(t) \to 0$, so Chebyshev's inequality gives $\Pp_t(\theta_j \ge a) \to 1$ and $\Pp_t(\theta_j \ge b) \to 0$. Moreover, since the distributions of the items in $F$ are fixed, $a$ and $b$ can be chosen such that, for all sufficiently large $t$,
\begin{equation*}
    \sum_{i=1}^I \Pp_t(\theta_i \ge a) > k,
    \qquad
    \sum_{i=1}^I \Pp_t(\theta_i \ge b) < k.
\end{equation*}
Since $q \mapsto \sum_i \Pp_t(\theta_i \ge q)$ is nonincreasing in $q$ and equals $k$ at $q^*(t)$ by \eqref{eq:q-star}, $q^*(t) \in K := [a, b]$ for all sufficiently large $t$.

For $q^\dagger(t)$, since every posterior mean converges to a point in $\cJ$, we can enlarge $K$ within $\cJ$, if necessary, so that it also contains every $\mu_i(t)$ for all sufficiently large $t$. By \eqref{eq:q-dagger}, $\mu_{[k+1]}(t)\le q^\dagger(t)\le\mu_{[k]}(t)$, so $q^\dagger(t)\in K$ as well.

Applying \cref{lem:uniform-improvement} to $\Pi_i^\infty$ with the compact set $K$ gives $\eta_i > 0$ such that
\begin{equation*}
    \E_{\Pi_i^\infty}[(q(t) - \theta_i)^+] - \E[(q(t) - X_{c_i^\infty})^+] \ge \eta_i,
\end{equation*}
where $X_{c_i^\infty} := \mu_i^\infty\!+\!c_i^\infty(\theta_i\!-\!\mu_i^\infty)$ and $c_i^\infty := \sqrt{\bar v_i^\infty(1) / v_i^\infty}$. Once $\Pi_{i,t} = \Pi_i^\infty$, the left-hand side is exactly $\ell_{i,t}(0)\!-\!\ell_{i,t}(1)$. Hence $\ell_{i,t}(0)\!-\!\ell_{i,t}(1) \ge \eta_i$ for all sufficiently large $t$. Let $c_F := \min_{i \in F} \eta_i > 0$. Then for all sufficiently large $t$,
\begin{equation*}
    \ell_{i,t}(0)-\ell_{i,t}(1)\ge c_F
    \qquad\text{for every $i\in F$.}
\end{equation*}
To get a contradiction, we next show that for every $j \notin F$, the maximum possible improvement from allocating samples to item $j$ converges to $0$.  For $j \notin F$, let
\begin{equation*}
    R_{j,t}:=\ell_{j,t}(0)-\inf_{d\ge0}\ell_{j,t}(d)
\end{equation*}
be the maximum possible improvement from allocating any number of samples to item $j$. Since $x\mapsto(q\!-\!x)^+$ is 1-Lipschitz, for any random $X$,
\begin{equation*}
    \FitDelims{|}{|}{\E[(q-X)^+]-(q-\E[X])^+}
    \le
    \E|X-\E[X]|
    \le
    \sqrt{\Var(X)}.
\end{equation*}
The current distribution for item $j$ has variance $v_j(t) = \bar{v}_j(t,0)$, and $\bar{v}_j(t,d)  \leq \bar{v}_j(t,0) $. Therefore
\begin{equation*}
    R_{j,t}\le 2\sqrt{v_j(t)}\to0.
\end{equation*}
Since $F^c$ is finite, $\sum_{j\notin F}R_{j,t}\to0$, and for all sufficiently large $t$, $\sum_{j\notin F}R_{j,t}<c_F.$ Thus, for sufficiently large $t$, the feasible allocation that places one sample on some $i_0 \in F$ and distributes the remaining $B_t\!-\!1$ samples arbitrarily achieves total improvement
\begin{equation*}
    \bigl(\ell_{i_0,t}(0) - \ell_{i_0,t}(1)\bigr) + \sum_{j \ne i_0}\bigl(\ell_{j,t}(0) - \ell_{j,t}(d_j)\bigr) \ge c_F.
\end{equation*}
By contrast, any allocation that assigns no samples to items in $F$ has a total improvement of at most $\sum_{j\notin F}R_{j,t}<c_F$. Therefore, the optimal solution allocates a sample to at least one item in $F$ for every sufficiently large $t$, contradicting the assumption. Hence $F = \varnothing$.
\end{proof}

\begin{proof}{Proof of \cref{prop:eoc-bpcs-consistency}}
Fix $\boldsymbol\theta^0\in\Theta^*$. By \cref{lem:no-starvation}, $N_i(t)\to\infty$ for every $i$, almost surely. Hence $\mu_i(t)\to\theta_i^0$ and $v_i(t)\to0$ for every $i$, almost surely. Next, for any realized $\boldsymbol\theta$,
\begin{align*}
    \mathrm{OC}(S_k(\boldsymbol\mu(t)),\boldsymbol\theta)
    &=
    \sum_{i\in S_k(\boldsymbol\theta)}\theta_i
    -
    \sum_{i\in S_k(\boldsymbol\mu(t))}\theta_i \\
    &\le
    \sum_{i\in S_k(\boldsymbol\theta)}|\theta_i-\mu_i(t)|
    +
    \sum_{i\in S_k(\boldsymbol\mu(t))}|\theta_i-\mu_i(t)| \\
    &\le
    2k\|\boldsymbol\theta-\boldsymbol\mu(t)\|_\infty.
\end{align*}
Taking expectation,
\begin{equation*}
    \mathrm{EOC}(\cH_t)
    \le
    2k\E_t\|\boldsymbol\theta-\boldsymbol\mu(t)\|_\infty
    \le
    2k\sum_{i=1}^I\E_t|\theta_i-\mu_i(t)|
    \le
    2k\sum_{i=1}^I\sqrt{v_i(t)}.
\end{equation*}
Since $v_i(t)\to0$ for every $i$, $\mathrm{EOC}(\cH_t)\to0$ almost surely.

Finally, assume $\theta^0_{[k]}>\theta^0_{[k+1]}$, and let $g:=\theta^0_{[k]}\!-\!\theta^0_{[k+1]}>0.$ Since $\mu_i(t)\to\theta_i^0$ for every $i$, eventually $|\mu_i(t)-\theta_i^0|<g/3.$ If $i\in S_k(\boldsymbol\theta^0)$ and $j\notin S_k(\boldsymbol\theta^0)$, then $\theta_i^0\ge\theta^0_{[k]}$ and $\theta_j^0\le\theta^0_{[k+1]}$. Hence, eventually,
\begin{equation*}
    \mu_i(t)
    \ge
    \theta^0_{[k]}-g/3
    >
    \theta^0_{[k+1]}+g/3
    \ge
    \mu_j(t).
\end{equation*}
Thus $S_k(\boldsymbol\mu(t))=S_k(\boldsymbol\theta^0)$ eventually almost surely.
\end{proof}

\subsection{The optimal exponential rate}\label{app:optimal-eoc-rate}

Throughout this subsection and the next, assume \cref{ass:rate-model} and $1\le B_t\le B$ for every $t$, where $B$ is a fixed positive integer. All convergence statements hold $\Pp_{\boldsymbol\theta^0}$-almost surely.

\begin{proof}{Proof of \cref{lem:pairwise-rate}}
By the definition of $\operatorname{Alt}(\boldsymbol\theta^0)$,
\begin{align*}
    \Gamma(\boldsymbol w;\boldsymbol\theta^0)
    &=\inf_{\boldsymbol x\in\operatorname{Alt}(\boldsymbol\theta^0)}
      \sum_{\ell=1}^I w_\ell d_\ell(\theta_\ell^0,x_\ell)\\
    &=\min_{i\le k<j}
      \inf_{\substack{\boldsymbol x\in\cJ^I\\x_i<x_j}}
      \sum_{\ell=1}^I w_\ell d_\ell(\theta_\ell^0,x_\ell)\\
    &=\min_{i\le k<j}
      \inf_{\substack{u,v\in\cJ\\u<v}}
      \{w_id_i(\theta_i^0,u)+w_jd_j(\theta_j^0,v)\}.
\end{align*}
The last equality follows by setting $x_\ell=\theta_\ell^0$ for $\ell\notin\{i,j\}$. We conclude by showing that the inner infimum equals
\begin{equation*}
    \min_{x\in[\theta_j^0,\theta_i^0]}
    \{w_id_i(\theta_i^0,x)+w_jd_j(\theta_j^0,x)\}.
\end{equation*}
Under \cref{ass:rate-model}, $d_\ell(\theta_\ell^0,x)$ is continuous and decreases as $x$ moves toward $\theta_\ell^0$. By continuity of $d_\ell(\theta_\ell^0,\cdot)$, we may replace $u<v$ with $u\le v$. Since moving $u$ toward $\theta_i^0$ and $v$ toward $\theta_j^0$ cannot increase the sum, we may restrict to $\theta_j^0\le u\le v\le\theta_i^0$. Within this interval, replacing $u$ by $v$ cannot increase $d_i(\theta_i^0,u)$. Thus we may take $u=v=x\in[\theta_j^0,\theta_i^0]$, and the infimum is a minimum by continuity and compactness.
\end{proof}

\begin{lemma}[EOC and posterior error rates]\label{lem:posterior-eoc-rate}
For every sampling rule,
\begin{equation*}
    -\frac{\log\mathrm{EOC}(\cH_t)}{\sum_iN_i(t)}
    -\Gamma(\widehat{\boldsymbol w}(t);\boldsymbol\theta^0)
    \longrightarrow0
    \qquad \text{and} \qquad
    -\frac{\log\Pp_t(S_k(\boldsymbol\theta)\ne S_k(\boldsymbol\theta^0))}
          {\sum_iN_i(t)}
    -\Gamma(\widehat{\boldsymbol w}(t);\boldsymbol\theta^0)
    \longrightarrow0.
\end{equation*}
\end{lemma}

\begin{proof}{Proof}
The posterior error statement follows from \cref{lem:posterior-open-set} with $O=\operatorname{Alt}(\boldsymbol\theta^0)$. To prove the EOC statement, we first bound the posterior expected loss of selecting the true top-$k$ set, $R_t:=\E_t[\mathrm{OC}(S_k(\boldsymbol\theta^0),\boldsymbol\theta)].$ We then compare $R_t$ with $\mathrm{EOC}(\cH_t)$.

For $0<\delta<U\!-\!L$, define
\begin{equation*}
    O_\delta:=\bigcup_{i\le k<j}
    \{\boldsymbol x\in\cJ^I:x_j>x_i+\delta\}.
\end{equation*}
On $O_\delta$, the opportunity cost of selecting $S_k(\boldsymbol\theta^0)$ is at least $\delta$. Also, opportunity cost is at most $k(U\!-\!L)$ and is zero outside $\operatorname{Alt}(\boldsymbol\theta^0)$. Therefore,
\begin{equation}\label{eq:app-fixed-selection-bounds}
    \delta\Pp_t(\boldsymbol\theta\in O_\delta)
    \le R_t
    \le k(U-L)\Pp_t(\boldsymbol\theta\in\operatorname{Alt}(\boldsymbol\theta^0)).
\end{equation}
Applying \cref{lem:posterior-open-set} to the probabilities in \eqref{eq:app-fixed-selection-bounds} gives
\begin{equation*}
    \Gamma(\widehat{\boldsymbol w}(t);\boldsymbol\theta^0)-o(1)
    \le -\frac{\log R_t}{\sum_iN_i(t)}\le \inf_{\boldsymbol y\in O_\delta}
        \sum_\ell\widehat w_\ell(t)d_\ell(\theta_\ell^0,y_\ell)
        +o(1).
\end{equation*}
It therefore suffices to show that for every $\varepsilon>0$, there is a $\delta>0$ such that, for every $\boldsymbol w\in\Delta_I$,
\begin{equation}\label{eq:app-separated-alternatives}
    \Gamma(\boldsymbol w;\boldsymbol\theta^0)
    \le\inf_{\boldsymbol y\in O_\delta}
        \sum_\ell w_\ell d_\ell(\theta_\ell^0,y_\ell)
    \le\Gamma(\boldsymbol w;\boldsymbol\theta^0)+\varepsilon.
\end{equation}
The first inequality follows from $O_\delta\subseteq\operatorname{Alt}(\boldsymbol\theta^0)$. For the second, choose $\delta>0$ small enough that $x\!\pm\!\delta\in\cJ$ and
\begin{equation*}
    |d_i(\theta_i^0,x\pm\delta)-d_i(\theta_i^0,x)|\le\varepsilon
    \qquad\text{for every $i$ and $x\in[\theta_I^0,\theta_1^0]$.}
\end{equation*}
Fix $\boldsymbol w$ and choose $i,j,x$ attaining the minimum in \eqref{eq:posterior-rate}. Setting $y_i=x\!-\!\delta$, $y_j=x\!+\!\delta$, and $y_\ell=\theta_\ell^0$ for all other items gives a vector in $O_\delta$. Hence
\begin{align*}
    \inf_{\boldsymbol y\in O_\delta}\sum_\ell w_\ell d_\ell(\theta_\ell^0,y_\ell)
    &\le w_id_i(\theta_i^0,x-\delta)+w_jd_j(\theta_j^0,x+\delta)\\
    &\le w_id_i(\theta_i^0,x)+w_jd_j(\theta_j^0,x)
          +\varepsilon(w_i+w_j)\\
    &\le\Gamma(\boldsymbol w;\boldsymbol\theta^0)+\varepsilon.
\end{align*}
This proves \eqref{eq:app-separated-alternatives}. Combining it with the preceding bounds on $-\log R_t/\sum_iN_i(t)$ and letting $\varepsilon\downarrow0$ gives
\begin{equation}\label{eq:app-fixed-selection-rate}
    -\frac{\log R_t}{\sum_iN_i(t)}
    -\Gamma(\widehat{\boldsymbol w}(t);\boldsymbol\theta^0)
    \longrightarrow0.
\end{equation}
We next compare $R_t$ with $\mathrm{EOC}(\cH_t)$. Let $g:=\theta_k^0\!-\!\theta_{k+1}^0>0$ and $V:=\{\boldsymbol x\in\cJ^I:\|\boldsymbol x\!-\!\boldsymbol\theta^0\|_\infty<g/4\}$ so that for every $\boldsymbol x\in V$, $S_k(\boldsymbol x)=S_k(\boldsymbol\theta^0)$. If $S_k(\boldsymbol\mu(t))\ne S_k(\boldsymbol\theta^0)$, its opportunity cost is at least $g/2$ when $\boldsymbol\theta\in V$, thus $\mathrm{EOC}(\cH_t) \ge (g/2)\Pp_t(\boldsymbol\theta\in V).$ If  $S_k(\boldsymbol\mu(t))=S_k(\boldsymbol\theta^0)$, then $\mathrm{EOC}(\cH_t)=R_t$. Since selecting the largest posterior means minimizes posterior expected opportunity cost, we also have $\mathrm{EOC}(\cH_t)\le R_t$ in both cases. Hence
\begin{equation*}
    \min\{R_t,(g/2)\Pp_t(\boldsymbol\theta\in V)\}
    \le\mathrm{EOC}(\cH_t)\le R_t.
\end{equation*}
Since $\boldsymbol\theta^0\in V$, the infimum in \cref{lem:posterior-open-set} is zero for $O=V$, giving
\begin{equation*}
    -\frac{\log\Pp_t(\boldsymbol\theta\in V)}
          {\sum_iN_i(t)}
    \longrightarrow0.
\end{equation*}
Taking logarithms and using \eqref{eq:app-fixed-selection-rate} proves the EOC statement.
\end{proof}

\begin{proof}{Proof of \cref{prop:optimal-eoc-rate}}
For part (i), $\Gamma$ is continuous in $\boldsymbol w$ by \eqref{eq:posterior-rate}, since each inner minimum is over a fixed compact interval. It therefore attains a finite maximum on $\Delta_I$. If every $w_i>0$, each pair minimum is positive because the two KL terms cannot both vanish when $\theta_i^0>\theta_j^0$. Thus $\Gamma^*(\boldsymbol\theta^0)>0$.

We will also use that every maximizing allocation has positive elements. To see this, take $w_i=0$ and choose some $j$ on the other side of the true top-$k$ boundary. Evaluating the pair at $x=\theta_j^0$ gives
\begin{equation*}
    0\le\Gamma(\boldsymbol w;\boldsymbol\theta^0)
    \le w_id_i(\theta_i^0,\theta_j^0)+w_jd_j(\theta_j^0,\theta_j^0)=0.
\end{equation*}
Such an allocation cannot attain the positive maximum.

For part (ii), \cref{lem:posterior-eoc-rate} and the definition of $\Gamma^*$ give
\begin{equation*}
    \limsup_{t\to\infty}
    -\frac{\log\mathrm{EOC}(\cH_t)}{\sum_iN_i(t)}
    \le\Gamma^*(\boldsymbol\theta^0).
\end{equation*}
For part (iii), if $\widehat{\boldsymbol w}(t)\to\boldsymbol w^*$ for a maximizing allocation $\boldsymbol w^*$, continuity gives
\begin{equation*}
    -\frac{\log\mathrm{EOC}(\cH_t)}{\sum_iN_i(t)}
    \longrightarrow\Gamma(\boldsymbol w^*;\boldsymbol\theta^0)
    =\Gamma^*(\boldsymbol\theta^0).
\end{equation*}
The posterior error result follows from \cref{lem:posterior-eoc-rate}.
\end{proof}

\subsection{The rate of TEAk}\label{app:allocation-proofs}\label{app:posterior-estimates}

We now prove \cref{thm:threshold-allocation,thm:half-optimal-eoc-rate}. We first establish \cref{thm:threshold-allocation} for $q(t)\to q_0\in(\theta_{k+1}^0,\theta_k^0)$. We then prove \cref{rem:unique-threshold-families,lem:rate-threshold-convergence} and apply \cref{thm:threshold-allocation} with $q_0=q^\dagger$. Finally, we compare the resulting rate with $\Gamma^*(\boldsymbol\theta^0)$ to prove \cref{thm:half-optimal-eoc-rate}.

To prove \cref{thm:threshold-allocation}, we first determine the exponential rate of $D_{i,t}(d)$. For this, we need the following lemma on the posterior density, mean, and variance. The mean and variance convergence was already established in \cref{lem:doob-consistency} for prior-almost every true parameter. Here, under the stronger assumptions of \cref{ass:rate-model}, we establish it for every fixed true parameter, together with the additional estimates needed to determine the rate of $D_{i,t}(d)$. The proof relies on \citet[Proposition~4 and Lemmas~EC.5, EC.7, and EC.8]{russo2020}.

\begin{lemma}[Posterior density, mean, and variance]\label{lem:rate-posterior-moments}
Let $\Pi_i^{(n)}$ be the posterior after the first $n$ observations of item $i$, with density $\pi_i^{(n)}$, mean $\mu_{i,n}$, and variance $v_{i,n}$. Thus $\Pi_i^{(0)}=\Pi_{i,0}$ and $\Pi_{i,t}=\Pi_i^{(N_i(t))}$. For every $i$,
\begin{equation}\label{eq:app-posterior-density-rate}
    \varepsilon_{i,n}:=\sup_{x\in\cJ}
    \FitDelims{|}{|}{\frac1n\log\frac{\pi_i^{(n)}(x)}{\pi_{i,0}(x)}+d_i(\theta_i^0,x)}
    \longrightarrow0,
\end{equation}
\begin{equation*}
    \mu_{i,n}\longrightarrow\theta_i^0,
    \qquad v_{i,n}\longrightarrow0,
    \qquad -\frac1n\log v_{i,n}\longrightarrow0.
\end{equation*}
If $N_i(t)\to\infty$, then for every fixed integer $m\ge0$,
\begin{equation}\label{eq:app-finite-observation-variance}
    \frac{\bar v_i(t,m)}{v_i(t)}\longrightarrow1.
\end{equation}
\end{lemma}

\begin{proof}{Proof}
Fix $i$ and let $\widetilde\pi_i^{(n)}$ be the posterior density of $\eta_i(\theta_i)$. By \citet[Lemmas~EC.5, EC.7, and EC.8]{russo2020}, with the latter two applied over $\mathcal E_i$,
\begin{equation*}
    \sup_{x\in\cJ}
    \FitDelims{|}{|}{\frac1n\log\widetilde\pi_i^{(n)}(\eta_i(x))
          +d_i(\theta_i^0,x)}\longrightarrow0.
\end{equation*}
Since $\widetilde\pi_{i,0}$ is bounded above and bounded away from zero, the same limit holds after dividing $\widetilde\pi_i^{(n)}$ by $\widetilde\pi_{i,0}$. The posterior-to-prior density ratio is unchanged by the transformation $\eta_i$, proving \eqref{eq:app-posterior-density-rate}.

We next establish convergence of the posterior mean and variance. By \citet[Proposition~4]{russo2020}, the posterior concentrates at $\theta_i^0$. Since $\cJ$ is bounded, this gives $\mu_{i,n}\to\theta_i^0$ and $v_{i,n}\to0$. His proof of that proposition in Section~EC.6.2 also shows that, for every sufficiently small $\delta>0$, there is a constant $c_{i,\delta}>0$ such that, for all sufficiently large $n$,
\begin{equation}\label{eq:app-posterior-outside-interval}
    \Pi_i^{(n)}(|\theta_i-\theta_i^0|\ge\delta)
    \le e^{-nc_{i,\delta}}.
\end{equation}

We next show that $-n^{-1}\log v_{i,n}\to0$. Fix $\varepsilon>0$ and choose $\delta>0$ so that $V:=(\theta_i^0\!-\!2\delta,\theta_i^0\!-\!\delta)\subset\cJ$ and $d_i(\theta_i^0,x)\le\varepsilon$ for every $x\in V$. Since $\mu_{i,n}\to\theta_i^0$, eventually $\mu_{i,n}\ge\theta_i^0\!-\!\delta/2$, giving
\begin{equation*}
    v_{i,n}
    \ge\int_V(x-\mu_{i,n})^2\Pi_i^{(n)}(dx)
    \ge\frac{\delta^2}{4}\Pi_i^{(n)}(V).
\end{equation*}
Taking logarithms and applying \cref{lem:posterior-open-set},
\begin{equation*}
    \limsup_{n\to\infty}-\frac1n\log v_{i,n}
    \le\inf_{x\in V}d_i(\theta_i^0,x)
    \le\varepsilon.
\end{equation*}
Letting $\varepsilon\downarrow0$ and using that the variance is bounded gives  $-n^{-1}\log v_{i,n}\to0$.

It remains to prove \eqref{eq:app-finite-observation-variance}. The claim is trivial for $m=0$. Fix $m\ge1$. The law of total variance gives $\bar v_i(t,m)/v_i(t)\le1$. We will show that $\liminf_{t\to\infty}\bar v_i(t,m)/v_i(t)\ge 1$.

Let $Y_1,\ldots,Y_m$ denote the new observations of item $i$. The expectation $\E_{\theta_i^0}$ is over these additional observations only. By Bayes' formula and minimization of expected squared error,
\begin{align*}
    \bar v_i(t,m)
    &=\E_t\!\FitDelims{[}{]}{
      \frac{\displaystyle
        \min_{a\in\R}\int_\cJ(x-a)^2
        \prod_{r=1}^m p_i(Y_r\mid x)\Pi_{i,t}(dx)}
      {\displaystyle
        \int_\cJ\prod_{r=1}^m p_i(Y_r\mid z)\Pi_{i,t}(dz)}
      }\\
    &=\E_{\theta_i^0}\!\FitDelims{[}{]}{
      \frac{\displaystyle
        \int_\cJ\prod_{r=1}^m p_i(Y_r\mid z)\Pi_{i,t}(dz)}
      {\displaystyle\prod_{r=1}^m p_i(Y_r\mid\theta_i^0)}
      \times
      \frac{\displaystyle
        \min_{a\in\R}\int_\cJ(x-a)^2
        \prod_{r=1}^m p_i(Y_r\mid x)\Pi_{i,t}(dx)}
      {\displaystyle
        \int_\cJ\prod_{r=1}^m p_i(Y_r\mid z)\Pi_{i,t}(dz)}
      }\\
    &=\E_{\theta_i^0}\!\FitDelims{[}{]}{
      \min_{a\in\R}\int_\cJ(x-a)^2
      \prod_{r=1}^m
      \frac{p_i(Y_r\mid x)}{p_i(Y_r\mid\theta_i^0)}
      \Pi_{i,t}(dx)}.
\end{align*}

We first bound the logarithm of the likelihood product term. We have
\begin{align*}
    \log\prod_{r=1}^m
    \frac{p_i(Y_r\mid x)}{p_i(Y_r\mid\theta_i^0)}
    &=
    \{\eta_i(x)-\eta_i(\theta_i^0)\}
    \sum_{r=1}^m T_i(Y_r)
    -m\{A_i(\eta_i(x))-A_i(\eta_i(\theta_i^0))\}\\
    &\ge
    -|\eta_i(x)-\eta_i(\theta_i^0)|
     \FitDelims{|}{|}{\sum_{r=1}^m T_i(Y_r)}
    -m|A_i(\eta_i(x))-A_i(\eta_i(\theta_i^0))|.
\end{align*}
To bound $\FitDelims{|}{|}{\sum_{r=1}^m T_i(Y_r)}$, fix $\varepsilon\in(0,1)$ and choose $M$ so that
\begin{equation*}
    \Pp_{\theta_i^0}\!
    \FitParens{\max_{1\le r\le m}|T_i(Y_r)|\le M}
    \ge1-\varepsilon.
\end{equation*}
If $\max_{1\le r\le m}|T_i(Y_r)|\le M$, then $\FitDelims{|}{|}{\sum_{r=1}^m T_i(Y_r)}\le mM$. By continuity of $\eta_i$ and $A_i$, choose $V=[\theta_i^0\!-\!\delta,\theta_i^0\!+\!\delta]\subset\cJ$ so that for $x\in V$
\begin{equation*}
    mM|\eta_i(x)-\eta_i(\theta_i^0)|
    +m|A_i(\eta_i(x))-A_i(\eta_i(\theta_i^0))|
    \le\varepsilon.
\end{equation*}
Thus, whenever $\max_{1\le r\le m}|T_i(Y_r)|\le M$, for $x\in V$ we have
\begin{equation*}
    \prod_{r=1}^m
    \frac{p_i(Y_r\mid x)}{p_i(Y_r\mid\theta_i^0)}
    \ge e^{-\varepsilon}.
\end{equation*}
Plugging this into the expression for $\bar v_i(t,m)$ gives
\begin{equation*}
    \bar v_i(t,m)
    \ge e^{-\varepsilon}(1-\varepsilon)
       \min_{a\in\R}\int_V(x-a)^2\Pi_{i,t}(dx)
    \ge e^{-\varepsilon}(1-\varepsilon)
       \bigl\{v_i(t)-(U-L)^2\Pi_{i,t}(V^c)\bigr\}.
\end{equation*}
For the second inequality, the minimizing $a$ belongs to $V$. The integral over $V^c$ is therefore at most $(U\!-\!L)^2\Pi_{i,t}(V^c)$, while the integral over $\cJ$ is at least $v_i(t)$.

By \eqref{eq:app-posterior-outside-interval} and $-n^{-1}\log v_{i,n}\to0$, $\Pi_{i,t}(V^c)/v_i(t)\to0$ when $N_i(t)\to\infty$. Dividing by $v_i(t)$ and taking the liminf gives
\begin{equation*}
    \liminf_{t\to\infty}\frac{\bar v_i(t,m)}{v_i(t)}
    \ge e^{-\varepsilon}(1-\varepsilon).
\end{equation*}
Letting $\varepsilon\downarrow0$ proves \eqref{eq:app-finite-observation-variance}.
\end{proof}

\begin{proof}{Proof of \cref{thm:threshold-allocation}}
We first show that every item is sampled infinitely often. Since $q(t)\to q_0\in(\theta_{k+1}^0,\theta_k^0)$, eventually $q(t)$ belongs to a compact interval $K\subset\cJ$. The proof of \cref{lem:no-starvation} then applies, using \cref{lem:rate-posterior-moments} for items sampled infinitely often and \cref{lem:uniform-improvement} for items that stop receiving samples. Thus $N_i(t)\to\infty$ and $\mu_i(t)\to\theta_i^0$ for every $i$.

For part (i), we first establish the rates of $D_{i,t}(d)$. We will show that
\begin{equation}\label{eq:app-gain-exponent}
    \max_{0\le d<B}
    \FitDelims{|}{|}{-\frac{\log D_{i,t}(d)}{N_i(t)}-d_i(\theta_i^0,q_0)}
    \longrightarrow0.
\end{equation}
We will then use these rates and optimality of each batch to prove $\widehat w_i(t)\to w_i(q_0)$. To prove \eqref{eq:app-gain-exponent}, define
\begin{equation*}
    J_{i,t}(d):=\ell_{i,t}(d)-(q(t)-\mu_i(t))^+.
\end{equation*}
We first establish the bounds
\begin{equation}\label{eq:app-gain-sandwich}
    \frac{\bar v_i(t,B-1)}{2\Var_t(Y)}J_{i,t}(B-1)
    \le D_{i,t}(d)\le J_{i,t}(0),
    \qquad 0\le d<B,
\end{equation}
where $Y$ is an additional observation of item $i$.

By \eqref{eq:affine}, $\widetilde\theta_i(d\!+\!1)$ is a contraction of $\widetilde\theta_i(d)$ about $\mu_i(t)$ with factor $c_i(t,d\!+\!1)/c_i(t,d)$. The convexity argument in the proof of \cref{lem:uniform-improvement} gives
\begin{equation}\label{eq:app-convex-gain-bound}
    0
    \le\FitParens{1-\frac{c_i(t,d+1)}{c_i(t,d)}}J_{i,t}(d)
    \le D_{i,t}(d)
    =J_{i,t}(d)-J_{i,t}(d+1)
    \le J_{i,t}(d)
    \le J_{i,t}(0).
\end{equation}
This proves the upper bound in \eqref{eq:app-gain-sandwich}. For the lower bound, we next bound $1\!-\!c_i(t,d\!+\!1)/c_i(t,d)$ from below.

Let $Z_d=(Y_1,\ldots,Y_d)$ be the next $d$ observations of item $i$, and let $Y=Y_{d+1}$. Applying the variance and covariance identities in the proof of \cref{lem:uniform-improvement} conditional on $(\cH_t,Z_d)$, followed by Cauchy--Schwarz, gives
\begin{equation*}
    \bar v_i(t,d)-\bar v_i(t,d+1)
    \ge\E_t\!\FitDelims{[}{]}{
       \frac{\Var(\theta_i\mid\cH_t,Z_d)^2}
            {\Var(Y\mid\cH_t,Z_d)}}
    \ge\frac{\bar v_i(t,d)^2}
              {\E_t[\Var(Y\mid\cH_t,Z_d)]}
    \ge\frac{\bar v_i(t,d)^2}{\Var_t(Y)}.
\end{equation*}
The last inequality follows from the law of total variance. Therefore,
\begin{equation*}
\begin{aligned}
    1-\frac{c_i(t,d+1)}{c_i(t,d)}
    =\frac{\bar v_i(t,d)-\bar v_i(t,d+1)}
            {\bar v_i(t,d)
             (1+c_i(t,d+1)/c_i(t,d))}
    \ge\frac{\bar v_i(t,d)}{2\Var_t(Y)}
     \ge\frac{\bar v_i(t,B-1)}{2\Var_t(Y)}.
\end{aligned}
\end{equation*}
Also, \eqref{eq:app-convex-gain-bound} gives $J_{i,t}(d)\ge J_{i,t}(B\!-\!1)$ for $d<B$. Substituting into the lower bound in \eqref{eq:app-convex-gain-bound} proves \eqref{eq:app-gain-sandwich}.

We now show that both $J_{i,t}(0)$ and $J_{i,t}(B\!-\!1)$ have exponential rate $d_i(\theta_i^0,q_0)$:
\begin{equation}\label{eq:app-loss-exponent}
    -\frac{\log J_{i,t}(d)}{N_i(t)}
    \longrightarrow d_i(\theta_i^0,q_0),
    \qquad d\in\{0,B-1\}.
\end{equation}
By \cref{lem:rate-posterior-moments},
\begin{equation*}
    c_i(t,d)\to1,
    \qquad
    z_i(t,d)
    =\mu_i(t)+\frac{q(t)-\mu_i(t)}{c_i(t,d)}
    \to q_0.
\end{equation*}
Since $\mu_i(t)\to\theta_i^0$ and $q_0\in(\theta_{k+1}^0,\theta_k^0)$, \eqref{eq:ell-general} gives, for all sufficiently large $t$,
\begin{equation*}
    J_{i,t}(d)=
    \begin{cases}
        c_i(t,d)\E_t[(z_i(t,d)-\theta_i)^+],
            &i\le k,\\
        c_i(t,d)\E_t[(\theta_i-z_i(t,d))^+],
            &i>k.
    \end{cases}
\end{equation*}
Write $n=N_i(t)$. By \eqref{eq:app-posterior-density-rate} and monotonicity of $d_i(\theta_i^0,x)$,
\begin{equation*}
    J_{i,t}(d)
    \le(U-L)e^{-n(d_i(\theta_i^0,z_i(t,d))
                         -\varepsilon_{i,n})}.
\end{equation*}
For a lower bound when $i\le k$, choose a sufficiently small $\delta>0$ and let $V=[q_0\!-\!2\delta,q_0\!-\!\delta]\subset\cJ$. Eventually $|z_i(t,d)\!-\!q_0|<\delta/2$, so
\begin{equation*}
    J_{i,t}(d)
    \ge c_i(t,d)\frac\delta2\Pi_{i,t}(V)
    \ge c_i(t,d)\frac\delta2\Pi_{i,0}(V)
       e^{-n(d_i(\theta_i^0,q_0-2\delta)
                         +\varepsilon_{i,n})}.
\end{equation*}
For $i>k$, take $V=[q_0\!+\!\delta,q_0\!+\!2\delta]$ and obtain the same bound with $q_0\!+\!2\delta$ in place of $q_0\!-\!2\delta$. Since $c_i(t,d)\to1$ and $\Pi_{i,0}(V)>0$, taking logarithms, letting $t\to\infty$, and then letting $\delta\downarrow0$ proves \eqref{eq:app-loss-exponent}.

To conclude \eqref{eq:app-gain-exponent} from \eqref{eq:app-gain-sandwich}, it remains to show that
\begin{equation*}
    -\frac1n\log
    \frac{\bar v_i(t,B-1)}{2\Var_t(Y)}
    \longrightarrow0.
\end{equation*}
The numerator is covered by \cref{lem:rate-posterior-moments}. For the denominator, let $M_i:=\int_\cJ\E_x[Y^2]\Pi_{i,0}(dx)<\infty$, by \cref{ass:model}. Equation \eqref{eq:app-posterior-density-rate} gives
\begin{equation*}
    \Var_t(Y)\le\E_t[Y^2]\le M_i e^{n\varepsilon_{i,n}}.
\end{equation*}
Since $\Var_t(Y)\ge v_i(t)\ge\bar v_i(t,B\!-\!1)$,
\begin{equation*}
    0
    \le-\frac1n\log
       \frac{\bar v_i(t,B-1)}{2\Var_t(Y)}
    \le-\frac1n\log\bar v_i(t,B-1)
          +\varepsilon_{i,n}+\frac{\log(2M_i)}n
      \longrightarrow0.
\end{equation*}
Combining this with \eqref{eq:app-loss-exponent} in \eqref{eq:app-gain-sandwich} proves \eqref{eq:app-gain-exponent}.

We now use these rates and batch optimality to show that
\begin{equation*}
    \frac{d_i(\theta_i^0,q_0)N_i(t)-\min_j d_j(\theta_j^0,q_0)N_j(t)}
         {\sum_jN_j(t)}
    \longrightarrow0
    \qquad\text{for every $i$.}
\end{equation*}
We first compare the sample counts for items sampled in the current round. Write
\begin{equation*}
    \varepsilon_t
    :=\max_{i,\,0\le d<B}
      \frac{|{-\log D_{i,t}(d)}-d_i(\theta_i^0,q_0)N_i(t)|}
           {\sum_jN_j(t)}.
\end{equation*}
By \eqref{eq:app-gain-exponent}, $\varepsilon_t\to0$. If $\Delta_i^t>0$ and $j\ne i$, transferring one sample from $i$ to $j$ is feasible. Optimality of $\boldsymbol\Delta^t$ gives
\begin{equation*}
    0
    \le\ell_{i,t}(\Delta_i^t-1)
        +\ell_{j,t}(\Delta_j^t+1)
        -\ell_{i,t}(\Delta_i^t)
        -\ell_{j,t}(\Delta_j^t)
    =D_{i,t}(\Delta_i^t-1)-D_{j,t}(\Delta_j^t).
\end{equation*}
Since $0\le\Delta_i^t\!-\!1\le B\!-\!1$ and $0\le\Delta_j^t\le B\!-\!1$,
\begin{align*}
    d_i(\theta_i^0,q_0)N_i(t)
    &\le-\log D_{i,t}(\Delta_i^t-1)
          +\varepsilon_t\sum_\ell N_\ell(t)\\
    &\le-\log D_{j,t}(\Delta_j^t)
          +\varepsilon_t\sum_\ell N_\ell(t)\\
    &\le d_j(\theta_j^0,q_0)N_j(t)
          +2\varepsilon_t\sum_\ell N_\ell(t).
\end{align*}
Taking the minimum over $j$ gives
\begin{equation}\label{eq:app-selected-count-comparison}
    \Delta_i^t>0
    \quad\Longrightarrow\quad
    d_i(\theta_i^0,q_0)N_i(t)\le\min_j d_j(\theta_j^0,q_0)N_j(t)+2\varepsilon_t\sum_jN_j(t).
\end{equation}

To extend this analysis to every item, we apply it at the last round in which that item was sampled. Let $\tau_i(t)<t$ be that round. Since $N_i(t)\to\infty$, we have $\tau_i(t)\to\infty$, and boundedness of the batches gives
\begin{equation*}
    N_i(t)
    =N_i(\tau_i(t))+\Delta_i^{\tau_i(t)}
    \le N_i(\tau_i(t))+B.
\end{equation*}
Also, $\min_j d_j(\theta_j^0,q_0)N_j(t)$ is nondecreasing. Applying \eqref{eq:app-selected-count-comparison} at $\tau_i(t)$ therefore gives
\begin{align*}
    0
    &\le\frac{d_i(\theta_i^0,q_0)N_i(t)
              -\min_j d_j(\theta_j^0,q_0)N_j(t)}
             {\sum_jN_j(t)}\\
    &\le\frac{B d_i(\theta_i^0,q_0)}{\sum_jN_j(t)}
       +2\varepsilon_{\tau_i(t)}
        \frac{\sum_jN_j(\tau_i(t))}{\sum_jN_j(t)}\\
    &\le\frac{B d_i(\theta_i^0,q_0)}{\sum_jN_j(t)}
       +2\varepsilon_{\tau_i(t)}
      \longrightarrow0.
\end{align*}
Dividing by $d_i(\theta_i^0,q_0)$, summing over $i$, and using $\sum_i\widehat w_i(t)=1$ gives
\begin{equation*}
    \frac{\min_j d_j(\theta_j^0,q_0)N_j(t)}{\sum_jN_j(t)}
    \longrightarrow
    \FitParens{\sum_i\frac1{d_i(\theta_i^0,q_0)}}^{-1}.
\end{equation*}
Hence
\begin{equation*}
    \widehat w_i(t)
    \longrightarrow
    \frac{1/d_i(\theta_i^0,q_0)}{\sum_j1/d_j(\theta_j^0,q_0)}
    =w_i(q_0),
\end{equation*}
proving part (i). Part (ii) follows from \cref{lem:posterior-eoc-rate} and continuity of $\Gamma$.

For part (iii), $w_i(q_0)d_i(\theta_i^0,q_0)=\FitParens{\sum_{\ell=1}^I1/d_\ell(\theta_\ell^0,q_0)}^{-1}$ for every $i$. For any $i\le k<j$ and $x\in[\theta_j^0,\theta_i^0]$, monotonicity gives
\begin{equation*}
    w_i(q_0)d_i(\theta_i^0,x)
    +w_j(q_0)d_j(\theta_j^0,x)
    \ge
    \begin{cases}
        w_i(q_0)d_i(\theta_i^0,q_0)=\FitParens{\sum_{\ell=1}^I1/d_\ell(\theta_\ell^0,q_0)}^{-1},&x\le q_0,\\
        w_j(q_0)d_j(\theta_j^0,q_0)=\FitParens{\sum_{\ell=1}^I1/d_\ell(\theta_\ell^0,q_0)}^{-1},&x\ge q_0.
    \end{cases}
\end{equation*}
Taking the minimum over $x$ and then over $i\le k<j$ in \eqref{eq:posterior-rate} gives $\Gamma(\boldsymbol w(q_0);\boldsymbol\theta^0)\ge\FitParens{\sum_{\ell=1}^I1/d_\ell(\theta_\ell^0,q_0)}^{-1}$.
\end{proof}

\begin{proof}{Proof of \cref{rem:unique-threshold-families}}
Fix $\boldsymbol x$ with $x_{[k]}>x_{[k+1]}$. The objective $\sum_i1/d_i(x_i,q)$ is continuous on $(x_{[k+1]},x_{[k]})$ and tends to infinity at its endpoints, so it attains a minimum. We prove uniqueness for each family by strict convexity in $q$ or in a strictly increasing function of $q$.

For Gaussian observations, each term is strictly convex in $q$, since
\begin{equation*}
    \frac1{d_i(x_i,q)}=\frac{2\sigma_i^2}{(x_i-q)^2},
    \qquad
    \frac{d^2}{dq^2}\frac1{d_i(x_i,q)}
    =\frac{12\sigma_i^2}{(x_i-q)^4}>0.
\end{equation*}
For lognormal observations, set $z=\log q$. Then
\begin{equation*}
    \frac1{d_i(x_i,e^z)}=\frac{2\sigma_i^2}{(\log x_i-z)^2},
    \qquad
    \frac{d^2}{dz^2}\frac1{d_i(x_i,e^z)}
    =\frac{12\sigma_i^2}{(\log x_i-z)^4}>0.
\end{equation*}

For Bernoulli observations, write $q=e^z/(1\!+\!e^z)$. We show that $1/d_i(x_i,q)$ is strictly convex in $z$. Differentiating gives
\begin{equation*}
    \frac{d^2}{dz^2}\frac1{d_i(x_i,q)}
    =\frac{2(q-x_i)^2-d_i(x_i,q)q(1-q)}
           {d_i(x_i,q)^3}.
\end{equation*}
To show that the numerator is positive, we bound $d_i(x_i,q)$ from above. Using $\log u\le u\!-\!1$,
\begin{equation*}
    d_i(x_i,q)
    =x_i\log\frac{x_i}{q}
      +(1-x_i)\log\frac{1-x_i}{1-q}
    \le x_i\FitParens{\frac{x_i}{q}-1}
       +(1-x_i)\FitParens{\frac{1-x_i}{1-q}-1}
    =\frac{(x_i-q)^2}{q(1-q)}.
\end{equation*}
Substituting gives
\begin{equation*}
    \frac{d^2}{dz^2}\frac1{d_i(x_i,q)}
    \ge\frac{(q-x_i)^2}{d_i(x_i,q)^3}>0,
\end{equation*}
since $q\in(x_{[k+1]},x_{[k]})$ implies $q\ne x_i$.

For Poisson and Gamma observations, write $q=e^z$. Then
\begin{equation*}
    d_i(x_i,q)=
    \begin{cases}
        x_i(e^v-1-v),\quad v=z-\log x_i,
            &\text{Poisson},\\
        \kappa_i(e^v-1-v),\quad v=\log x_i-z,
            &\text{Gamma}.
    \end{cases}
\end{equation*}
Since $x_i,\kappa_i>0$ and $dv/dz=\pm1$, it suffices to show that $1/(e^v\!-\!1\!-\!v)$ is strictly convex for $v\ne0$. Differentiating gives
\begin{equation*}
    \frac{d^2}{dv^2}\frac1{e^v-1-v}
    =\frac{e^v(e^v+2e^{-v}+v-3)}{(e^v-1-v)^3}.
\end{equation*}
The denominator is positive for $v\ne0$. To show that the numerator is positive, use $e^v>1\!+\!v$ and $e^{-v}>1\!-\!v$ to get
\begin{equation*}
    e^v+2e^{-v}+v-3
    >(1+v)+2(1-v)+v-3=0.
\end{equation*}
Thus the second derivative is positive, proving strict convexity in $z$ for both families.
\end{proof}

\begin{proof}{Proof of \cref{lem:rate-threshold-convergence}}
By \cref{lem:rate-posterior-moments} and the proof of \cref{lem:no-starvation}, $\boldsymbol\mu(t)\to\boldsymbol\theta^0$. Define
\begin{equation*}
    H_t(q):=\sum_i\frac1{d_i(\mu_i(t),q)},
    \qquad H(q):=\sum_i\frac1{d_i(\theta_i^0,q)}.
\end{equation*}
We will show that every limit of $q^\dagger(t)$ belongs to $(\theta_{k+1}^0,\theta_k^0)$ and minimizes $H$. Uniqueness then gives convergence to $q^\dagger$.

Since $\mu_{[k]}(t)\to\theta_k^0$ and $\mu_{[k+1]}(t)\to\theta_{k+1}^0$, eventually $\mu_{[k+1]}(t)<q^\dagger<\mu_{[k]}(t).$ Since $q^\dagger(t)$ minimizes $H_t$ on this interval,
\begin{equation}\label{eq:app-threshold-comparison}
    H_t(q^\dagger(t))\le H_t(q^\dagger)\longrightarrow H(q^\dagger)<\infty.
\end{equation}
Let $t_n\to\infty$ satisfy $q^\dagger(t_n)\to\bar q$. Since $\mu_{[k+1]}(t)\le q^\dagger(t)\le\mu_{[k]}(t)$, we have $\bar q\in[\theta_{k+1}^0,\theta_k^0]$. If $\bar q=\theta_k^0$ (and likewise $\bar q=\theta_{k+1}^0$), then
\begin{equation*}
    H_{t_n}(q^\dagger(t_n))
    \ge\frac1{d_k(\mu_k(t_n),q^\dagger(t_n))}\longrightarrow\infty,
\end{equation*}
contradicting \eqref{eq:app-threshold-comparison}. Therefore $\bar q \in(\theta_{k+1}^0,\theta_k^0)$, and continuity gives
\begin{equation*}
    H(\bar q)=\lim_{n\to\infty}H_{t_n}(q^\dagger(t_n))\le H(q^\dagger).
\end{equation*}
Thus $\bar q$ minimizes $H$, and \cref{ass:unique-threshold} gives $\bar q=q^\dagger$. Since $q^\dagger(t)$ is eventually contained in a compact interval and every such limit equals $q^\dagger$, we conclude that $q^\dagger(t)\to q^\dagger$.
\end{proof}

\begin{proof}{Proof of \cref{thm:half-optimal-eoc-rate}}
By \cref{lem:rate-threshold-convergence} and \cref{thm:threshold-allocation}(ii),
\begin{equation*}
    -\frac{\log\mathrm{EOC}(\cH_t)}{\sum_iN_i(t)}
    \longrightarrow\Gamma(\boldsymbol w(q^\dagger);\boldsymbol\theta^0)
    \le\Gamma^*(\boldsymbol\theta^0).
\end{equation*}
For the lower bound, write $H(q):=\sum_i1/d_i(\theta_i^0,q)$. Part (iii) of \cref{thm:threshold-allocation} gives
\begin{equation*}
    \Gamma(\boldsymbol w(q^\dagger);\boldsymbol\theta^0)
    \ge\frac1{H(q^\dagger)}.
\end{equation*}
It therefore remains to show $\Gamma^*(\boldsymbol\theta^0)/2\le1/H(q^\dagger)$.

Let $\boldsymbol w^*$ maximize $\Gamma$. By the proof of \cref{prop:optimal-eoc-rate}, $w_i^*>0$ for every $i$. The functions (of $q$)
\begin{equation*}
    \min_{i\le k}w_i^*d_i(\theta_i^0,q),
    \qquad
    \min_{j>k}w_j^*d_j(\theta_j^0,q)
\end{equation*}
are continuous on $[\theta_{k+1}^0,\theta_k^0]$. The first is positive at $\theta_{k+1}^0$ and zero at $\theta_k^0$, while the second is zero at $\theta_{k+1}^0$ and positive at $\theta_k^0$. Thus they are equal at some $\bar q\in(\theta_{k+1}^0,\theta_k^0)$. Write
\begin{equation*}
    r:= \min_{i\le k}w_i^*d_i(\theta_i^0,\bar q)
      =\min_{j>k}w_j^*d_j(\theta_j^0,\bar q)>0.
\end{equation*}
Choose $i\le k<j$ attaining the two minima. Evaluating \eqref{eq:posterior-rate} at this pair and $x=\bar q$ gives
\begin{equation*}
    \Gamma^*(\boldsymbol\theta^0)
    \le w_i^*d_i(\theta_i^0,\bar q)+w_j^*d_j(\theta_j^0,\bar q)
    =2r.
\end{equation*}
Since $w_i^*d_i(\theta_i^0,\bar q)\ge r$ for every $i$,
\begin{equation*}
    rH(\bar q)=\sum_i\frac r{d_i(\theta_i^0,\bar q)}
    \le\sum_iw_i^*=1.
\end{equation*}
Since $q^\dagger$ minimizes $H$,
\begin{equation*}
    \frac12\Gamma^*(\boldsymbol\theta^0)
    \le r\le\frac1{H(\bar q)}\le\frac1{H(q^\dagger)}
    \le\Gamma(\boldsymbol w(q^\dagger);\boldsymbol\theta^0).
\end{equation*}
The corresponding result for the posterior error follows from \cref{lem:posterior-eoc-rate}.
\end{proof}

\section{Additional simulation results}\label{app:pcs}

This section presents the remaining plots of the experiments in \cref{sec:experiments}. \Cref{app:pcs-medium} gives the EOC plots for the Bernoulli, exponential, and lognormal families and the PCS counterparts of all the EOC plots, which use the middle column in \cref{tab:instances} for distributions of the true means. \Cref{app:pcs-hard} and \cref{app:pcs-easy} give the EOC and PCS plots for the right and left columns, respectively. \Cref{tab:instances} is repeated below for reference.

\begin{table}[!htb]
\centering
\small
\begin{tabular}{l ccc}
\toprule
Likelihood $Y_{i,n}\mid\theta_i$        & \multicolumn{3}{c}{Distribution of true means $\theta_i$}\\
\midrule
$N(\theta_i,1)$                             & $\mathrm{Unif}(-2,2)$       & $\mathrm{Unif}(-1,1)$    & $\mathrm{Unif}(-0.5,0.5)$\\
$\mathrm{Exp}(1/\theta_i)$                  & $\mathrm{Unif}(1,9)$        & $\mathrm{Unif}(1,5)$     & $\mathrm{Unif}(1,3)$\\
$\mathrm{Bern}(\theta_i)$                   & $\mathrm{Unif}(0.05,0.95)$  & $\mathrm{Unif}(0.2,0.8)$ & $\mathrm{Unif}(0.35,0.65)$\\
$\mathrm{Lognormal}(\log\theta_i\!-\!1/2,\,1)$  & $\mathrm{Unif}(1,9)$        & $\mathrm{Unif}(1,5)$     & $\mathrm{Unif}(1,3)$\\
\bottomrule
\end{tabular}
\end{table}

\subsection{Middle column}\label{app:pcs-medium}

\begin{figure}[!htb]
\centering
\includegraphics[width=0.92\textwidth]{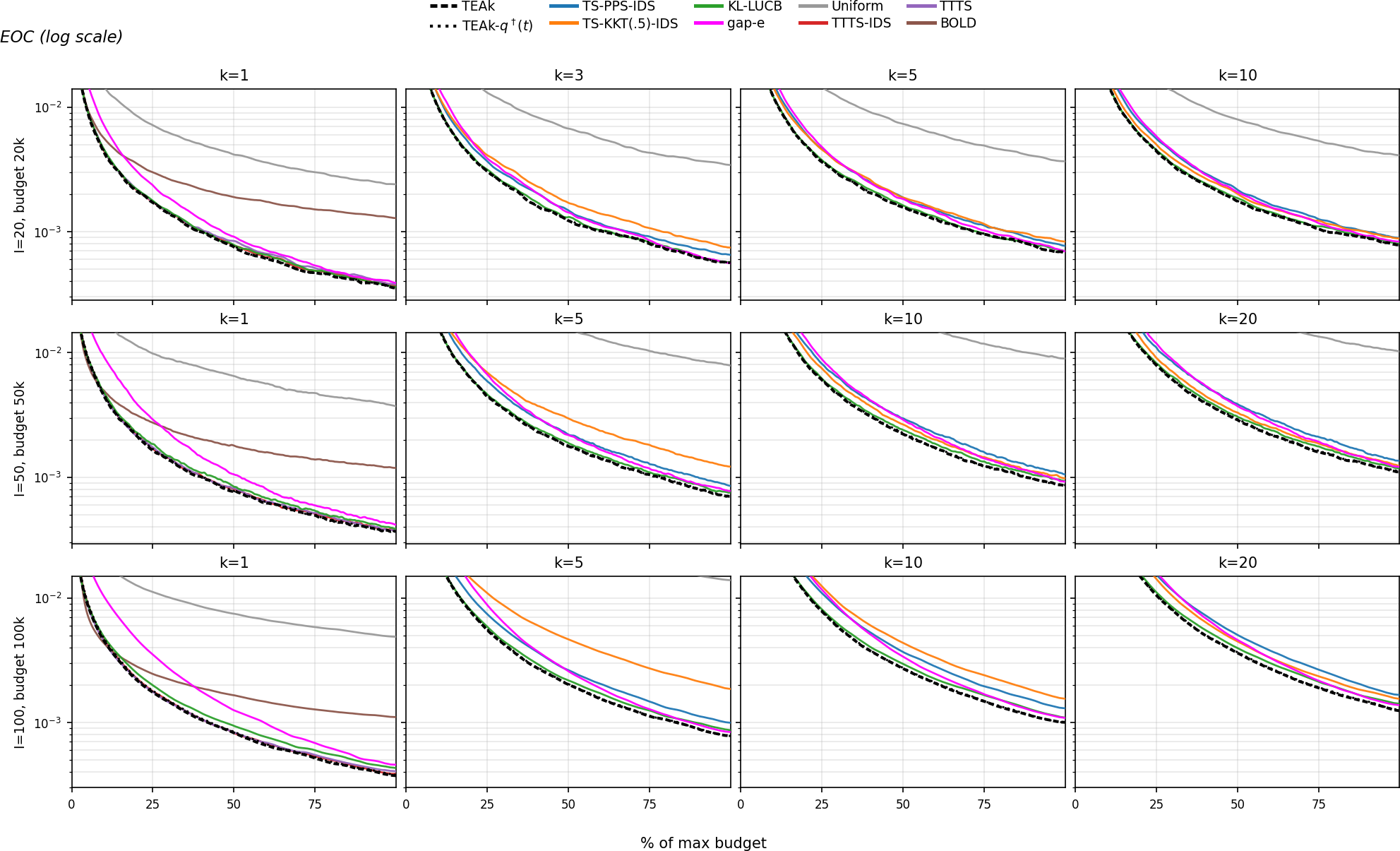}
\caption{Bernoulli: same format as \cref{fig:grid-gauss}.}
\label{fig:grid-bern}
\end{figure}

\begin{figure}[!htb]
\centering
\includegraphics[width=0.92\textwidth]{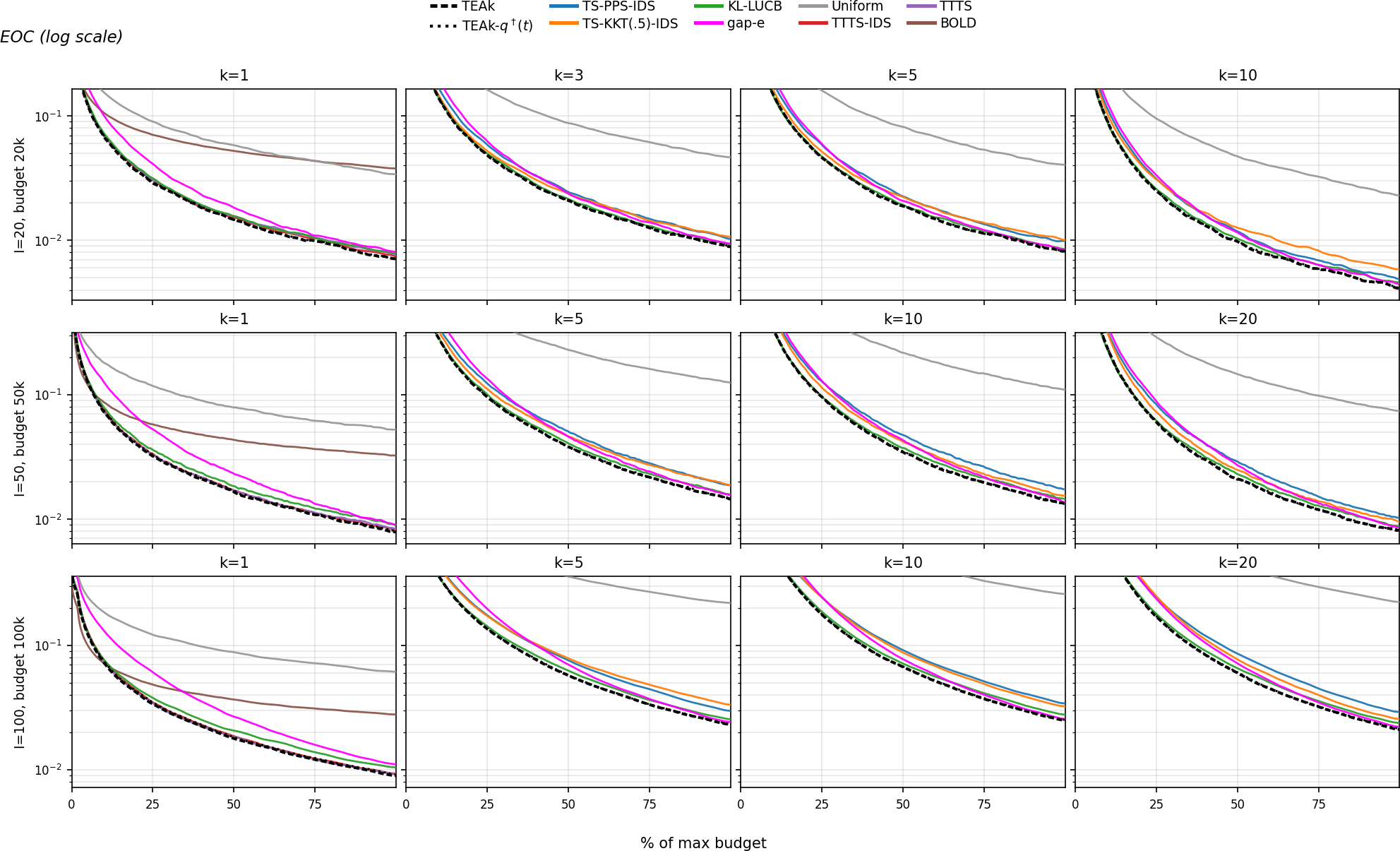}
\caption{Exponential: same format as \cref{fig:grid-gauss}.}
\label{fig:grid-exp}
\end{figure}

\begin{figure}[!htb]
\centering
\includegraphics[width=0.92\textwidth]{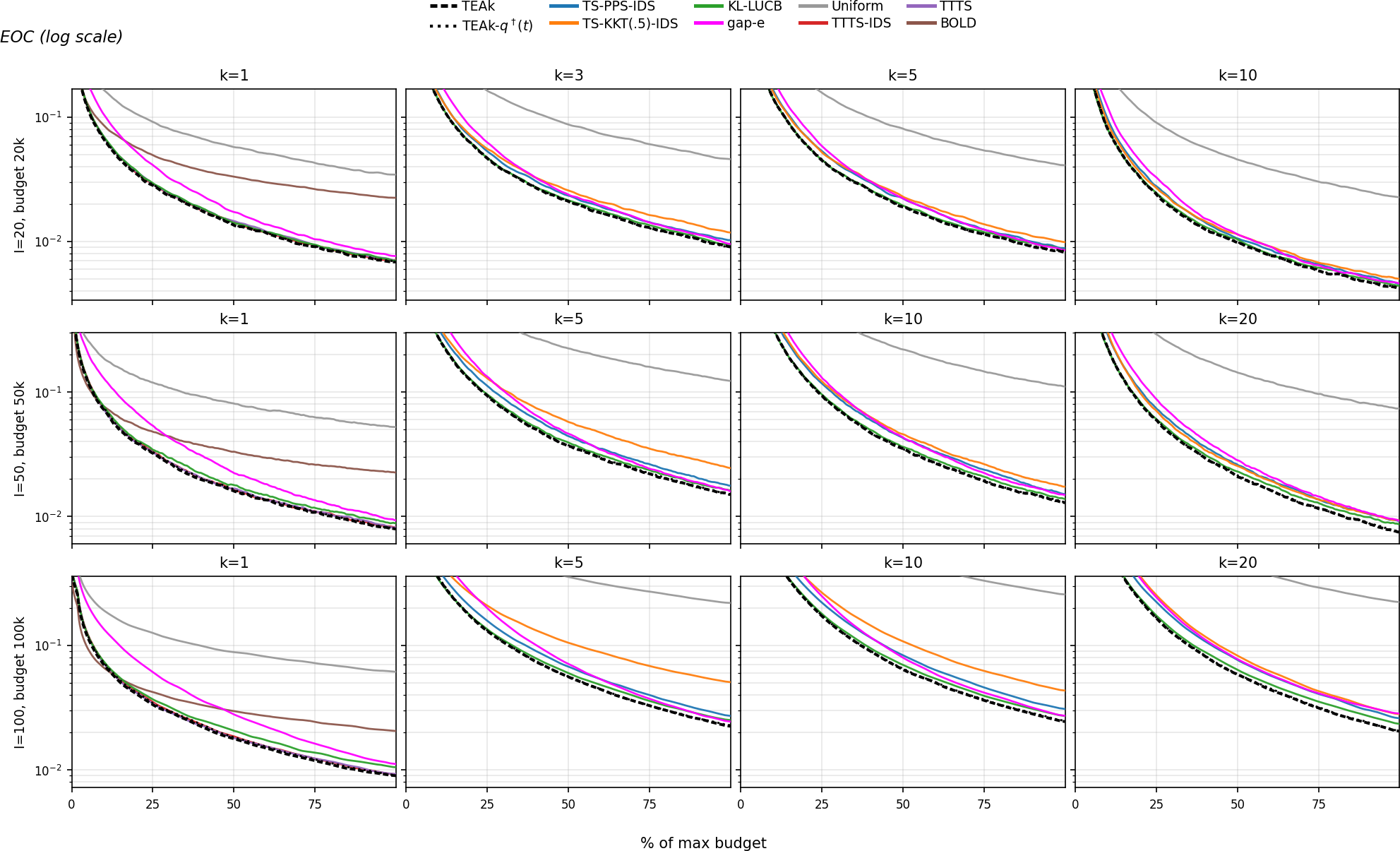}
\caption{Lognormal: same format as \cref{fig:grid-gauss}.}
\label{fig:grid-logn}
\end{figure}

\begin{figure}[!htb]
\centering
\includegraphics[width=0.92\textwidth]{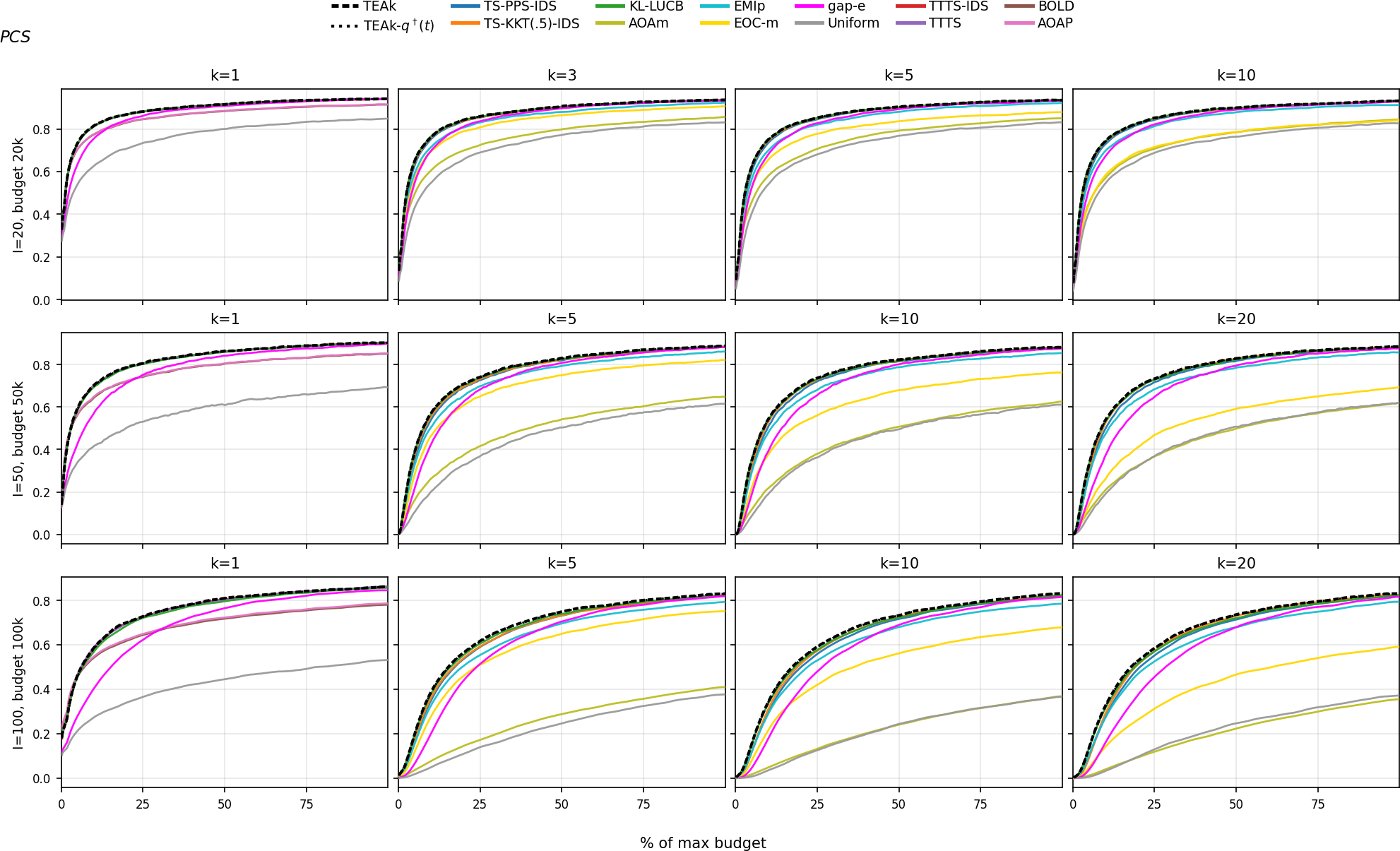}
\caption{Gaussian: PCS.}
\label{fig:grid-pcs-gauss}
\end{figure}

\begin{figure}[!htb]
\centering
\includegraphics[width=0.92\textwidth]{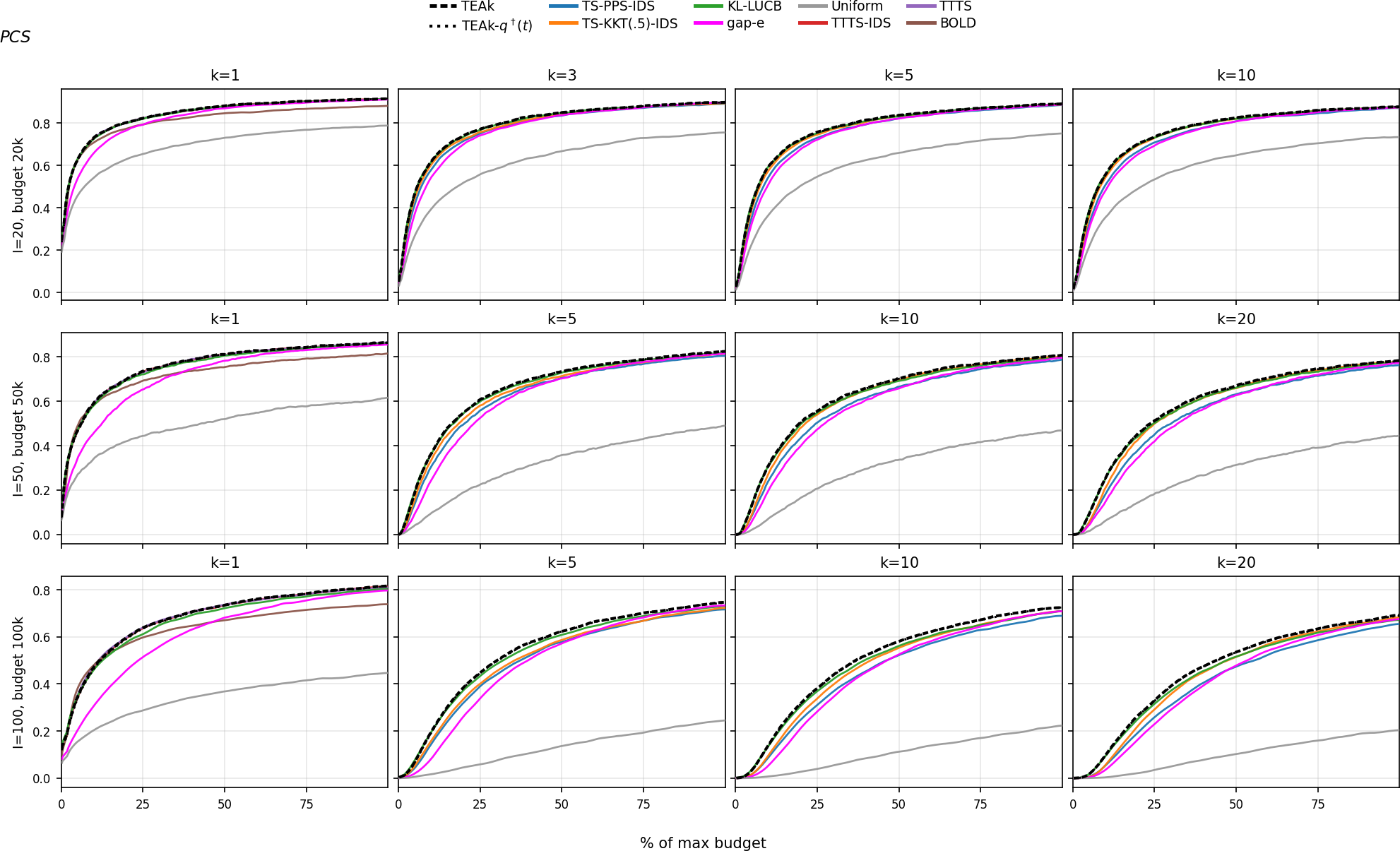}
\caption{Bernoulli: PCS.}
\label{fig:grid-pcs-bern}
\end{figure}

\begin{figure}[!htb]
\centering
\includegraphics[width=0.92\textwidth]{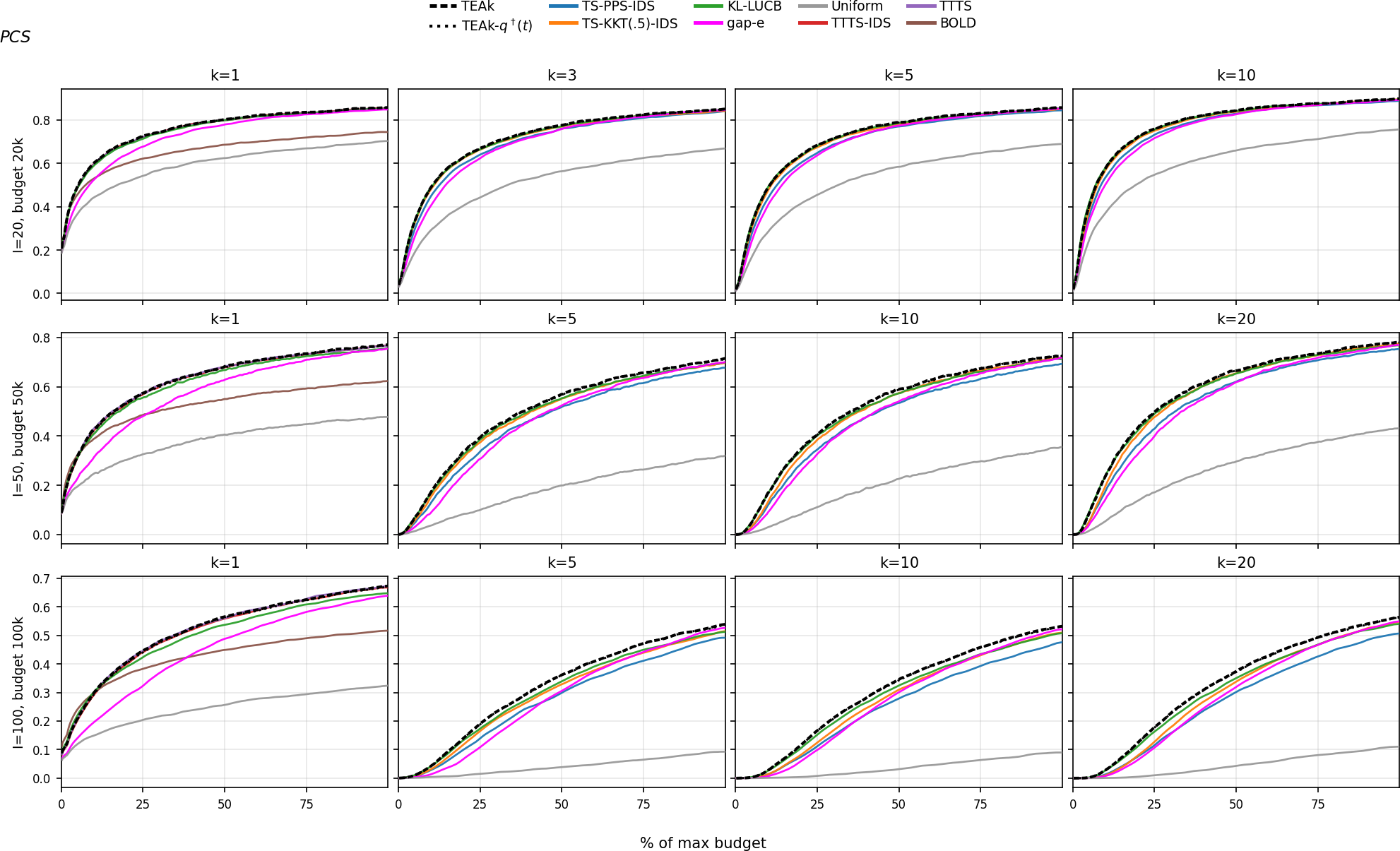}
\caption{Exponential: PCS.}
\label{fig:grid-pcs-exp}
\end{figure}

\begin{figure}[!htb]
\centering
\includegraphics[width=0.92\textwidth]{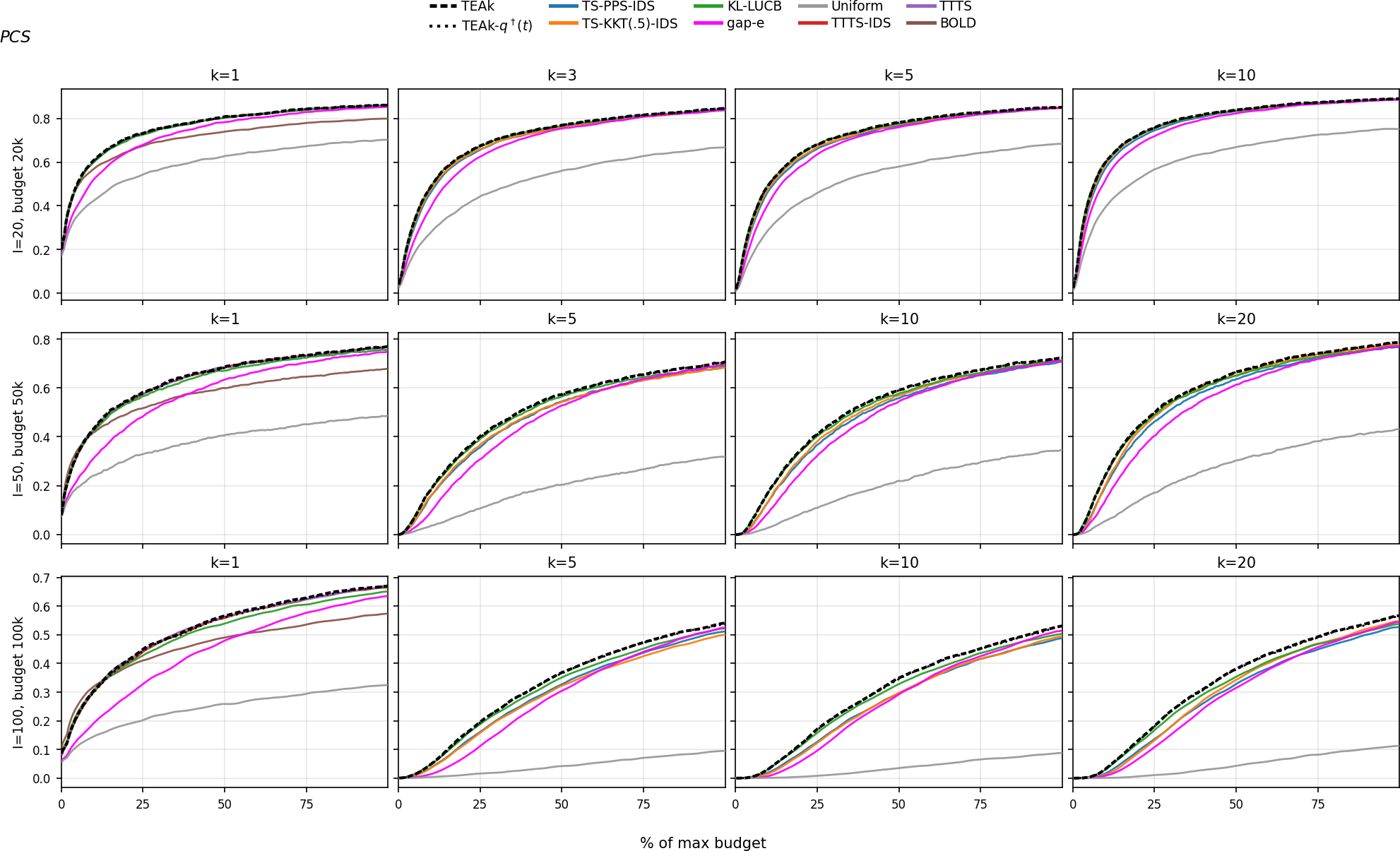}
\caption{Lognormal: PCS.}
\label{fig:grid-pcs-logn}
\end{figure}

\clearpage

\subsection{Right column}\label{app:pcs-hard}

\begin{figure}[!htb]
\centering
\includegraphics[width=0.92\textwidth]{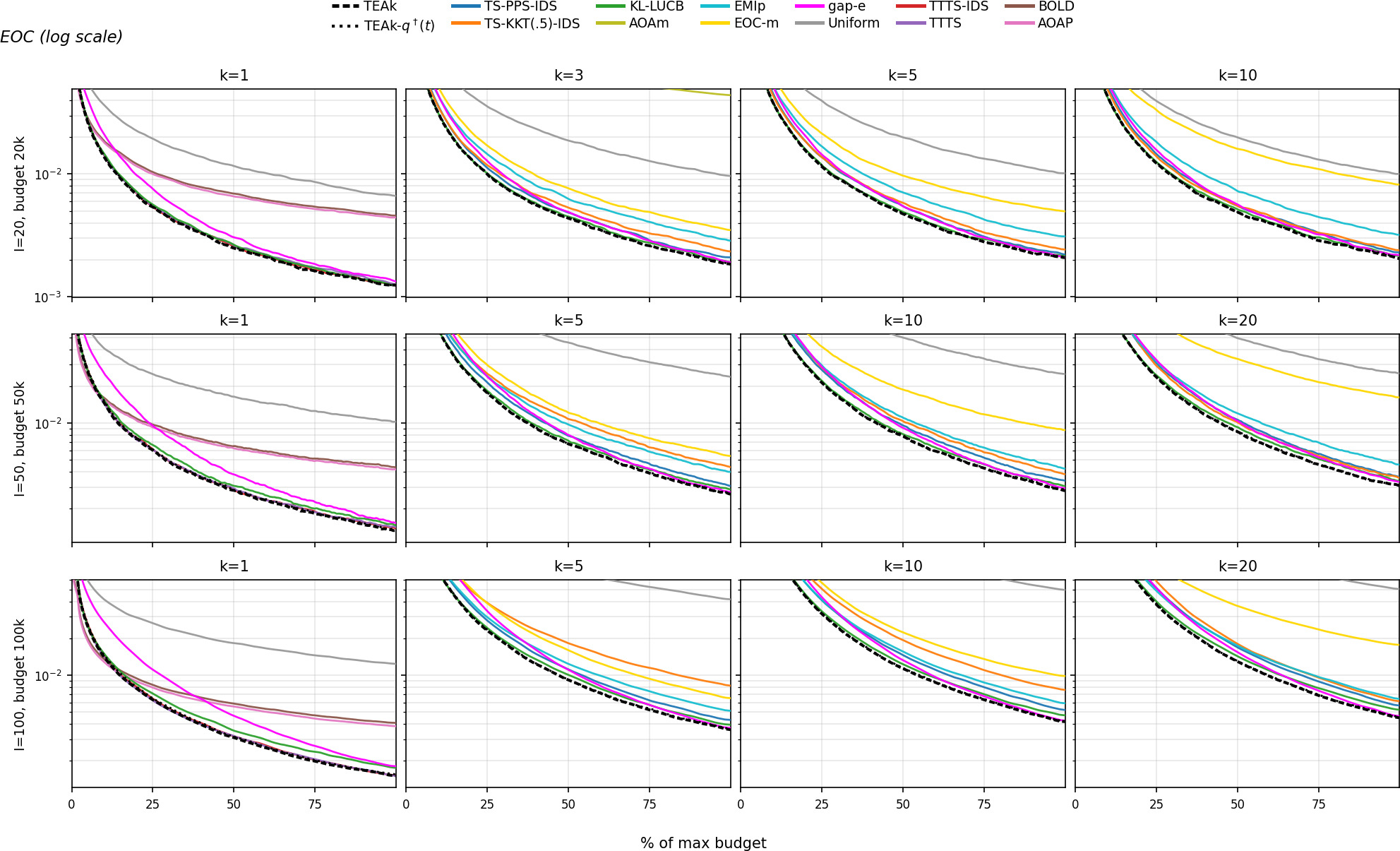}
\caption{Gaussian: EOC.}
\label{fig:grid-eoc-gauss-hard}
\end{figure}

\begin{figure}[!htb]
\centering
\includegraphics[width=0.92\textwidth]{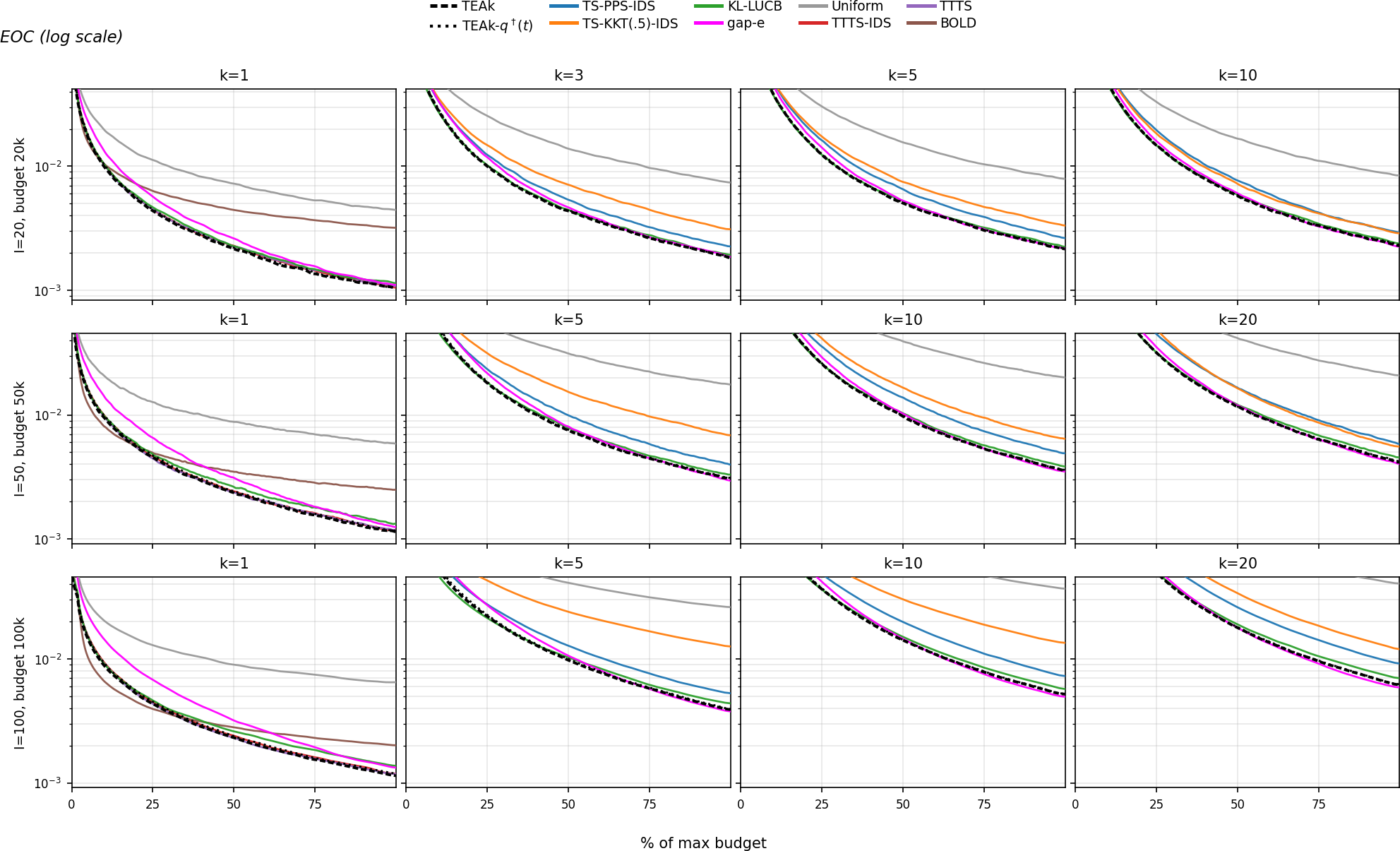}
\caption{Bernoulli: EOC.}
\label{fig:grid-eoc-bern-hard}
\end{figure}

\begin{figure}[!htb]
\centering
\includegraphics[width=0.92\textwidth]{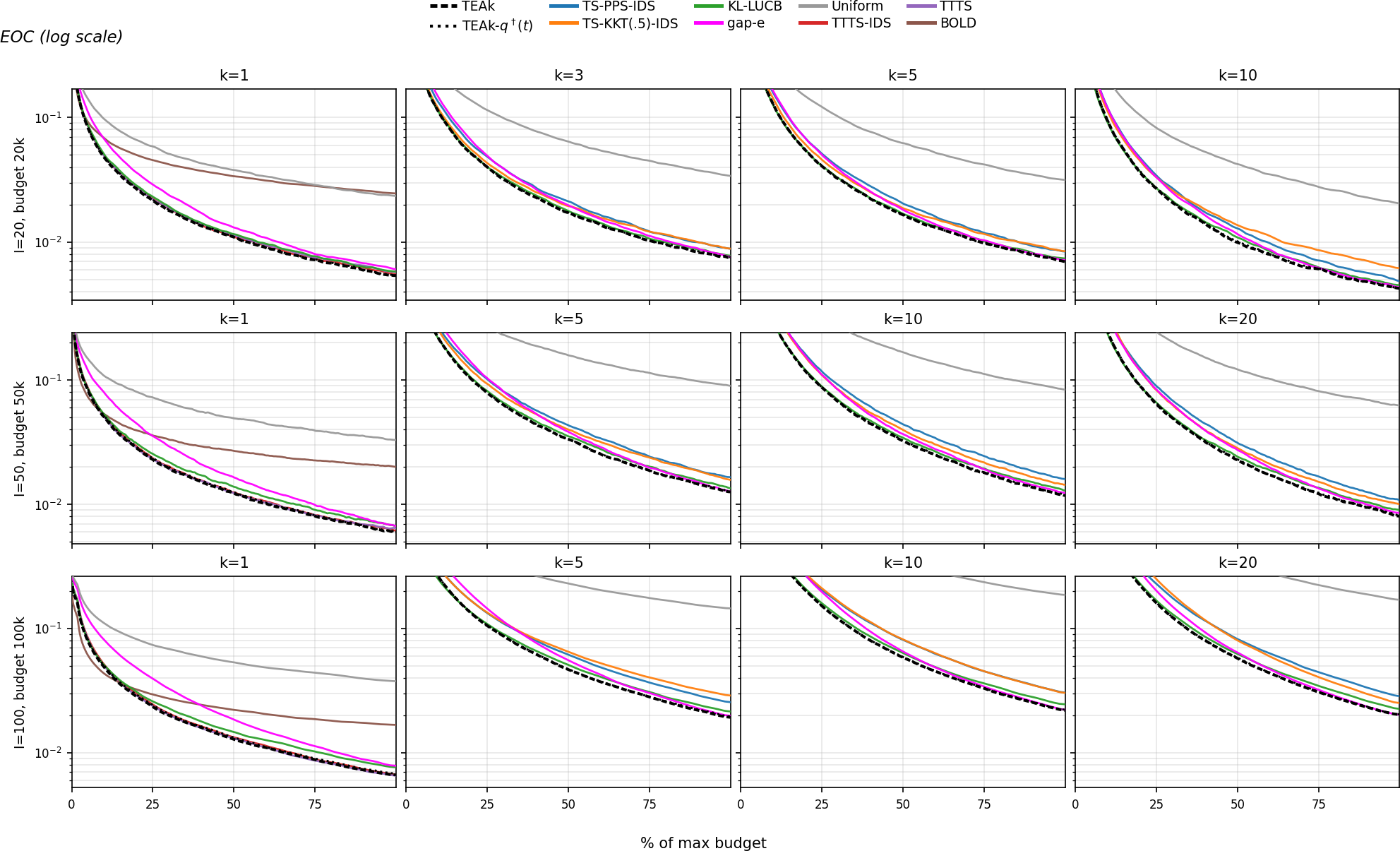}
\caption{Exponential: EOC.}
\label{fig:grid-eoc-exp-hard}
\end{figure}

\begin{figure}[!htb]
\centering
\includegraphics[width=0.92\textwidth]{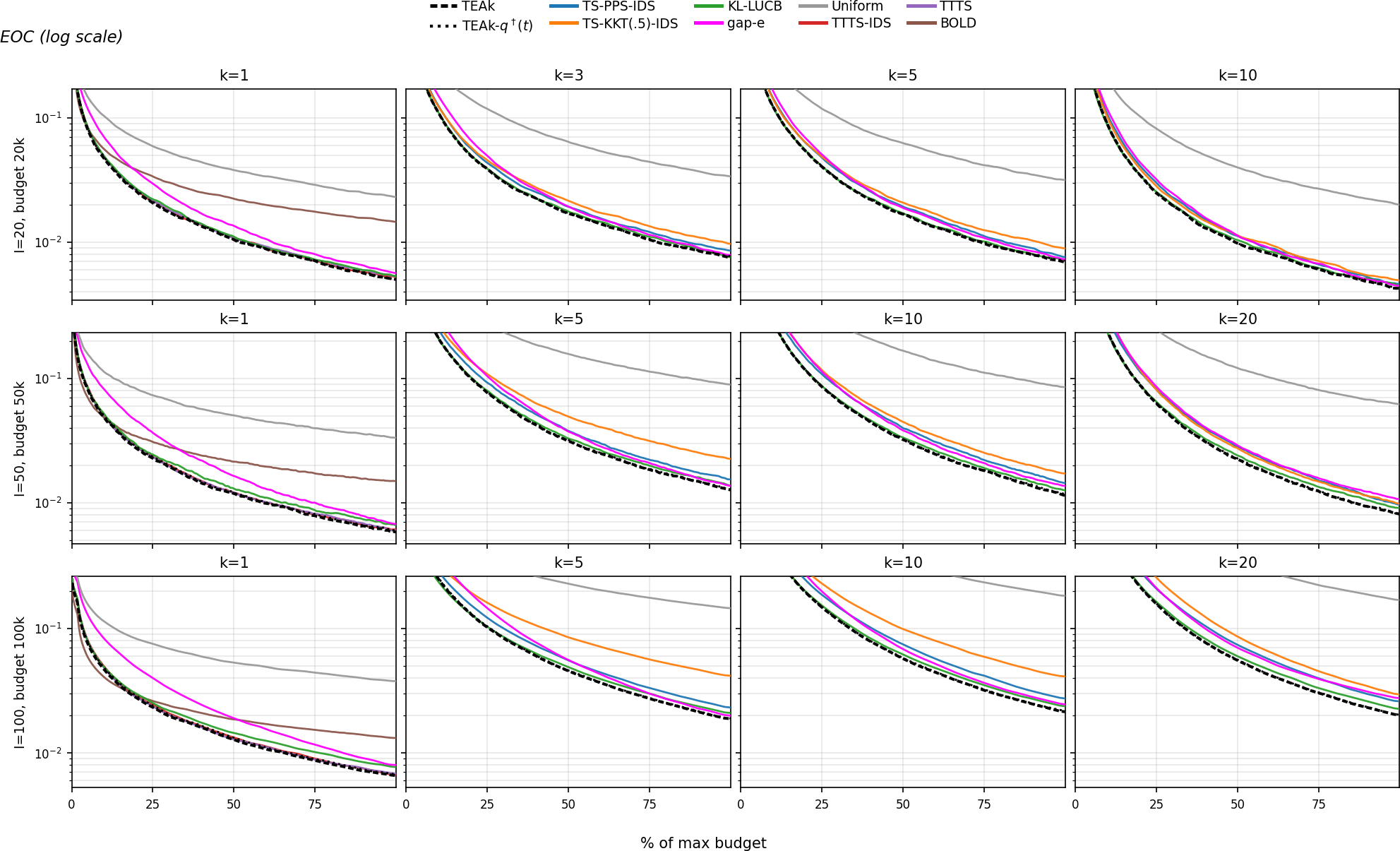}
\caption{Lognormal: EOC.}
\label{fig:grid-eoc-logn-hard}
\end{figure}

\begin{figure}[!htb]
\centering
\includegraphics[width=0.92\textwidth]{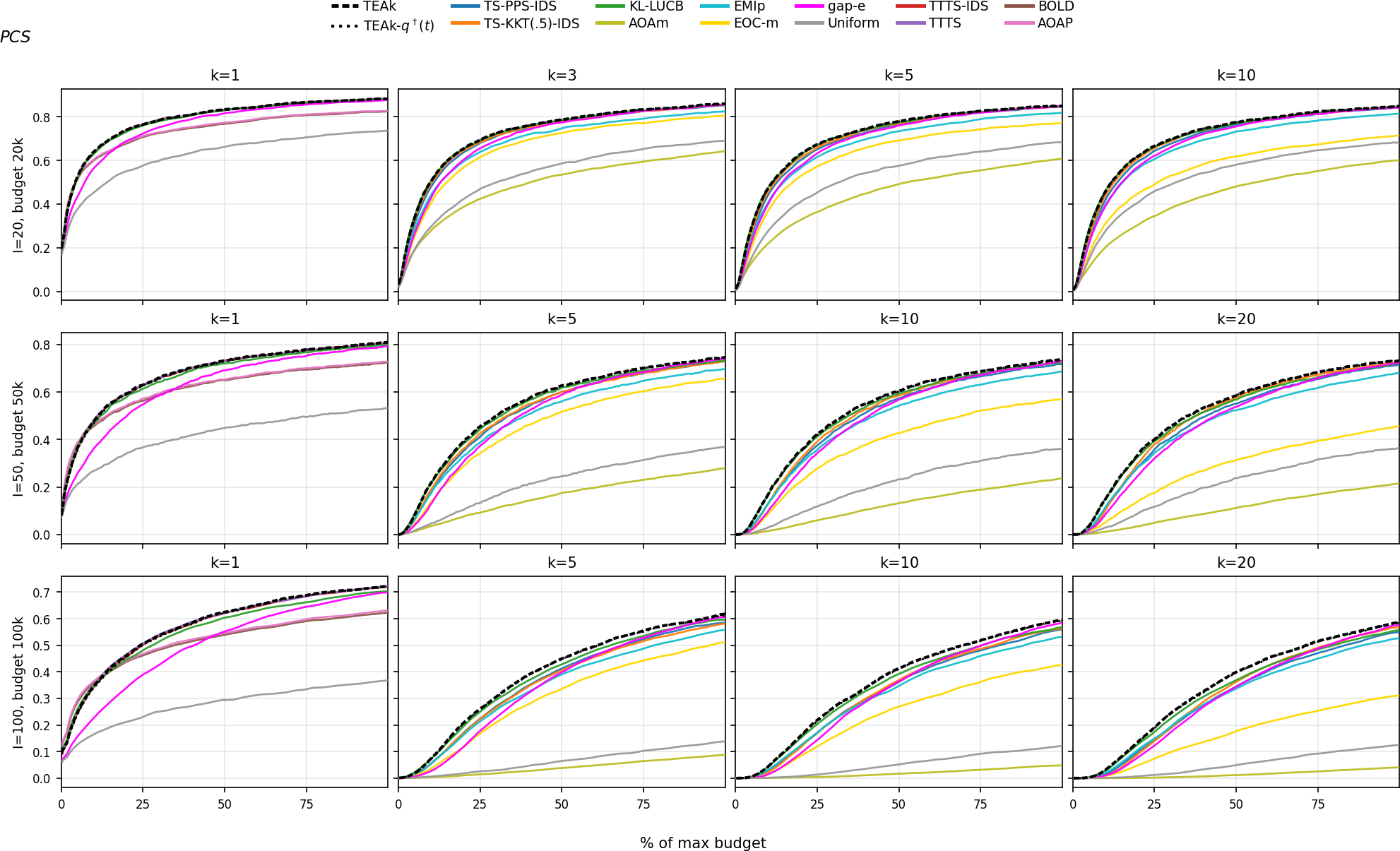}
\caption{Gaussian: PCS.}
\label{fig:grid-pcs-gauss-hard}
\end{figure}

\begin{figure}[!htb]
\centering
\includegraphics[width=0.92\textwidth]{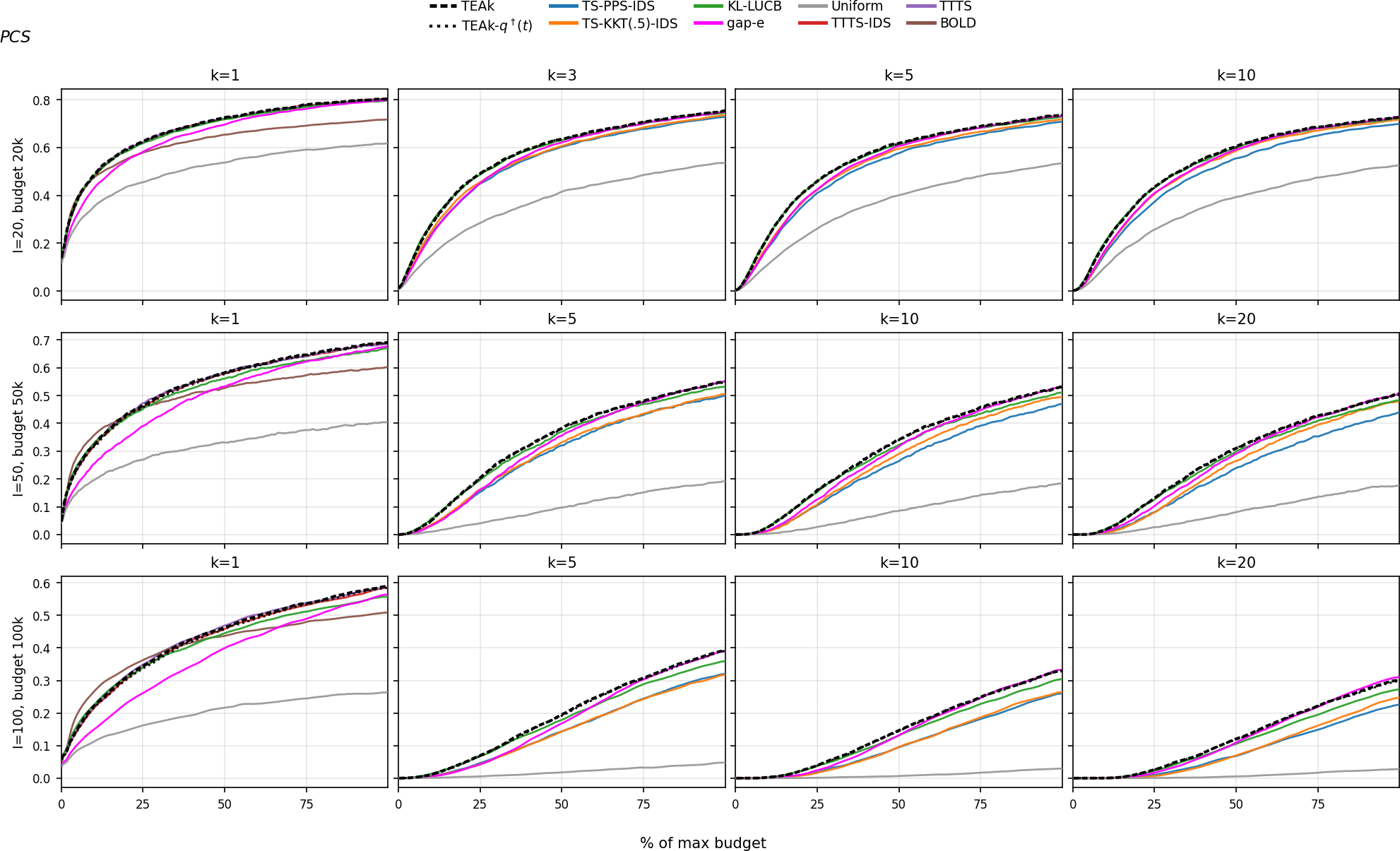}
\caption{Bernoulli: PCS.}
\label{fig:grid-pcs-bern-hard}
\end{figure}

\begin{figure}[!htb]
\centering
\includegraphics[width=0.92\textwidth]{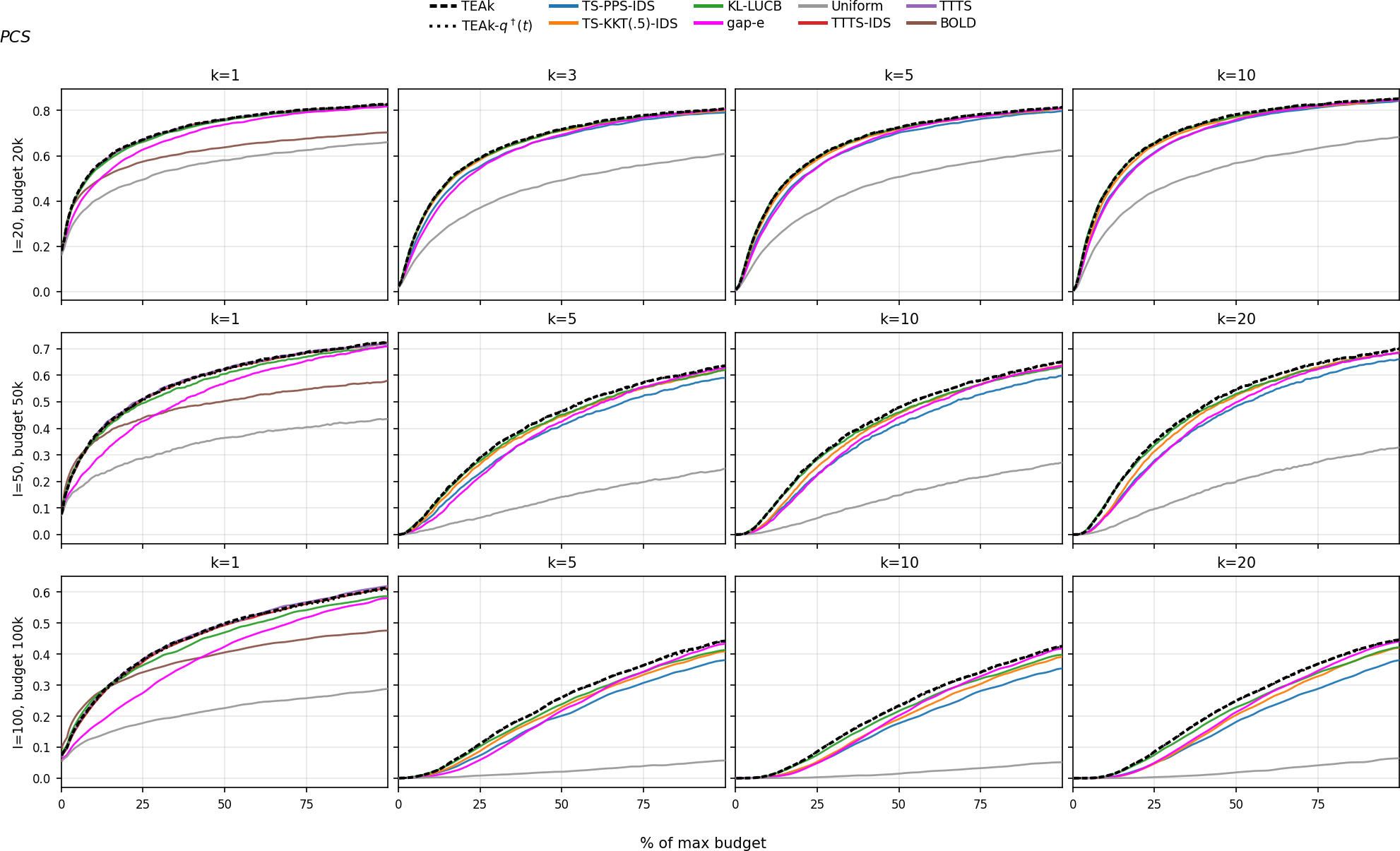}
\caption{Exponential: PCS.}
\label{fig:grid-pcs-exp-hard}
\end{figure}

\begin{figure}[!htb]
\centering
\includegraphics[width=0.92\textwidth]{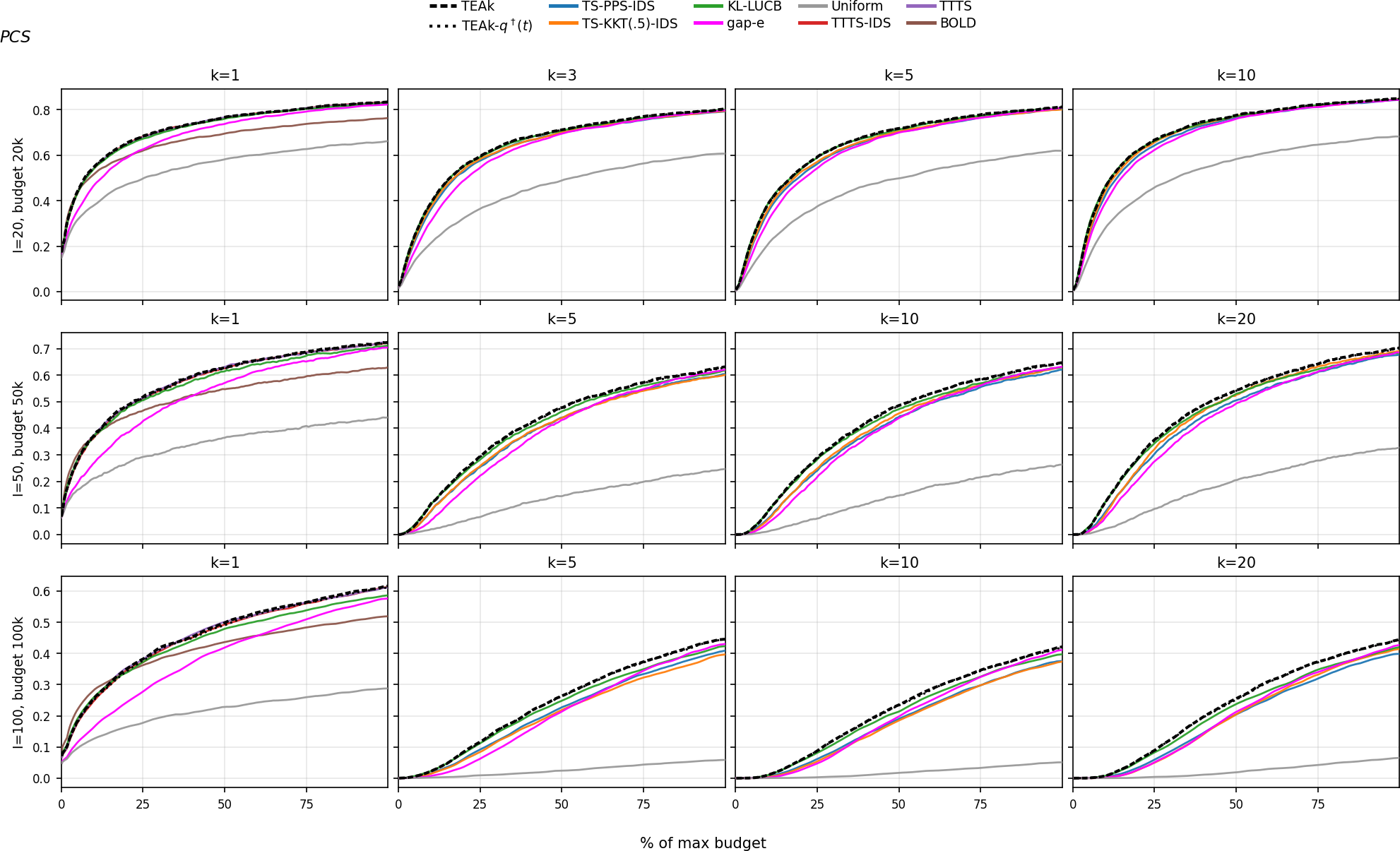}
\caption{Lognormal: PCS.}
\label{fig:grid-pcs-logn-hard}
\end{figure}

\clearpage

\subsection{Left column}\label{app:pcs-easy}

\begin{figure}[!htb]
\centering
\includegraphics[width=0.92\textwidth]{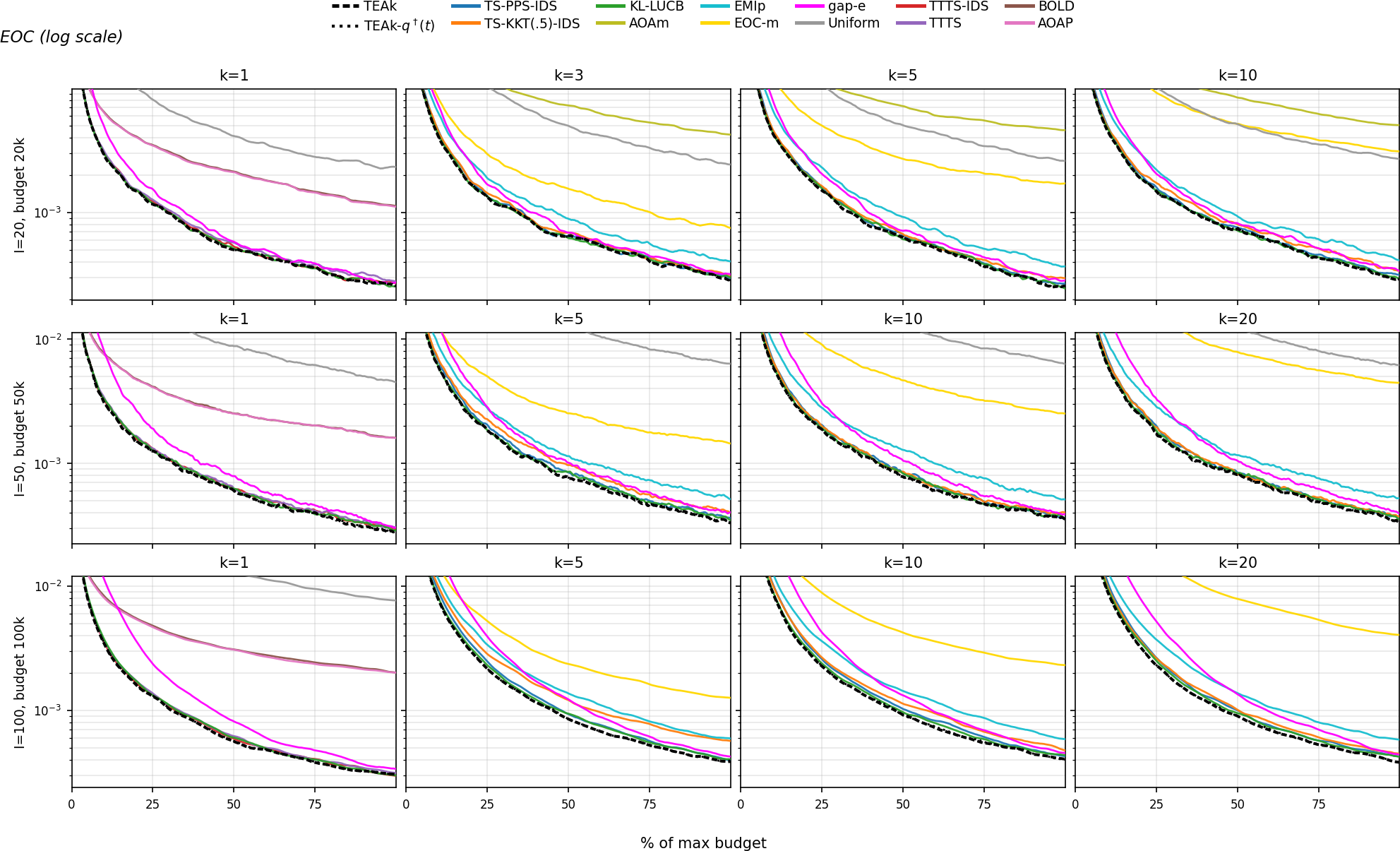}
\caption{Gaussian: EOC.}
\label{fig:grid-eoc-gauss-easy}
\end{figure}

\begin{figure}[!htb]
\centering
\includegraphics[width=0.92\textwidth]{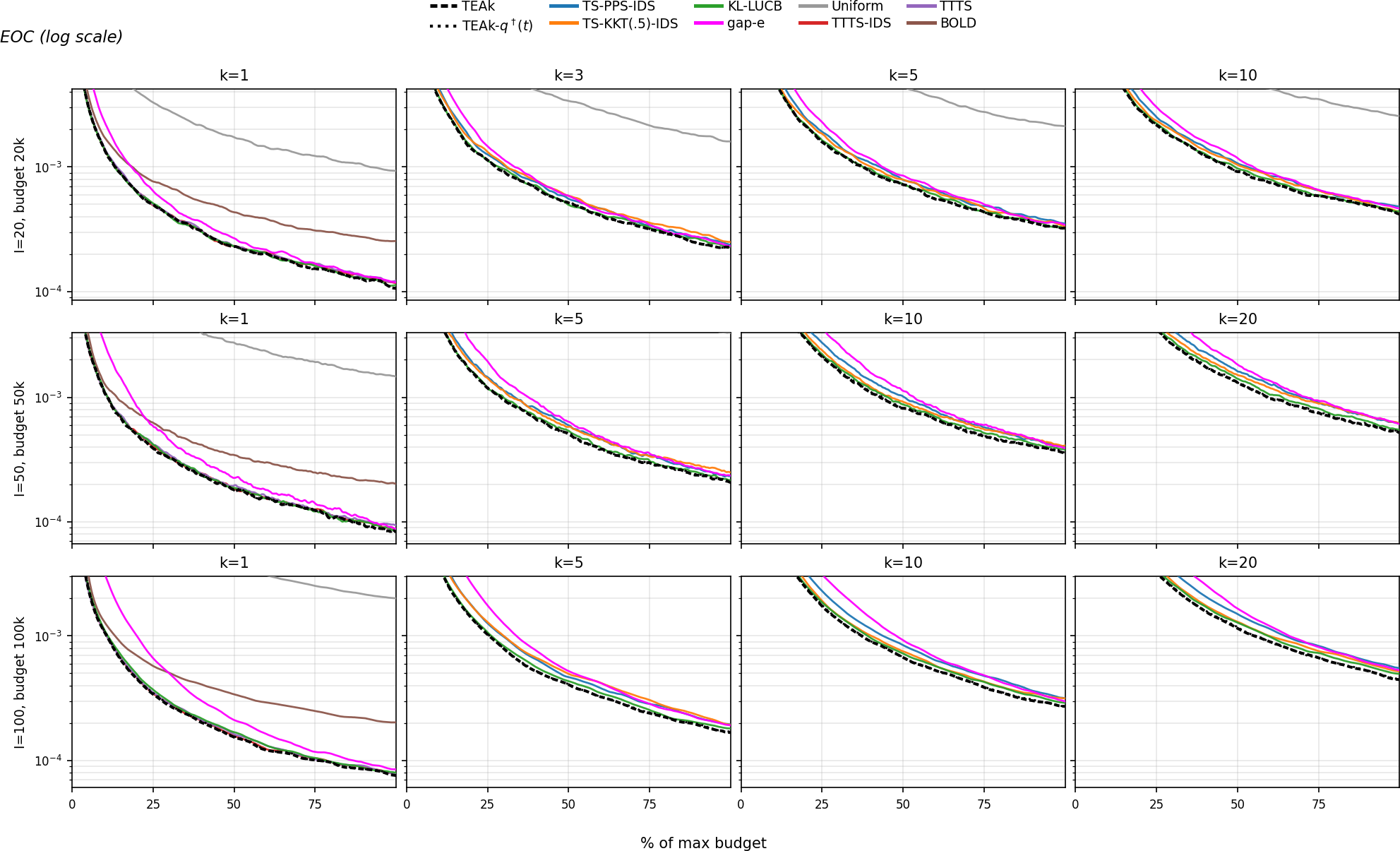}
\caption{Bernoulli: EOC.}
\label{fig:grid-eoc-bern-easy}
\end{figure}

\begin{figure}[!htb]
\centering
\includegraphics[width=0.92\textwidth]{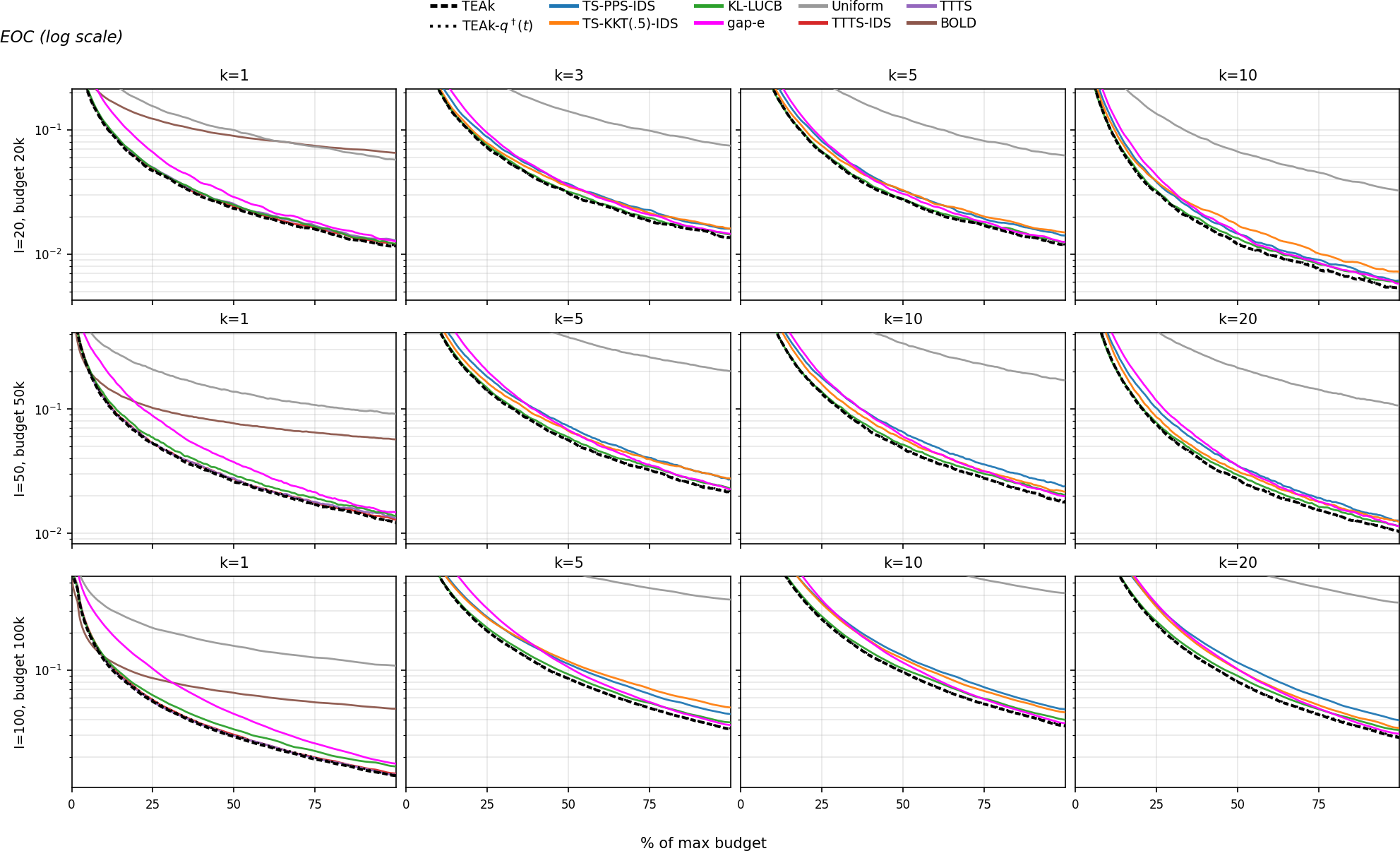}
\caption{Exponential: EOC.}
\label{fig:grid-eoc-exp-easy}
\end{figure}

\begin{figure}[!htb]
\centering
\includegraphics[width=0.92\textwidth]{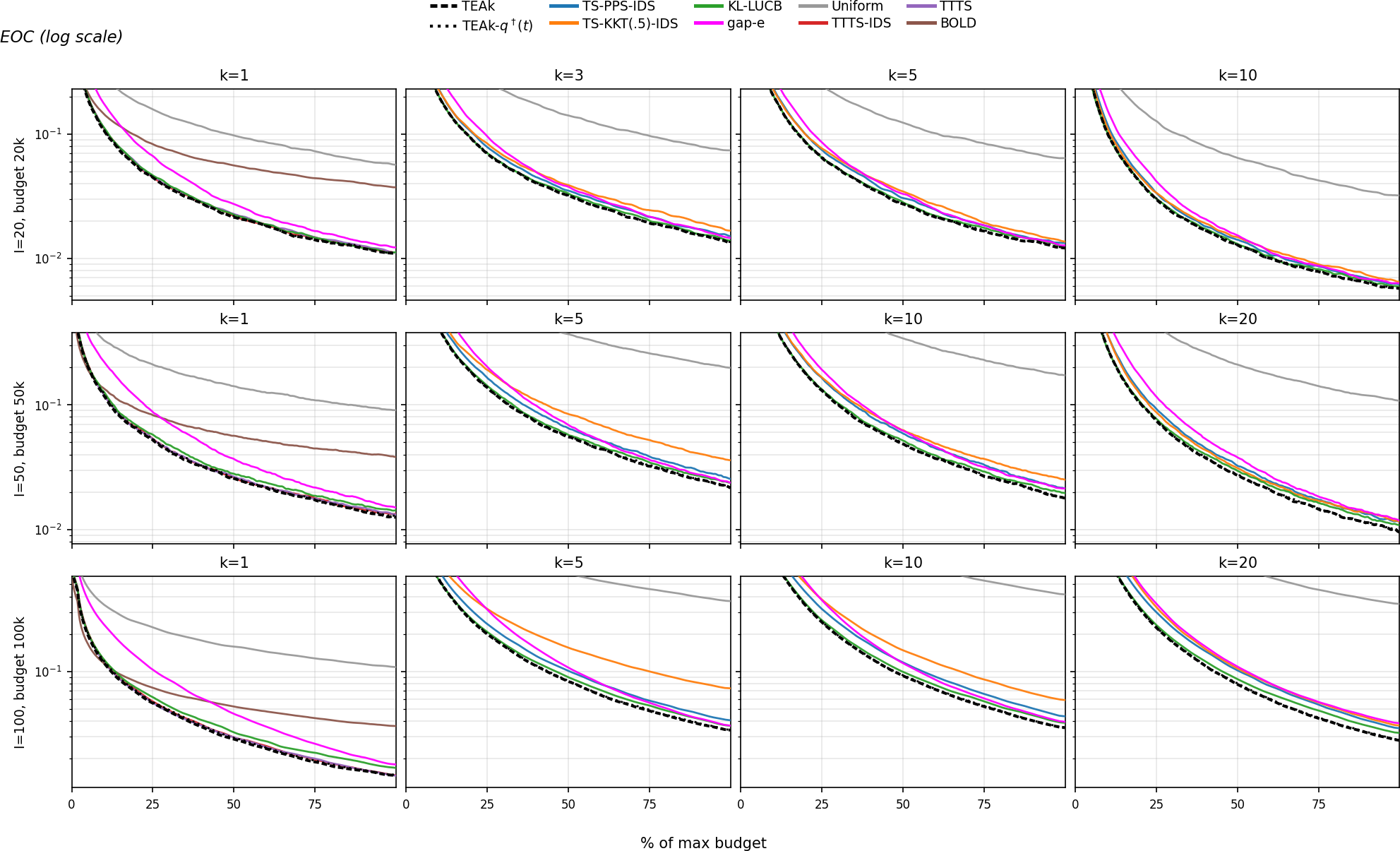}
\caption{Lognormal: EOC.}
\label{fig:grid-eoc-logn-easy}
\end{figure}

\begin{figure}[!htb]
\centering
\includegraphics[width=0.92\textwidth]{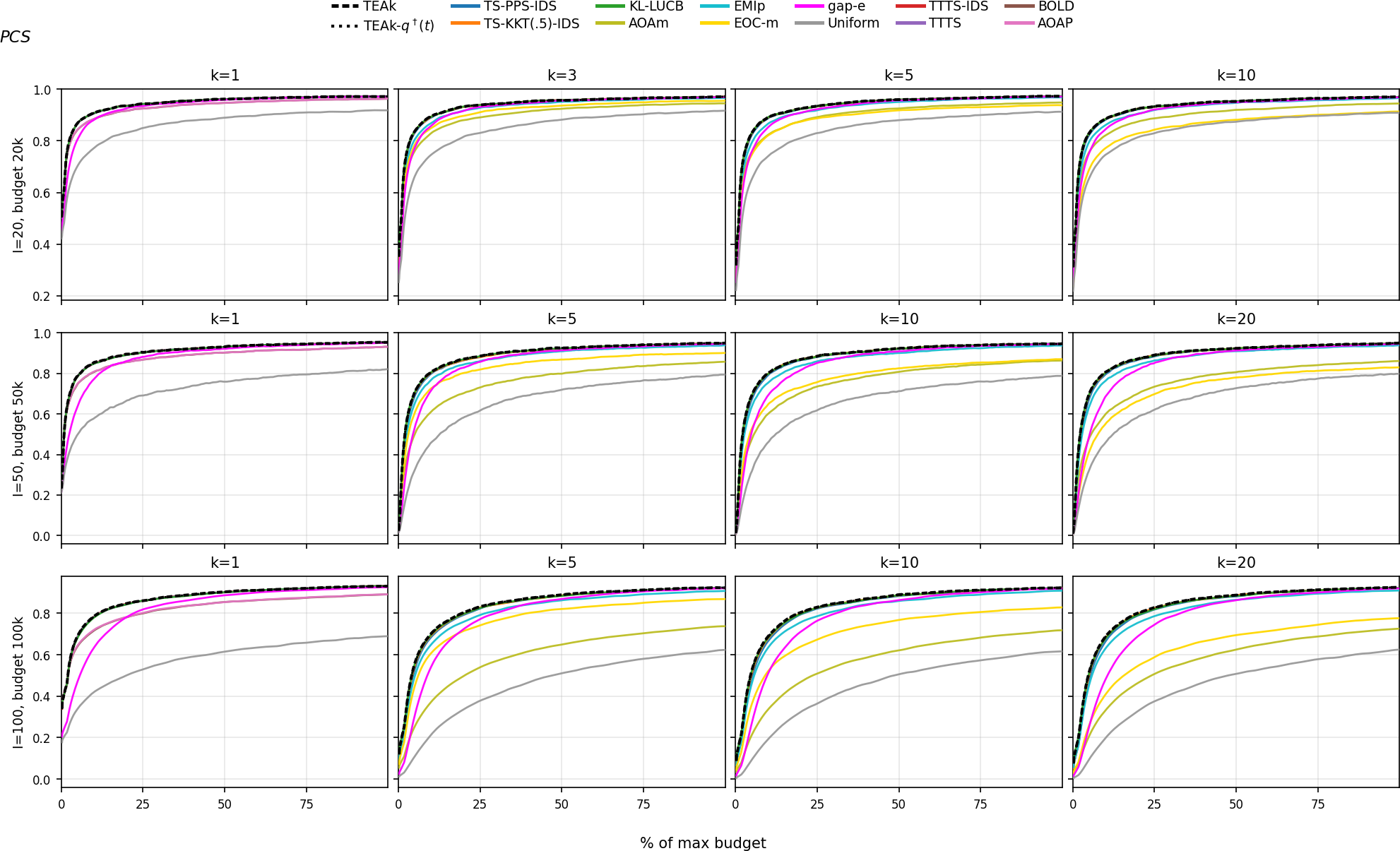}
\caption{Gaussian: PCS.}
\label{fig:grid-pcs-gauss-easy}
\end{figure}

\begin{figure}[!htb]
\centering
\includegraphics[width=0.92\textwidth]{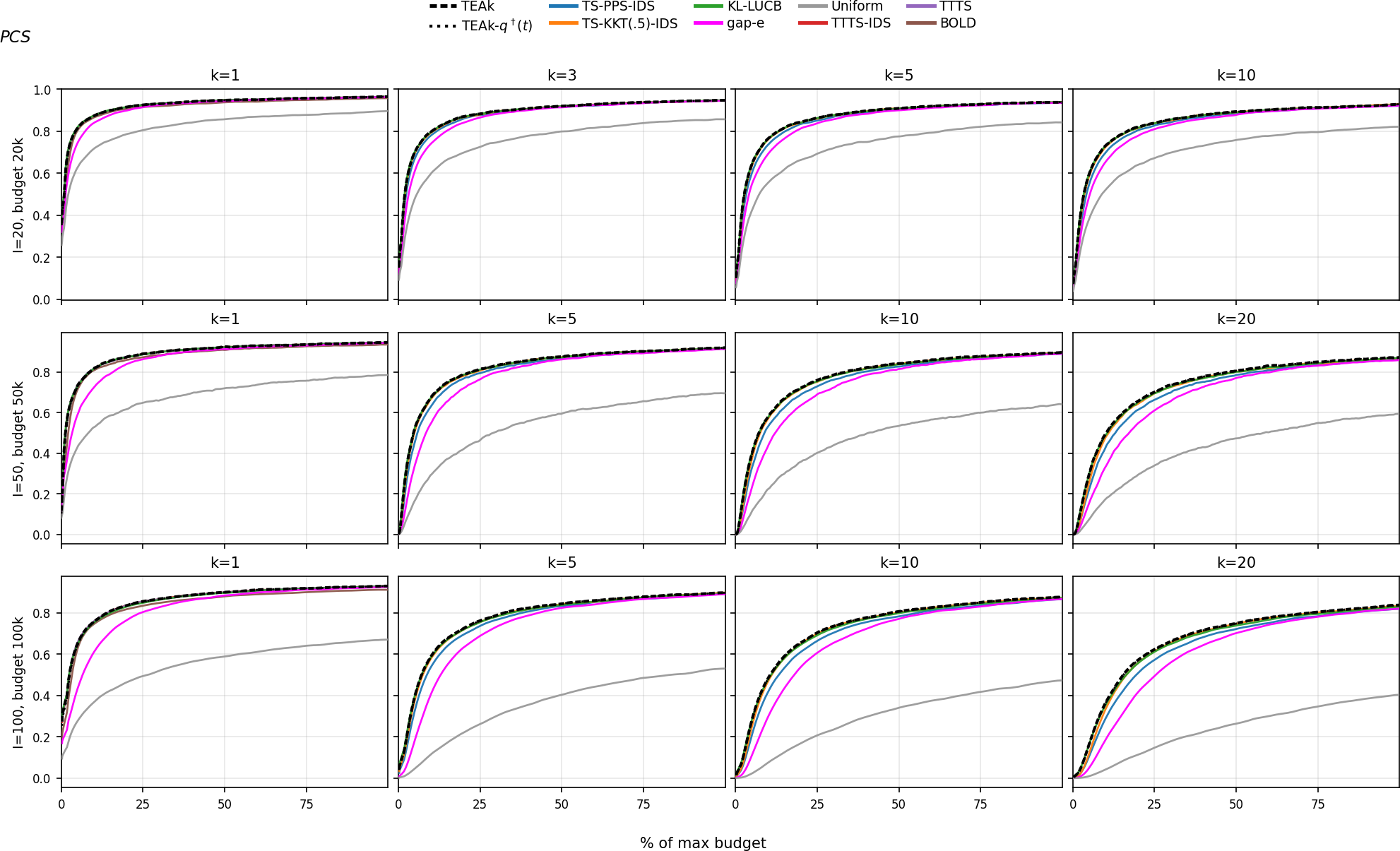}
\caption{Bernoulli: PCS.}
\label{fig:grid-pcs-bern-easy}
\end{figure}

\begin{figure}[!htb]
\centering
\includegraphics[width=0.92\textwidth]{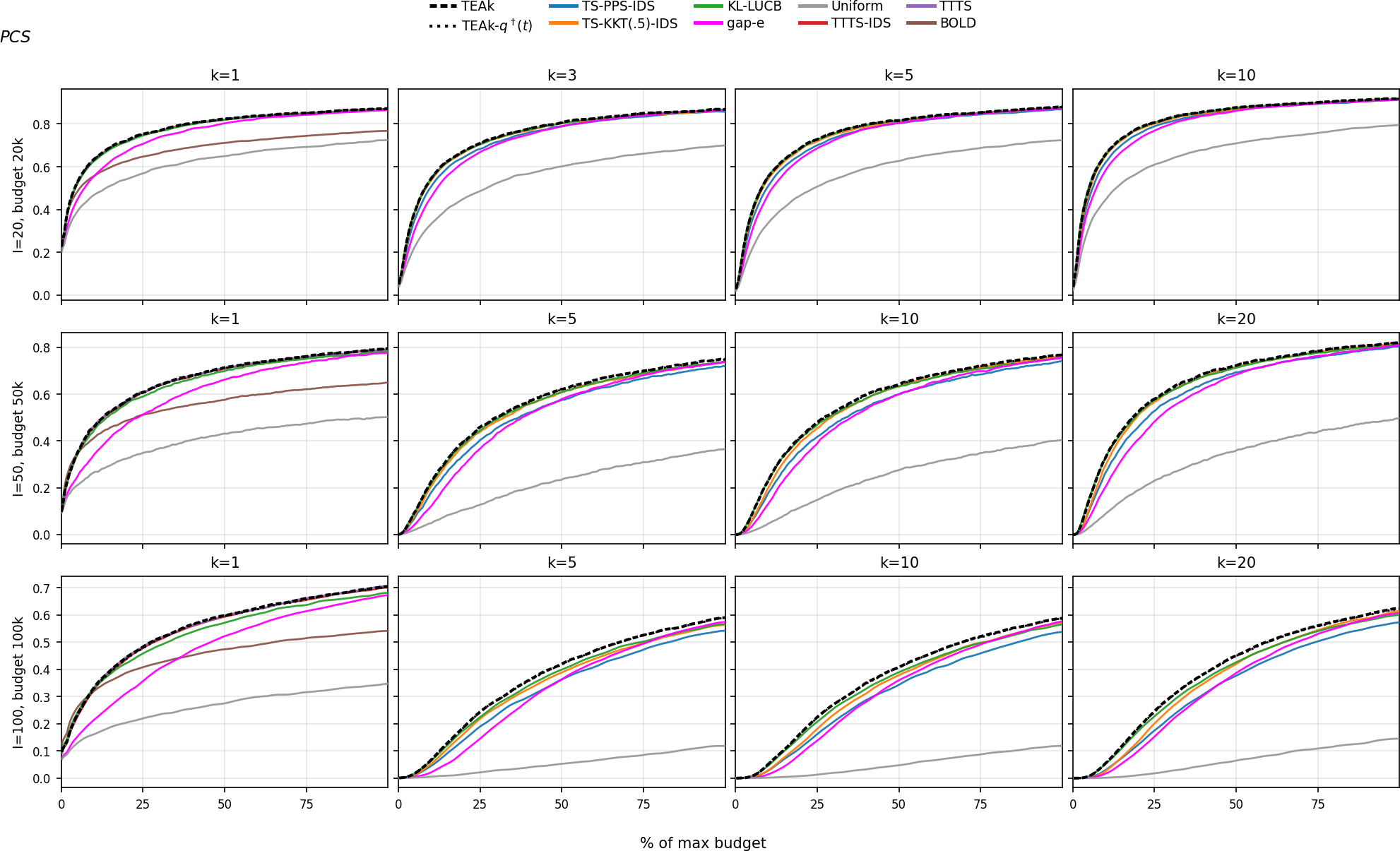}
\caption{Exponential: PCS.}
\label{fig:grid-pcs-exp-easy}
\end{figure}

\begin{figure}[!htb]
\centering
\includegraphics[width=0.92\textwidth]{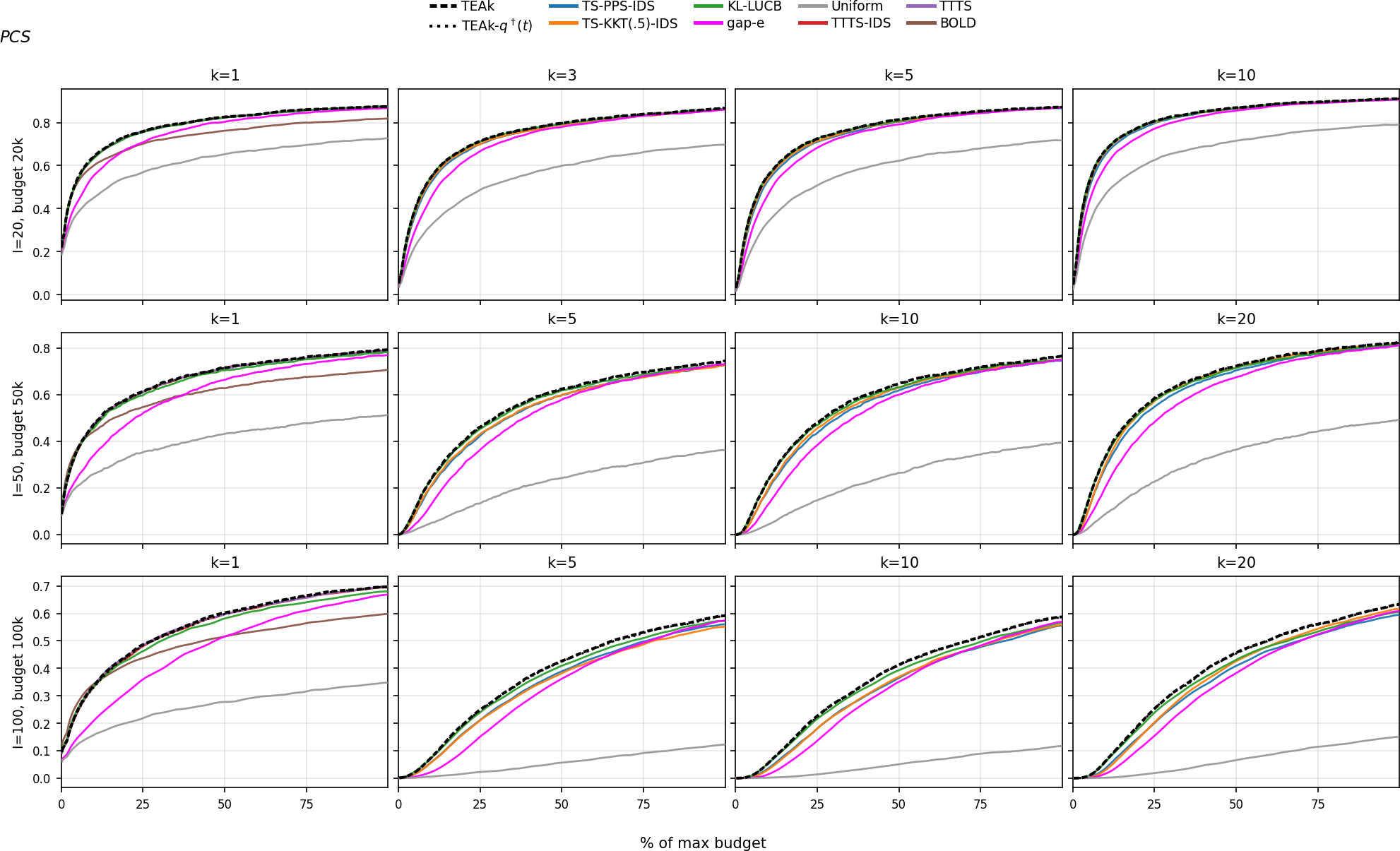}
\caption{Lognormal: PCS.}
\label{fig:grid-pcs-logn-easy}
\end{figure}

\end{APPENDICES}

\end{document}